\documentclass{siamart190516}
\usepackage{amssymb,version,graphicx,fancybox,pifont,color}
\usepackage{url,hyperref}
\usepackage{subeqnarray}
\usepackage{leftidx}
\usepackage{float}
\usepackage{amsmath}
\usepackage{bm}
\usepackage{epstopdf}
\usepackage{mathrsfs}
\usepackage[utf8]{inputenc}
\usepackage{mathtools}
\usepackage{indentfirst}
\usepackage{tcolorbox}
\usepackage{color}
\usepackage[export]{adjustbox}
\usepackage{epstopdf}
\usepackage{mathrsfs}
\usepackage{algorithm}
\usepackage{algorithmic}
\usepackage{graphicx}
\usepackage{subfigure}
\usepackage{diagbox}
\usepackage{pifont}
\usepackage{tikz}
\usepackage[makeroom]{cancel}
\usepackage{mcode}
\usepackage{multirow}
\allowdisplaybreaks
\numberwithin{equation}{section}

\def\cM{{\mathcal M}}
\allowdisplaybreaks
\numberwithin{equation}{section}
\newcommand{\R}{\mathbb{R}}

\newcommand{\vertiii}[1]{{\left\vert\kern-0.25ex\left\vert\kern-0.25ex\left\vert #1 
    \right\vert\kern-0.25ex\right\vert\kern-i0.25ex\right\vert}}

\newcommand{\bSigma}{\mbox{\boldmath${\Sigma}$}}

\newcommand{\btheta}{\mbox{\boldmath${\theta}$}}

\newcommand{\bbeta}{\mbox{\boldmath${\beta}$}}
\newcommand{\bxi}{\mbox{\boldmath${\xi}$}}

\newcommand{\bUpsilon}{\mbox{\boldmath${\Upsilon}$}}

\newtheorem{thm}{\bf Theorem}[section]
\newtheorem{lem}[thm]{Lemma}

\newsiamremark{remark}{Remark}

\DeclareMathOperator*{\argmin}{arg\,min}

\usepackage{enumitem}
\setlist[itemize]{left=16pt} %control itemize incident

\def\bc{{\bf c}}

\def\bff{{\bf f}}

\def\bH{{\bf H}}
\def\bG{{\bf G}}

\def\bp{{\bf p}}
\def\br{{\bf r}}
\def\bs{{\bf s}}

\def\bv{{\bf v}}

\def\bx{{\bf x}}
\def\by{{\bf y}}

\def\cH{{\cal H}}

\def\cS{{\cal S}}

\newcommand{\bomega}{\mbox{\boldmath$\omega$}}

\title{A Structure-Guided Gauss--Newton Method \\ for Shallow R{\scriptsize e}LU Neural Network \thanks{Submitted to the editors
    DATE.
\funding{This work of Zhiqiang Cai and Min Liu was supported in part by the National Science Foundation under grant DMS-2110571. The work of Jianlin Xia was supported in part by the National Science Foundation under grant DMS-2111007.}}} 

\author{Zhiqiang Cai\thanks{Department of Mathematics, Purdue University, West Lafayette, IN
    (\email{caiz@purdue.edu}, \email{ding158@purdue.edu}, \email{liu1957@purdue.edu}, \email{xiaj@purdue.edu}).}
\and Tong Ding\footnotemark[2]
\and Min Liu\thanks{School of Mechanical Engineering, Purdue University, West Lafayette, IN (\email{liu66@purdue.edu})}
\and Xinyu Liu\footnotemark[2]
\and Jianlin Xia\footnotemark[2]}
\begin{document}
\maketitle
\begin{abstract}
In this paper, we introduce a structure-guided Gauss--Newton (SgGN) method for solving least-squares problems using shallow ReLU neural networks. The method is designed to simultaneously exploit three distinct structural features of the problem: (1) the least-squares form of the objective function, (2) the layered architecture of the neural network, and (3) the internal structure of the Gauss--Newton matrix, which allows explicit separation and removal of its singular components.

By formulating the training task as a separable nonlinear least-squares (SNLS) problem, the method classifies the output layer parameters as linear and the hidden layer parameters as nonlinear. Optimization proceeds through a block-iterative scheme that alternates between a damped Gauss--Newton update for the nonlinear parameters and a direct linear solver for the linear ones. Under reasonable assumptions, we prove that the mass and layer Gauss--Newton matrices involved in these updates are symmetric and positive definite. Moreover, the structured form of the layer Gauss--Newton matrix enables efficient and reliable computation of search directions without the need for heuristic regularization or shifting techniques such as those used in Levenberg--Marquardt methods.

The SgGN method is validated on a variety of challenging function approximation tasks, including problems with discontinuities and sharp transitions, settings where standard optimizers typically struggle. Numerical results consistently demonstrate that SgGN achieves faster convergence and significantly greater accuracy, particularly in adaptively repositioning breaking hyperplanes to align with the underlying structure of the target function.
\end{abstract}

\begin{keywords}
structure-guided Gauss--Newton method, neural network, least squares, mass matrix, Gauss-Newton matrix, positive definiteness
\end{keywords}

\begin{AMS}
 65D15, 65K10
\end{AMS}
\section{Introduction}\label{sec: intro}
When a neural network (NN) is employed as a model for least-squares data fitting, determining the optimal network parameters involves solving a high-dimensional, non-convex optimization problem. This problem is often computationally intensive and complex. In practice, the most commonly used optimization algorithms (iterative solvers) in machine learning are first-order gradient-based methods (see, e.g., survey papers \cite{SIAM_Review18, Survey19, SCZZ2020_Survey}), primarily due to their low per-iteration cost and ease of implementation. However, the efficiency of these methods is highly sensitive to hyper-parameter choices, especially the learning rate, which is often difficult to tune. Furthermore, these methods typically exhibit slow convergence and are prone to stagnation, commonly referred to as the plateau phenomenon in training tasks (see, e.g., \cite{ain22}).

Recently, there has been growing interest in applying second-order optimization methods, such as BFGS \cite{broyden1970convergence, fletcher1970new, goldfarb1970family, shanno1970conditioning}, to solve NN-related optimization problems. For a comprehensive overview of their benefits and recent advances, we refer the reader to the survey articles \cite{SIAM_Review18, Survey19, SCZZ2020_Survey}. Among second-order techniques, the Gauss--Newton (GN) method is particularly well suited for solving nonlinear least-squares (NLS) problems. As detailed in classical books \cite{dennis1996numerical, ortega2000iterative}, the GN method is derived from Newton’s method, but leverages the structure of the least-squares objective by approximating the Hessian with its principal component, known as the GN matrix.
In recent years, GN-type methods have found increasing applications in machine learning. In particular, the Kronecker-Factored Approximate Curvature (KFAC) method \cite{martens2015optimizing} provides a structured approximation to the Fisher information matrix, enabling efficient matrix inversion. The Kronecker-Factored Recursive Approximation (KFRA) method \cite{botev2017practical} further refines this idea by constructing a block-diagonal approximation to the GN matrix for feedforward NNs. Beyond traditional learning tasks, GN-type approaches have also been used in the context of NN–based discretizations of partial differential equations (PDEs) (see, e.g., \cite{Hao2024, Jnini2024}).

Despite its appealing features, the GN method suffers from a fundamental limitation: while the GN matrix is always positive semi-definite, it is often singular. To address this, additional regularization techniques \textemdash such as the shifting strategy employed in the Levenberg--Marquardt (LM) method \cite{Levenberg, Marquardt} \textemdash are typically introduced to enforce invertibility. However, this modifies the original optimization problem and introduces a new layer of complexity, particularly in choosing an appropriate shifting parameter, which can be highly non-trivial in practice.

The purpose of this paper is to design and investigate a novel structure-guided Gauss-Newton (SgGN) iterative method for solving LS optimization problems using shallow ReLU NNs. The method utilizes both the LS structure of the objective and the layered architecture of ReLU NNs, and formulates the training task as a separable nonlinear LS (SNLS) problem \cite{osborne2007separable}. Building on this framework, the SgGN method leverages a natural decomposition by separating the neural network parameters into two groups: the linear parameters $\hat{\mathbf{c}}$ corresponding to the weights and bias of the output layer, and the nonlinear parameters $\br$ associated with the weights and biases of the hidden layer. This separation enables a block-iterative optimization strategy, wherein the method alternates between solving for $\hat{\mathbf{c}}$ using a direct linear solver and updating  $\br$ via a damped GN iteration. This iterative scheme not only exploits the layered structure of the network, but also the structure of the GN matrix. These structural insights allow us to isolate the source of singularity in the GN matrix and to rigorously justify the positive definiteness of the intermediate matrices that arise during optimization.  

At each SgGN iteration, the linear solver involves a mass matrix $\mathcal{A}\left(\br\right)$, defined in \cref{A-f}, which depends only on the nonlinear parameters. The nonlinear GN iterative solver relies on a newly derived structured form of the GN matrix for shallow ReLU NN (see \cref{H}). This matrix takes the form:
\begin{equation}
   \mathcal{G}(\br) = \big(D(\bc)\otimes I_{d+1}\big) {\mathcal{H}}(\br)\big(D(\bc)\otimes I_{d+1}\big), \label{eq:gnmat}
\end{equation}
where $d$ is the input dimension, $\bc$ is the weights in the linear parameters $\hat{\mathbf{c}}$ corresponding to coefficients of a linear combination of the neurons, $D(\bc)$ is a diagonal matrix formed from the entries of $\bc$, and $I_{d+1}$ is the order-$(d+1)$ identity matrix. The matrix ${\mathcal{H}}(\br)$, referred to as the layer GN matrix in this work, depends solely on the nonlinear parameters and is given explicitly in \cref{H}. Importantly, both $\mathcal{A}\left(\br\right)$ and  ${\mathcal{H}}(\br)$ are functions of the nonlinear parameter $\br$ and are independent of the linear parameters $\hat{\mathbf{c}}$ .

Theoretically, we show that both $\mathcal{A}\left(\br\right)$ and ${\mathcal{H}}(\br)$ are {\it symmetric positive definite} provided that the neurons are linearly independent (see \cref{l: A} and \cref{l: Hspd}). This property is critical, as it ensures that each iteration of the SgGN algorithm involves well-defined subproblems. The natural positive definiteness of $\mathcal{A}\left(\br\right)$ and ${\mathcal{H}}(\br)$ enables the use of a wide range of efficient direct or iterative solvers for computing the updates to the linear parameters (see \cref{GN1}) and for determining the GN search direction for the nonlinear parameters (see \cref{GN2}).
Moreover, the factored form of the GN matrix in \eqref{eq:gnmat} offers several important theoretical and practical advantages:

\begin{itemize}    
\item [(a)] \textbf{No artificial shifting:} The positive definiteness of $\mathcal{H}(\mathbf{r})$ eliminates the need for additional techniques such as the shifting strategy in the LM method to enforce {\it invertibility} of the GN matrix \textemdash a step commonly required in traditional GN methods. 

\item [(b)] \textbf{Structural insight into singularity:} The factorized form $\mathcal{G}(\br)$ makes the source of singularity explicit. If an entry of $\mathbf{c}$, the coefficients in the linear combination of hidden layer neurons, is zero, then the corresponding neuron does not need to be updated. Once filtered, the remaining system is strictly positive definite. In contrast, the LM method ignores such structure and instead perturbs all directions uniformly via shifting, often distorting the optimization problem and leading to suboptimal performance.
\end{itemize}

The proposed SgGN method is applicable to both continuous and discrete least-squares approximation problems. Its convergence and accuracy are validated through numerical experiments on a range of one- and two-dimensional test problems, including cases with discontinuities and sharp transitions\textemdash scenarios where commonly used optimization algorithms such as Adam \cite{kingma2015}, BFGS, and KFRA often perform poorly. In all tested cases, the SgGN method exhibits significantly faster convergence and higher approximation accuracy. These improvements are particularly evident in the method’s ability to adaptively reposition the breaking hyperplanes (points in 1D and lines in 2D) that are determined by nonlinear parameters, to accurately align with the underlying structure of the target functions.

The remainder of this paper is organized as follows. \Cref{sec: net} introduces the set of approximating functions generated by shallow ReLU NNs and establishes the linear independence of neurons. The LS optimization problem and the corresponding nonlinear algebraic system for stationary points are described in \Cref{sec: prob}. 
In \Cref{sec: training}, the structure of the GN matrix for the nonlinear parameters is derived and the resulting SgGN method is proposed. \Cref{sec: discrete} presents the SgGN method for discrete LS optimization. The numerical results are given in \Cref{sec: numerics}, which illustrates the performance of the method on a range of function approximation tasks. Finally, conclusions and potential directions for future work are discussed in \Cref{sec:conl}.

\section{Shallow ReLU neural network}\label{sec: net}
This section describes shallow ReLU NN as a set of continuous piecewise linear functions mapping $\R^d$ to $\R$. For clarity and simplicity, we restrict our discussion to scalar-valued output functions, as the extension to higher-dimensional outputs is conceptually straightforward and does not alter the core analytical framework. We focus on the fundamental analytical and geometric properties of this function class, which are critical for understanding the behavior of the least-squares optimization problem and for developing the proposed structure-guided Gauss--Newton method.

The ReLU activation function, short for rectified linear unit, is defined as
\begin{equation}\label{relu}
    \sigma(t) =\max\{0, t\}= \left\{\begin{array}{ll}
     t, & t>0,\\[2mm]
     0, & t\leq 0.
     \end{array}\right.
 \end{equation}
Its first- and second-order weak derivatives are the Heaviside (unit) step and the Dirac delta functions given by
\begin{equation}\label{H-delta}
    H(t) = \sigma^\prime(t)= \left\{\begin{array}{ll}
    1, & t>0,\\[2mm]
    0, & t< 0
    \end{array}\right.
    \quad \mbox{ and }\quad 
    \delta(t)=\sigma^{\prime\prime}(t)=H^\prime(t) =\left\{\begin{array}{ll}
    \infty, & t=0,\\[2mm]    0, & t\neq 0,
    \end{array}\right.
\end{equation}
respectively. 

Let $\Omega$ be a connected, bounded open domain in $\R^d$. For any $\bx=(x_1,\ldots,x_d)^T\in \Omega \subset\R^d$, by appending $1$ to the inhomogeneous $(x_1,\ldots,x_d)$-coordinates, we have the following homogeneous coordinates:
\[
\by^T=\left(1,\bx^T\right)=(1,x_1,\ldots,x_d).
\]
A standard shallow ReLU NN with $n$ neurons may be viewed as the set of continuous piecewise linear functions from $\Omega\subset \R^d$ to $\R$, defined as follows:
\begin{equation}\label{NN}
    {\cM}_n(\Omega)=\left\{c_0+
    \sum_{i=1}^{n}c_i\sigma({\bomega}_{i}\cdot\bx+b_i):\,  \bx\in \Omega,\, c_i\in \R, \, b_{i}\in \R, 
    \,{\bomega}_i\in \mathcal{S}^{d-1}\right\},
\end{equation}
where $\bc=(c_1,\ldots,c_n)^T$ and $c_0$ are the output weights and bias, respectively; $\bomega_i=(\omega_{i1},\ldots, \omega_{id})^T$ and $b_i$ are the respective weight and bias of the $i^{\mbox{\scriptsize th}}$ neuron in the hidden layer, with $\bomega_i$ restricted to lie on the unit sphere $\mathcal{S}^{d-1}$ in $\R^{d}$. This weight normalization constraint is imposed without loss of generality, as it preserves the expressiveness of the network while narrowing down the solution set for a given approximation problem (see \cite{LiuCai1}). For notational convenience, we define the full collection of nonlinear parameters as
\begin{equation}\label{z}
\br=\begin{bmatrix}
    \br_1 \\ \vdots \\ \br_n
\end{bmatrix}   
   \quad\mbox{with }\quad \br_i = \begin{bmatrix}
   b_i \\ \omega_{i1} \\ \vdots \\ \omega_{id}\end{bmatrix}
   = \begin{bmatrix} b_i \\ \bomega_i\end{bmatrix},
\end{equation}
which will be used throughout the subsequent analysis and algorithmic development.

Any NN function $v(\bx)$ in ${\cM}_n(\Omega)$ is determined by the parameters $\hat{\bc}=(c_0,c_1,\ldots, c_n)^T$ and $\br$, and takes the following form 
\begin{equation}\label{v}
    v(\bx) = v(\bx; \hat{\bc}, \br) 
= c_0+  \sum_{i=1}^{n}c_i\sigma(\br_i\cdot\by),
\end{equation}
where $\hat{\bc}$ consists of the coefficients in the linear combination and is referred to as the {\em linear parameters}, while $\br$ denotes the {\em nonlinear parameters} (associated with the hidden layer weights and biases). 
To understand the geometric meaning of the nonlinear parameters $\br$, notice that the ReLU activation function $\sigma(t)$ is a continuous piecewise linear function with a single {\it breaking point} at $t=0$. Consequently, each neuron $\sigma (\br_i\textcolor{blue}{\cdot}\,\by)=\sigma (\bomega_i\cdot\bx+b_i)$ defines a continuous piecewise linear function with a corresponding  {\it breaking hyperplane} (see \cite{Cai2021linear,LiuCai1}):
\begin{equation}\label{P-i}
\mathcal{P}_i(\br_i)=\left\{\bx\in \Omega\subset \R^d:\, {\bomega}_{i}\cdot\bx + b_i=0 \right\}.
\end{equation}
Together with the boundary of the domain $\Omega$, these hyperplanes induce  a {\it physical partition}, denoted by ${\cal K}(\br)$, of the domain $\Omega$ \cite{LiuCai1, Cai2023AI}. This partition ${\cal K}(\br)$ consists of irregular, polygonal subdomains of $\Omega$ (see \cref{fig6:SgGN_h,fig6:SgGN_v} below for some examples). The NN function $v(\bx)$ defined in \cref{v} is thus a continuous piecewise linear function with respect to ${\cal K}(\br)$.

Now we turn to the discussion of the linear independence of some ridge functions defined for fixed parameter $\br$ in \cref{z}. To this end, let $\sigma_0(\bx)=1$, and for $i=1,\ldots,n$,
\begin{equation}\label{s-H}
    \sigma_i(\bx) %=\sigma_i(\bx;\br_i)
    =\sigma(\br_i\cdot\by) 
    \quad\mbox{and}\quad 
    H_i(\bx) %= H_i(\bx;\br_i) 
    = H(\br_i\cdot\by),
\end{equation}
where $\sigma$ and $H$ denote the ReLU activation and Heaviside step functions given in \cref{relu} and \cref{H-delta}, respectively. Under the assumption that the hyperplanes $\left\{\mathcal{P}_i(\br_i)\right\}_{i=1}^n$ are distinct, it is well known (see, e.g., Theorem~2.1 in \cite{he2020relu} and Lemma~2.1 in \cite{LiuCai1}) that the set of functions $\left\{\sigma_i (\bx )\right\}_{i=0}^n$ is linearly independent in $\R^d$. 

However, since $\Omega$ is a bounded sub-domain of $\R^d$ and each neuron $\sigma_i(\bx)$ is piecewise linear, certain degeneracies may arise. In particular, if there exist $\hat{d}>d$ ridge functions $\left\{\sigma_{i_k} (\bx )\right\}_{k=1}^{\hat{d}}$ whose restrictions to $\Omega$ are linear (i.e., their breaking hyperplanes lie outside of $\Omega$), then the set $\left\{\sigma_0(\bx)\right\}\cup \left\{\sigma_{i_k} (\bx)\right\}_{k=1}^{\hat{d}}$ is linearly dependent in $\Omega$. Consequently, the full set $\left\{\sigma_i (\bx )\right\}_{i=0}^n$ would also be linearly dependent in $\Omega$. To avoid this situation, we require that the intersection of the breaking hyperplane of each neuron with the domain $\Omega$ is not empty. In particular, let us introduce an admissible set for $\br$ defined in \cref{z} as follows:
\begin{equation}\label{r}
    \bUpsilon=\left\{\br=(\br_1,\ldots,\br_n):\, \br_i=(b_i,\bomega_i), \, b_i\in \R,\, \bomega_i\in \cS^{d-1},\, \mathcal{P}_i(\br_i)\cap \Omega\not= \emptyset \right\}.
\end{equation}

\begin{lem}\label{l:l-indep-sigma}
For fixed $\br\in \bUpsilon$, assume that the hyperplanes $\left\{\mathcal{P}_i(\br_i)\right\}_{i=1}^n$ are distinct. Then the set of functions $\left\{\sigma_i (\bx )\right\}_{i=0}^n$ is linearly independent in $\Omega$.
\end{lem}

\begin{proof}
The lemma may be proved in a similar fashion as that of Lemma~2.1 in \cite{LiuCai1}.
\end{proof}

\begin{lem}\label{l: l-indep}
Under the assumptions of \cref{l:l-indep-sigma}, the set of functions 
\[\{H_{i}(\bx),x_1H_{i}(\bx),\ldots,x_dH_{i}(\bx)\}_{i=1}^{n}\] is linearly independent in $\Omega$. 
\end{lem}
\begin{proof}
For each $i=1,\ldots,n$, the linear independence of $\{1,x_1,\ldots,x_d\}$ implies that the set of functions
\[ 
     \big\{H_{i}(\bx),x_1H_{i}(\bx),\ldots,x_dH_{i}(\bx)\big\} = H_{i}(\bx) \{1, x_1, \ldots, x_d\}
\]
is linearly independent. Now, linear independence of $\big\{H_{i}(\bx),x_1H_{i}(\bx),\ldots,x_dH_{i}(\bx)\big\}_{i=1}^{n}$ in $\Omega$ follows from the assumptions on the hyperplanes. 
\end{proof}

\section{Continuous least-squares optimization problems}\label{sec: prob}

Let $\|\cdot\|_\mu$ denote the weighted $L^2(\Omega)$ norm defined by
\[
\|v\|_\mu=\left(\int_\Omega \mu(\bx) v^2(\bx)\,d\bx\right)^{1/2}.
\]
Given function $u(\bx)$ defined in $\Omega$, we define the corresponding least-squares functional as
\[
\mathcal{J}_\mu(v) = \dfrac{1}{2} \|v-u\|^2_\mu.
\]
The best LS approximation to $u(\bx)$ within the NN function class $\cM_n(\Omega)$ is then obtained by solving
\begin{equation}\label{min}
  u_n(\bx;\hat{\bc}^*,\br^*) =\argmin_{v\in\cM_n(\Omega)}\mathcal{J}_\mu(v) =\argmin_{\hat{\bc}\in {\scriptsize\R}^{n+1},\, \br\in {\scriptsize\bUpsilon}}\mathcal{J}_\mu(u_n(\cdot;\hat{\bc},\br)),
\end{equation}
where $u_n(\bx;\hat{\bc},\br)$ has the form of
\begin{equation}\label{u_n}
u_n(\bx)=u_n(\bx;\hat{\bc},\br)
= c_0+\sum_{i=1}^{n}c_i\sigma(\br_i\cdot\by).
\end{equation}

Problem \cref{min} is a classic SNLS problem (see, e.g., \cite{osborne2007separable} and references therein) which is a special case of the broader class of NLS problems. The NLS problems are often addressed using variants of the GN methods, which exploit the underlying quadratic structure of the objective function \cite{dennis1996numerical, ortega2000iterative}. One of the most widely used GN variants is the LM algorithm \cite{Levenberg, Marquardt}, which modifies the GN method by adding a regularization (or shifting) term to ensure matrix invertibility. The LM update rule is given by:
\begin{equation}\label{LM}
    \btheta^{(k+1)}=\btheta^{(k)} - \gamma_{k+1} \left[G\left(\btheta^{(k)}\right) +\lambda_k I \right]^{-1} \nabla_{\btheta}\mathcal{J}_\mu\left(u_n\big(\cdot; \btheta^{(k)}\big)\right),
\end{equation}
where $\btheta = (\hat{\bc}^T, \br^T)^T$, $\gamma_{k+1}\in \R_{+}$ is the step size, $\lambda_k>0$ is the shifting/damping/regularization parameter. The matrix $G\left(\btheta\right)$ is the GN approximation to the Hessian of the loss functional and is given by:
\begin{equation}\label{GN_LM}
G\left(\btheta\right) = \nabla_{\btheta}\mathcal{J}_\mu\left(u_n\big(\cdot; \btheta^{(k)}\big)\right)^T\nabla_{\btheta}\mathcal{J}_\mu\left(u_n\big(\cdot; \btheta^{(k)}\big)\right).
\end{equation}
Although $G(\btheta)$ is always symmetric and positive semi-definite for all $\btheta$, it may be singular, making its inversion problematic. The LM algorithm addresses this by introducing a shift $\lambda_k I$, but this modification comes with trade-offs, primarily concerning the selection of $\lambda_k$. Various heuristic strategies have been proposed for choosing $\lambda_k$, but no universally optimal strategy exists. In practice, users must carefully tune the shifting parameter for each specific problem. For instance, the implementation of the LM algorithm in the built-in MATLAB function \texttt{lsqnonlin} requires users to configure several shifting-related parameters, such as the initial damping factor and the number of inner iterations used for adjusting it. 

The NLS problem in \cref{min} is separable, as the NN function $u_n(\bx)\in \cM_n(\Omega)$ given in \cref{u_n} is a linear combination of neurons that depends on the nonlinear parameters. SNLS problems have been studied by many researchers since the 1970s, and two principal solution strategies have emerged to exploit this structure. One approach leverages the separability by alternating between updates of the linear parameters $\hat{\bc}$ and the nonlinear parameter $\br$. Specifically, at each iteration, the parameter pair$\left( \br^{(k+1)} ,\hat{\bc}^{(k+1)}\right)$ is computed using a block Gauss-Seidel procedure:
\begin{itemize}
 \item[(1)] Solve the minimization problem in \cref{min} with fixed $\hat{\bc} = \hat{\bc}^{(k)}$, i.e., compute ${\br}\in \bUpsilon$ such that 
\begin{equation}\label{nonlinear}
    {\br}^{(k+1)} = \argmin_{\br\in \bUpsilon} \mathcal{J}_\mu\left(u_n\big(\cdot; \hat{\bc}^{(k)},\br\big)\right).
    \end{equation}
    
    \item[(2)] Solve the minimization problem in \cref{min} with given $\br = \br^{(k+1)}\in \bUpsilon$, i.e., compute $\hat{\bc}\in \R^{n+1}$ such that 
    \begin{equation}\label{linear-min}
    \hat{\bc}^{(k+1)} = \argmin_{\hat{\bc}\in \R^{n+1}} \mathcal{J}_\mu\left(u_n\big(\cdot; \hat{\bc},\br^{(k+1)}\big)\right).
    \end{equation}
\end{itemize}

This alternating scheme has also been adopted in machine learning settings for discrete least-squares problems (see \cref{min-d}). For instance, in \cite{varproj2}, it was applied to deep neural networks, where the nonlinear subproblem \eqref{nonlinear} was addressed using gradient descent over$\br \in \R^{n(d+1)}$.

An alternative approach is to solve $\hat{\bc}$ in terms of $\br$ and substitute it back into the loss functional to get a minimization problem with fewer unknowns. This leads to a reduced optimization problem involving only the nonlinear variables. The resulting method, known as the Variable Projection (VarPro) method, was introduced in \cite{varproj}. When the Gauss–Newton method is used to solve the reduced problem, it is referred to as the VarProGN method (see \cite{varproj}). The advantage of VarProGN depends on the problem structure. When the number of linear parameters is significantly larger than the number of nonlinear parameters, VarProGN can yield a lower-dimensional optimization problem and potentially improve efficiency. However, when the parameter counts are comparable,  or when nonlinear parameters dominate, the reduced problem may introduce unnecessary complexity into the nonlinear structure. In such cases, the alternating (block Gauss–Seidel) approach may be preferable, especially when both linear and nonlinear subproblems can be solved effectively.

\section{A structure-guided Gauss-Newton (SgGN) method}\label{sec: training}

In this section, we introduce our SgGN method for solving the minimization problem in (\ref{min}), guided by structures of both the separable least squares and the ReLU NN architecture. Specifically, we adopt the alternating method described in the previous section for the outer iteration. In particular, by exploiting the algebraic structures of the minimization problem in \cref{nonlinear},  we develop a modified GN method that explicitly reveals possible singularities of the GN matrix.

\subsection{Optimality condition}\label{s:OC}

Here we use the optimality condition to derive the corresponding systems of nonlinear algebraic equations. To this end, let 
\[
\bSigma(\bx;\br)=(\sigma_1(\bx),\ldots,\sigma_n(\bx))^T \quad\mbox{and}\quad \hat{\bSigma}(\bx;\br) = (\sigma_0(\bx), \sigma_1(\bx),\ldots, \sigma_n(\bx))^T, 
\]
where $\sigma_i(\bx)=\sigma(\br_i\cdot\by)$ is defined in \cref{s-H}. Then we have 
\begin{equation}\label{u-S}
u_n(\bx)=u_n(\bx;\hat{\bc},\br) =\hat{\bSigma}(\bx;\br)^T\hat{\bc}.
\end{equation}
Let $u^*_n(\bx)=u_n(\bx;\hat{\bc}^*,\br^*)\in\cM_n(\Omega)$ be a solution of \cref{min}, then $(\hat{\bc}^*,\br^*)$ is a critical point of the loss function $\mathcal{J}_\mu(u_n(\cdot;\hat{\bc},\br))$. That is, $(\hat{\bc}^*,\br^*)$ satisfies the following system of algebraic equations
\begin{equation}\label{cps}
    {\bf 0}=\nabla_{\hat{\bc}} \mathcal{J}_\mu\left(u_n(\cdot;\hat{\bc}^*,\br^*)\right)
    \quad\mbox{and}\quad
    {\bf 0}=\nabla_{\br} \mathcal{J}_\mu\left(u_n(\cdot;\hat{\bc}^*,\br^*)\right),
\end{equation}
where $\nabla_{\hat{\bc}}$ and $\nabla_{\br}$ denote the gradients with respect to the respective parameters $\hat{\bc}$ and $\br$.

In the following, we derive specific forms of the algebraic equations in \eqref{cps}. By \cref{u-S} and the fact that $\nabla_{\hat{\bc}} u_n(\bx)= \hat{\bSigma}(\bx)$,
we have 
\begin{align*}
\nabla_{\hat{\bc}} \mathcal{J}_\mu(u_n(\cdot;\hat{\bc},\br)) = &\int_\Omega\mu(\bx)(u_n(\bx)-u(\bx)) \nabla_{\hat{\bc}} u_n(\bx)\,d\bx \\
=&\left(\int_\Omega \mu(\bx)\hat{\bSigma}(\bx)^T\hat{\bSigma}(\bx)\,d\bx\right)\hat{\bc} -\int_\Omega \mu(\bx)u(\bx) \hat{\bSigma}(\bx)\,d\bx.
\end{align*}
Hence, the first equation in (\ref{cps}) becomes 
\begin{equation}\label{linear}
   {\bf 0}=  \nabla_{\hat{\bc}} \mathcal{J}_\mu(u_n(\cdot;\hat{\bc}^*,\br^*)) =\mathcal{A}(\br^*) \,\hat{\bc}^* -\bff(\br^*),
\end{equation} 
where $\mathcal{A}(\br^*)$ and $\bff(\br^*)$ are respectively the {\it mass matrix} and the right-hand side vector given by
\begin{equation}\label{A-f}
\mathcal{A}(\br^*)= \int_\Omega \mu(\bx)\hat{\bSigma}(\bx;\br^*)^T\hat{\bSigma}(\bx;\br^*)\,d\bx \quad\mbox{and}\quad \bff(\br)= \int_\Omega \mu(\bx) u(\bx) \hat{\bSigma}(\bx;\br^*)\,d\bx.
\end{equation} 
Specifically, we have $\mathcal{A}(\br^*)=\left(a_{ij}\right)_{(n+1)\times (n+1)}$ and $\bff(\br^*)=\left(f_i\right)_{(n+1)\times 1}$ with 
\[
a_{ij}=\int_\Omega \mu(\bx) \sigma(\br_i^*\cdot\by)\sigma(\br_j^*\cdot\by)d\bx \quad\mbox{and}\quad f_i= \int_\Omega \mu(\bx) u(\bx) \sigma(\br_i^*\cdot\by)d\bx.
\]

\begin{lem}\label{l: A}
For fixed $\br\in \bUpsilon$, under the assumptions of \cref{l:l-indep-sigma}, the mass matrix $\mathcal{A}(\br)$ is symmetric  positive definite.
\end{lem} 

\begin{proof}
The symmetry of $\mathcal{A}(\br)$ is evident. For any $\bxi\in \R^{n+1}$, let $v(\bx)=\bxi^T\hat{\bSigma}(\bx)$. Then the positive definiteness of $\mathcal{A}(\br)$ is a direct consequence of the fact that $\bxi^T \mathcal{A}(\br) \bxi = \left\|v\right\|_\mu^2$ and \cref{l:l-indep-sigma}.
\end{proof}

Next, we calculate $\nabla_{\br} \mathcal{J}_\mu(u_n(\cdot;\hat{\bc},\br))$. To simplify the expression of formulas, we use the Kronecker product, denoted by $\otimes$, of two matrices.

Let 
\[
\bH=\bH(\bx)=\left(H_1(\bx),\ldots, H_n(\bx)\right)^T.
\]
For $i,j=1,\ldots,n$, the fact that 
\[
\nabla_{\br_i}\sigma_j\left(\bx\right) =\nabla_{\br_i}\sigma(\br_j\cdot\by)
=\left\{\begin{array}{ll}
  {\bf 0},   &  i\not=j,\\[2mm]
   H_j(\bx) \by,  & i=j
\end{array}
\right.
\]
implies 
\begin{equation}\label{grad-r}
   \nabla_{\br} u_n(\bx;\hat{\bc},\br) %=\left(D\, \bH(\bx)\right)\otimes \by 
   = \left(D(\bc)\otimes I_{d+1}\right) \left(\bH(\bx)\otimes \by\right),
\end{equation}
where $D(\bc)=\mbox{diag}(c_1,\ldots, c_n)$ is the diagonal matrix with the $i^{\mbox{\scriptsize th}}$-diagonal element $c_i$.

Denote a scaled gradient vector of $\mathcal{J}_\mu\left(u_n(\bx;\hat{\bc},\br)\right)$ with respect to $\br$ by
\begin{equation}\label{G}
    \bG(\hat{\bc},\br)=\int_\Omega \mu(\bx) \big(u_n(\bx)-u(\bx)\big) {\bH}(\bx)\otimes \by\,d\bx.
\end{equation}
Then it is easy to check that the second equation in \cref{cps} becomes
\begin{equation}\label{grad-r-J}
    {\bf 0}=\nabla_{\br}\mathcal{J}_\mu\left(u_n(\bx;\hat{\bc}^*,\br^*)\right)= \left(D(\bc^*)\otimes I_{d+1}\right)\bG(\hat{\bc}^*,\br^*).
\end{equation}

\subsection{Structure of the Gauss--Newton matrix}
To address possible singularities of the GN matrix of the NLS problem in \cref{nonlinear}, below we derive a factorized form of the GN matrix by making use of the shallow ReLU neural network structure. To this end, let $\delta_i(\bx)=  \delta\left({\br}_{i}\cdot\by\right)$ for $i=1,\ldots,n$, where $\delta(t)$ is the Dirac delta function defined in (\ref{H-delta}). Denote the $n\times n$ diagonal matrix with the $i^{\mbox{\scriptsize th}}$-diagonal element $\delta_i(\bx)$ by 
\[
\Lambda(\bx)=\mbox{diag}(\delta_1(\bx),\ldots, \delta_n(\bx)).
\]
For $i,j=1,\ldots,n$, it is straightforward to verify that
\[ \nabla_{\br_i}H_j\left(\bx\right)=\left\{\begin{array}{ll}
  {\bf 0},   &  i\not=j,\\[2mm]
   \delta_j(\bx) \by,  & i=j
\end{array}
\right.
\]
in the weak sense. This implies
\begin{equation}\label{grad H}
    \nabla_{\br}{\bH}(\bx)^T =  \Lambda(\bx) \otimes \by. 
\end{equation}
Let 
\[
\hat{\mathcal{H}}(\bc,\br)=\int_\Omega \mu (x)\left(u_n(\bx)-u(\bx)\right) \Lambda(\bx)
    \otimes \left(\by \by^T\right)\,d\bx.
\]
Introducing the {\em layer GN matrix}
\begin{equation}\label{l-GN}
    {\mathcal{H}}(\br)=\int_\Omega \mu(\bx)\left[{\bH(\bx)}\bH(\bx)^T \right]\otimes\left[\by\by^T\right]\,d\bx,
\end{equation}
the lemma below reveals the structure of the GN matrix.

\begin{lem}\label{l: H}
The Hessian matrix of $\mathcal{J}_{\mu}\left(u_n(\bx;\hat{\bc},\br)\right)$ with respect to $\br$ has the form of 
\begin{equation}\label{Ha}
    \nabla_{\br}\big(\nabla_{\br} \mathcal{J}_{\mu}\left(u_n(\bx;\hat{\bc},\br)\right)\big)^T =\mathcal{G}(\bc,\br) +\hat{\mathcal{H}}(\bc,\br) \big(D(\bc)\otimes I_{d+1}\big) ,
\end{equation} 
where $\mathcal{G}(\bc,\br)$ is the GN matrix with respect to $\br$ given by
\begin{equation}\label{H}
    \mathcal{G}(\bc,\br) = \big(D(\bc)\otimes I_{d+1}\big) {\mathcal{H}}(\br)\big(D(\bc)\otimes I_{d+1}\big).
\end{equation}
\end{lem}

\begin{proof}
It follows from (\ref{grad-r-J}), the product rule, (\ref{grad-r}), and (\ref{grad H}) that
\begin{eqnarray*}\nonumber
    \nabla_{\br} \bG(\hat{\bc},\br)
    &=&\int_\Omega \mu \big(\nabla_{\br}u_n\big)\left({\bH}\otimes \by\right)^T\,d\bx+
    \int_\Omega \mu \left(u_n-u\right)\big(\nabla_{\br} {\bH}^T\big) \otimes\by^T\,d\bx \\[2mm]
    &=& \left(D(\bc)\otimes I_{d+1}\right)\int_\Omega \mu \big({\bH}\otimes \by\big)\big({\bH}\otimes \by\big)^T
    \,d\bx + \int_\Omega  \mu\left(u_n-u\right)
    \Lambda(\bx) \otimes\big(\by\by^T\big)\,d\bx ,
\end{eqnarray*}
which, together with \cref{grad-r-J} and the transpose rule of the Kronecker product, implies (\ref{Ha}). This completes the proof of the lemma.
\end{proof}

\begin{theorem}\label{l: Hspd}
Under the assumptions of \cref{l:l-indep-sigma}, the layer GN matrix ${\mathcal{H}}(\br)$ is symmetric positive definite. Accordingly, $\mathcal{G}(\bc,\br)$ is positive definite if and only if $c_i\not= 0$ for $i=1,\ldots,n$.
\end{theorem}

\begin{proof}
Based on Equation \cref{l-GN}, ${\mathcal{H}}(\br)$ is symmetric. For any $\bv^T=(\bbeta_1^T, \ldots, \bbeta_{n}^T)\in \R^{n(d+1)}$ with $\bbeta_i\in \R^{d+1}$, let
\[
v(\bx)= \sum_{i=1}^{n}\big(\bbeta_i^T\by\big)H_i(\bx).
\]
It can be observed that
\[
\bv^T{\mathcal{H}}(\br)\bv=\bv^T\left(\int_\Omega \mu(\bx) \big({\bH}\otimes\by\big) \big({\bH}\otimes\by\big)^T\,d\bx\right)\bv
=\Vert v\Vert_\mu^2\ge 0,
\]
which, together with \cref{l: l-indep}, implies that ${\mathcal{H}}(\br)$ is positive definite. Next, the result on the positive definiteness of $\mathcal{G}(\bc,\br)$ then naturally follows.
This completes the proof of the theorem.
\end{proof}

With the assumptions of \cref{l:l-indep-sigma}, \cref{l: Hspd} indicates that the singularity of the GN matrix $\mathcal{G}(\bc,\br)$ is directly determined by $\bc$. Once any possible zero entries of $\bc$ are removed, $\mathcal{G}(\bc,\br)$ is guaranteed to be positive definite.

\subsection{SgGN method: Removing singularity without shifting\label{sub:sggn}}
Solving the NLS problem outlined in \cref{nonlinear} typically involves methods like the LM algorithm, which adds a shifting term $\lambda_k I$ to the GN matrix for handling potential singularity. However, determining an optimal value for $\lambda_k$ remains challenging in practice. 
Our SgGN method circumvents the need for such regularization by exploiting the factorized structure of $\mathcal{G}(\bc,\br)$, as established in \cref{l: H} and \cref{l: Hspd}. This structure provides explicit information about which parameters require updating, allowing for a more targeted approach that naturally handles singularity without artificial shifting. 

In particular, if a coefficient $c_i$ in the linear parameter vector $\bc$ is zero, the corresponding $i^{th}$ neuron makes no contribution to the current NN approximation $u_n\left(\bx; \hat{\bc}^{(k+1)},\br^{(k)}\right)$. This is mathematically reflected by the corresponding $i^{th}$ component in the optimality condition in \cref{grad-r-J}, which automatically holds since the corresponding gradient component is zero. Consequently, the associated nonlinear parameter $\br_i^{(k)}$ does not require updating.

\begin{remark}
In the case where $c_i=0$, the LM algorithm still computes an search direction for updating $i^{th}$ neuron through introducing shifting. This, while effectively solving a perturbed version of the problem, fundamentally distorts the original mathematical structure and generates an artificial search direction, one that is more influenced by the choice of the shifting parameter. To illustrate why this may be problematic, consider a simple example where the GN matrix takes the form $\begin{bmatrix}
    d_1 &\\ &0
\end{bmatrix}$. The LM method modifies this matrix by adding a multiple of the identity and solves the following system::
\[\left(\begin{bmatrix}
    d_1 &\\ &0
\end{bmatrix}+\lambda I\right)\begin{bmatrix}
    p_1\\ p_2
\end{bmatrix}=\begin{bmatrix}
    g_1\\ g_2
\end{bmatrix}.
\]
The resulting solution components are
\[
p_1=\frac{g_1}{d_1+\lambda},\quad p_2=\frac{g_2}{\lambda}.
\]
In order for the first component $p_1$ to be close to the true answer $\dfrac{g_1}{d_1}$ (without shifting), it is desirable to use small $\lambda$. However, doing so would lead to an update $p_2$ with a large magnitude, although there is supposed to be no update along that direction.
\end{remark}

Here, our structure of the GN matrix $\mathcal{G}(\bc,\br)$ enables us to explicitly remove the singularity. That is, we identify and exclude neurons with $c_i=0$ from the update process based on our explicit matrix structure as in  \cref{l: H} and \cref{l: Hspd}. In this way, we are able to formulate a reduced positive definite system that focuses solely on the relevant parameters. Specifically, let $\tilde{\bc}$ represent the subset of linear parameters with nonzero values, with $\tilde{\br}$ and $\tilde{\bH}(\bx)$ respectively denoting their corresponding nonlinear parameters and Heaviside functions. We construct a diagonal matrix $D(\tilde{\bc})$ containing only the nonzero elements $c_i\not=0$, and let $\tilde{\cH}(\tilde{\br})$ be the corresponding layer GN matrix for these active parameters. This allows us to define the {\em reduced GN matrix}
\begin{equation}\label{rGN}
    \tilde{\mathcal{G}}(\tilde{\bc},\tilde{\br}) = \big(D(\tilde{\bc})\otimes I_{d+1}\big) {\tilde{\mathcal{H}}}(\tilde{\br})\big(D(\tilde{\bc})\otimes I_{d+1}\big).
\end{equation}
With this formulation, one step of our modified GN method for solving \cref{nonlinear} becomes
\begin{eqnarray*}\nonumber
   \tilde{\br}^{(k+1)}&=&\tilde{\br}^{(k)}-\gamma_{k+1} \tilde{\mathcal{G}}^{-1}\left(\tilde{\bc}^{(k)},\tilde{\br}^{(k)}\right) \left(D\left(\tilde{\bc}^{(k)}\right)\otimes I_{d+1}\right) \tilde{\bG}\left(\hat{\bc}^{(k)}, {\br}^{(k)}\right) \\[2mm] \label{mGN} &=&  \tilde{\br}^{(k)}-\gamma_{k+1} \left(D^{-1}\left(\tilde{\bc}^{(k)}\right)\otimes I_{d+1}\right) %\big(D^{-1}(\tilde{\bc}^{(k)})\otimes I_{d+1}\big) 
   {\tilde{\mathcal{H}}}^{-1}\left(\tilde{\br}^{(k)}\right) \tilde{\bG}\left(\hat{\bc}^{(k)}, {\br}^{(k)}\right), 
\end{eqnarray*}
where the scaled gradient vector
\begin{equation}
    \tilde{\bG}(\hat{\bc},\br)=\int_\Omega \mu(\bx) \big(u_n(\bx)-u(\bx)\big) \tilde{\bH}(\bx)\otimes \by\,d\bx\label{eq:gtilde}
\end{equation}  incorporates only the active Heaviside functions. This selective update approach naturally maintains positive definiteness without requiring artificial regularization.

Following \cref{Ha}, we can see that the GN method for solving the nonlinear system \cref{grad-r-J} involves the solution of a linear system with coefficient matrix $\mathcal{G}(\bc,\br)$ in the factorized form that explicitly reveals its singularity. With this insight, we now describe the SgGN method. The algorithm starts with an initial function approximation $u_n^{(0)}(\bx)=u_n\left(\bx;\hat{\bc}^{(0)}, \br^{(0)}\right)$ by  initializing the nonlinear parameters $\br^{(0)}$, as they define the breaking hyperplanes that partition the domain, effectively forming a computational ``mesh'' for our approximation. Given $\br^{(0)}$, we compute the optimal linear parameters $\hat{\bc}^{(0)}$ on the current physical partition by solving 
\begin{equation}\label{initial-c}
    \mathcal{A}\big(\br^{(0)}\big)\,\hat{\bc}^{(0)} = \bff\big(\br^{(0)}\big).
\end{equation} 
For a detailed discussion of this initialization strategy, see \cite{LiuCai1, Cai2021linear}.
Then, given the approximation $u_n^{(k)}(\bx)=u_n\left(\bx;\hat{\bc}^{(k)}, \br^{(k)}\right)$ at the $k^{th}$ iteration, the process of obtaining 
\[
u_n^{(k+1)}(\bx)=u_n\left(\bx;\hat{\bc}^{(k+1)}, \br^{(k+1)}\right)
\]
proceeds as follows:
\begin{itemize}
\item[(i)] First, identify a set of {\em active} neurons in the current approximation $u_n^{(k)}(\bx)$. This is done by examining the linear weights $\bc^{(k)}$ associated with each neuron. Neurons whose corresponding weight magnitude $\left|{c}_i^{(k)}\right|$ meets or exceeds a threshold $\epsilon_{\bc}$ are considered active. This defines the active set $\mathcal{I}_{\text{active}}$:
\begin{equation}
\mathcal{I}_{\text{active}} = \left\{ i \in \{1, \dots, n\} \left|\; \left|{c}_i^{(k)}\right| \right.\ge \epsilon_{\bc} \right\}.\label{eq:active}
\end{equation}
The parameters corresponding to these active neurons, denoted as $\tilde{\bc}^{(k)}$ and $\tilde{\br}^{(k)}$, are extracted to form a reduced system. %for refinement.

\item[(ii)] Next, construct and solve a reduced GN system using only the parameters associated with the active set identified in step (i). Unlike traditional methods that artificially modify the entire system, we form $\tilde{\bG}(\hat{\bc}^{(k)},\tilde{\br}^{(k)})$, $D(\tilde{\bc}^{(k)})$, and $\tilde{\mathcal{H}}(\tilde{\br}^{(k)})$, then solve
\begin{equation}\label{GN2}
{\tilde{\mathcal{H}}}\left({\tilde{\br}}^{(k)}\right){\tilde{\bs}}^{(k+1)} =  \tilde{\bG}(\hat{\bc}^{(k)},\tilde{\br}^{(k)})
\end{equation}
to obtain an intermediate search direction $\tilde{\bs}^{(k+1)}$. This focused approach ensures the system remains positive definite.
% without artificial regularization. 
After obtaining the solution, the final search direction in the reduced space is computed by scaling as
\begin{equation*}
{\tilde{\bp}}^{(k+1)} = \left({D}^{-1} \left({\tilde{\bc}}^{(k)}\right)\otimes I_{d+1}\right){\tilde{\bs}}^{(k+1)},
\end{equation*}
and then map it back to the full parameter space by initializing $\bp^{(k+1)} = \bm{0}$ and setting $\bp^{(k+1)}_i = \tilde{\bp}^{(k+1)}_i$ only for indices $i \in \mathcal{I}_{\text{active}}$. This selective update strategy explicitly excludes inactive parameters, fundamentally addressing the mathematical structure of the problem rather than using shifting.

\item[(iii)] Then, determine the optimal step size $\gamma_{k+1}$ along the computed search direction $\bp^{(k+1)}$ by minimizing the line search objective function $\mathcal{J}_\mu$:
\begin{equation*}
\gamma_{k+1} = \underset{\gamma \in {\R^{+}_{0}} }{\arg\min} \mathcal{J}_\mu\left(u_n\big(\cdot;{\hat{\bc}}^{(k)},{\br}^{(k)}-\gamma {\bp}^{(k+1)}\big)\right),
\end{equation*}
which leads to the n\textbf{}onlinear parameter update
\begin{equation*}
{\br}^{(k+1)} = {\br}^{(k)} - \gamma_{k+1} {\bp}^{(k+1)}.
\end{equation*}
\item[(iv)] Finally, compute the optimal linear parameters $\hat{\bc}^{(k+1)}$ on the current physical partition by solving
\begin{equation}\label{GN1}
    \mathcal{A}\big(\br^{(k+1)}\big) \,\hat{\bc}^{(k+1)}=\bff\big(\br^{(k+1)}\big).
\end{equation} 
These updated parameters $\hat{\bc}^{(k+1)}$ and ${\br}^{(k+1)}$ together define the improved function approximation $u_n^{(k+1)}(\bx)=u_n\left(\bx;\hat{\bc}^{(k+1)},\br^{(k+1)}\right)$.
\end{itemize}

\smallskip

This iterative process continues until a desired level of accuracy or a maximum number of iterations is reached, yielding the final approximation function. See \cref{alg:iterations} for a pseudocode summary of the SgGN method.

\begin{algorithm}
    \caption{A structure-guided Gauss--Newton (SgGN) method for (\ref{min})}\label{alg:iterations}
    \begin{algorithmic}
    \REQUIRE{Initial approximation function $u_n^{(0)}(\bx) = u_n\left(\bx;\hat{\bc}^{(0)},{\br}^{(0)}\right)$, target function $u(\bx)$, coefficient threshold $\epsilon_{\bc}>0$, density function $\mu(\bx)$}
    \ENSURE{Optimized approximation function $u_n^{(K)}(\bx) = u_n\left(\bx;\hat{\bc}^{(K)},{\br}^{(K)}\right)$}
    \FOR{$k=0,1,\ldots, K-1$}
    
    \STATE $\triangleright$ \textit{Identify active parameters $\tilde{\bc}^{(k)}$ and $\tilde{\br}^{(k)}$ based on the current approximation $u_n^{(k)}(\bx) = u_n\left(\bx;\hat{\bc}^{(k)},{\br}^{(k)}\right)$}
    \smallskip
    \STATE Form active index set \cref{eq:active},
    \STATE Extract active linear parameters $\tilde{\bc}^{(k)} \leftarrow \left(c_i^{(k)}\right)_{i \in \mathcal{I}_{\text{active}}}$ and,
    \STATE Extract active nonlinear parameters $\tilde{\br}^{(k)} \leftarrow \left(\br_i^{(k)}\right)_{i \in \mathcal{I}_{\text{active}}}$.

    \smallskip
    
    \STATE $\triangleright$ \textit{Update nonlinear parameters $\br^{(k)}$}
    \smallskip
    \STATE Form the reduced matrices $D\left(\tilde{\bc}^{(k)}\right)$ and $\tilde{\mathcal{H}}\left(\tilde{\br}^{(k)}\right)$ in \eqref{rGN} and vector $\tilde{\bG}\left(\hat{\bc}^{(k)},\tilde{\br}^{(k)}\right)$ in \eqref{eq:gtilde},
    \STATE{Compute the reduced intermediate search direction ${\tilde{\bs}}^{(k)}\leftarrow {\tilde{\mathcal{H}}}\left({\tilde{\br}}^{(k)}\right){\tilde{\bs}}^{(k)} = - \tilde{\bG}(\hat{\bc}^{(k)},\tilde{\br}^{(k)})$},
    \STATE{Compute the reduce final search direction by scaling ${\tilde{\bp}}^{(k)} \leftarrow\left({D}^{-1} \left({\tilde{\bc}}^{(k)}\right)\otimes I_{d+1}\right){\tilde{\bs}}^{(k)}$},
    \STATE Initialize full search direction $\bp^{(k)}\leftarrow \bm{0}$
    \FOR{$i\in \mathcal{I}_{\text{active}}$}
    \STATE $\bp^{(k)}_i = \tilde{\bp}^{(k)}_i$
    \ENDFOR
    \STATE{Compute the step size $\gamma_{k+1} \leftarrow\underset{\gamma \in {\R^{+}_{0}} }{\arg\min} \mathcal{J}_\mu\left(u_n\big(\cdot;{\hat{\bc}}^{(k)},{\br}^{(k)}-\gamma {\bp}^{(k)}\big)\right)$},
    \STATE{Update nonlinear parameter $\br^{(k+1)} \leftarrow \br^{(k)} - \gamma_{k+1} \bp^{(k)}$}.

    \smallskip

    \STATE $\triangleright$ \textit{Update linear parameters $\hat{\bc}^{(k)}$}

    \smallskip
    
    \STATE Form the mass matrix $\mathcal{A}\left(\br^{(k)}\right)$ and  the right-hand side vector ${\bff\left(\br^{(k)}\right)}$ using equation \eqref{A-f},
    \STATE{Compute linear parameters $\hat{\bc}^{(k+1)}\leftarrow\mathcal{A}\left(\br^{(k)}\right)\hat{\bc}^{(k+1)} =\bff \left(\br^{(k)}\right)$}.
    \STATE $\triangleright$ \textit{Form the new approximation $u_n^{(k+1)}(\bx) = u_n^{(k+1)}\left(\bx; \hat{\bc}^{(k+1)}, \br^{(k+1)}\right)$}
    \smallskip
        \STATE Form the new approximation $u_n^{(k+1)}(\bx) = {c}_0^{(k+1)} + \sum\limits_{i=1}^n {c}_i^{(k+1)} \sigma\left(\br_i^{(k+1)}\cdot\by\right)$ .
        \smallskip
    \IF{a desired loss or a specified number of iterations is reached}
    \STATE Return $u_n^{(k+1)}(\bx) = u_n\left(\bx;\hat{\bc}^{(k+1)}, \br^{(k+1)}\right)$
    \ENDIF
\ENDFOR
    \end{algorithmic}
\end{algorithm}

\subsection{Practical considerations} \label{subsec:conditioning}

This section discusses some practical issues in the implementation of the SgGN method.

 The SgGN method needs to solve the linear systems in \eqref{GN1} and \eqref{GN2} during the iterations. Since the focus of this work is on designing the SgGN optimization method, here we just briefly mention some numerical issues related to these linear systems and leave the details to a forthcoming paper \cite{SgGN2}.
 The linear systems may be solved with direct or iterative solvers.
If $\br^{(k)}$ satisfies the assumption in Lemma~\ref{l: l-indep}, both the matrices $\mathcal{A}\left(\br^{(k)}\right)$ and $\mathcal{H}\left(\br^{(k)}\right)$ are symmetric positive definite (see \cref{l: A} and \cref{l: Hspd}). Nevertheless, both of them can be very ill conditioned. Take the 1D case as an example. When the $n$ breaking points are uniformly distributed over $\Omega = [0,1]$, the mass matrix $\mathcal{A}$ and the layer GN matrix $\mathcal{H}$ have $2$-norm condition numbers \cite{SgGN2,Qing22} 
\begin{equation}
    \kappa_2(\mathcal{A}) = O(n^4)\quad\mbox{and}\quad \kappa_2(\mathcal{H}) = O(n^2),
\end{equation}
respectively.
The condition numbers are even larger when there are clustered breaking points. A rigorous characterization of the conditioning is given in \cite{SgGN2}. In this paper, we use direct solvers for the purpose of verifying the convergence of the SgGN algorithm.  To accommodate highly ill-conditioned matrices, the direct inversion is done through truncated SVDs. This suffices our purpose of comparing the convergence of the SgGN with other methods.

The feasibility of designing more practical direct solvers based on structured methods will be discussed in \cite{SgGN2}. The SNLS problem in \cref{min} is a nonconvex optimization, and hence initialization is critical for the success of any optimization/iterative/training scheme. As discussed in \Cref{sec: net}, the nonlinear parameters $\br$ determine the basis functions $\left\{\sigma(\br_i\cdot\by)\right\}_{i=1}^n$ and hence the physical partition $\mathcal{K}(\br)$ of the domain $\Omega$, together with the linear parameters $\hat{\bc}$ serving as a NN approximation. Since the optimal configuration of this partition is generally unknown beforehand, a common and unbiased strategy (see, e.g., \cite{LiuCai1, Cai2021linear}) is to initialize $\br^{(0)}$ such that the corresponding hyperplanes uniformly partition the domain. The initial values of the linear parameters $\hat{\bc}^{(0)}$ are then the solution of \cref{min} with fixed $\br=\br^{(0)}$. This approach provides a reasonable starting point for the optimization process, which will subsequently adjust these hyperplanes to better capture the underlying function's features. However, this uniform partition strategy may not provide a good initial $\br^{(0)}$ for relatively large $n$. One may use the method of various continuations \cite{AlGe:90, LiuCai1, Cai2021linear, Cai2021nonlinear} for constructing a good initial; in particular, the adaptive neuron enhancement (ANE) method \cite{LiuCai1, LiuCai2} is a natural method of continuation with respect to the number of neurons.

Another practical aspect is the evaluation of integrals required by the SgGN method. These include the integrals forming the mass matrix $\mathcal{A}(\br)$, the right-hand side vector $\bff(\br)$, the layer GN matrix $\mathcal{H}(\br)$, and the gradient term $\bG(\hat{\bc},\br)$. These integrals are typically computed numerically using a suitable quadrature rule.
A commonly used approach is the composite midpoint rule applied over a uniform partition of the domain, which provides a simple and effective means of approximating integrals with reasonable accuracy. In our numerical experiments (see \Cref{sec: numerics}), we adopted this rule to evaluate all integrals involved in the SgGN method, including those defining the loss functional.
To mitigate this trade-off, one may employ adaptive quadrature strategies (see, e.g., \cite{LiCaRa23}), which selectively refine the quadrature points in regions of interest. Such approaches can reduce the total number of quadrature points while maintaining a comparable level of accuracy to uniform methods, thereby improving overall efficiency.

\section{SgGN for discrete least-squares optimization problems}\label{sec: discrete}
Building upon its formulation for continuous LS problems, this section details the adaptation of the SgGN method for discrete LS problems.

Consider a given discrete data set $\left\{\left(\bx^{i}, u^{i}\right)\right\}_{i=1}^m$, where each $\bx^i \in \Omega$ is an input point with a corresponding target function value $u^i\in\R$. An associated distribution function $0 \leq \mu(\bx) \leq 1$ is also defined for these points. The objective is to find a NN function $u^*_n(\bx)=u_n(\bx;\hat{\bc}^*,\br^*)\in \cM_n(\Omega)$ of the form in \cref{u_n} %from the model class $\cM_n(\Omega)$ 
that solves the discrete LS minimization problem:
\begin{equation}\label{min-d}
u_n(\bx;\hat{\bc}^*,\br^*) =\argmin_{v\in\cM_n(\Omega)}\mathcal{J}_{m,\mu}(v) =\argmin_{\hat{\bc}\in {\scriptsize\R}^{n+1},\, \br\in {\scriptsize\bUpsilon}}\mathcal{J}_{m,\mu}(u_n(\cdot;\hat{\bc},\br)),
\end{equation}
where $\mathcal{J}_{m,\mu}(v)$ is the weighted discrete LS loss function given by
\begin{align*}
\mathcal{J}_{m,\mu}(v) = \dfrac{1}{2} \sum_{i=1}^{m}\mu(\bx^{i})\left(v(\bx^{i}) - u^i\right)^2 = \frac{1}{2}\|v - u\|^2_{m,\mu}
\end{align*}
and $\|\cdot\|_{m,\mu}$ denotes the weighted discrete $L^2(\Omega)$ norm defined as
\begin{equation*}
    \|v\|_{m,\mu} = \left(\sum_{i=1}^{m}\mu(\bx^{i})v^2(\bx^{i})\right)^{\frac{1}{2}}.
\end{equation*}

The SgGN framework (\cref{alg:iterations}) extends naturally to this discrete problem in \cref{min-d}. The core structural insights and the alternating optimization strategy between linear and nonlinear parameters are preserved. The primary adaptation involves replacing the integral-defined matrices and vectors from \Cref{sec: training} with their discrete analogs, formed by summations over the data set. Specifically, the mass matrix and the right-hand side vector are given by
\begin{align*}
\mathcal{A}(\br) = \sum_{i=1}^{m} \mu(\bx^i)\, \hat{\bSigma}(\bx^i) \,\hat{\bSigma}(\bx^i)^T \quad \mbox{and}\quad
\bff(\br) = \sum_{i=1}^{m} \mu(\bx^i)\, u^i\, \hat{\bSigma}(\bx^i);
\end{align*}
the scaled gradient vector of $\mathcal{J}_{m,\mu}(u_n(\cdot;\hat{\bc},\br))$ with respect to $\br$ is 
\begin{equation*}
    \bG(\hat{\bc},\br)=\sum_{i=1}^{m} \mu(\bx^{i}) \big(u_n(\bx^{i})-u^i\big)\, {\bH}(\bx^{i})\otimes \by^{i},
\end{equation*}
where $\by^i=\by(\bx^{i})= (1,x_1^i,\ldots,x_d^i)^T$; and the layer GN and the GN matrices are given by
\begin{equation*}
    {\mathcal{H}}(\br)=\sum_{i=1}^{m} \mu(\bx^i) \left({\bH}(\bx^i)\,{\bH}(\bx^i)^T\right) \otimes\left(\by^i(\by^i)^T\right) \, \mbox{ and }\, \mathcal{G}(\bc,\br)= \big(D(\bc)\otimes I_{d+1}\big) {\mathcal{H}}(\br)\big(D(\bc)\otimes I_{d+1}\big),
\end{equation*}
respectively. 

Since $\|\cdot\|_{m,\mu}$ defines a norm in $\mathbb{R}^m$, under the assumptions of \cref{l:l-indep-sigma}, we may show, in a similar fashion to those of \cref{l: A} and \cref{l: Hspd}, that $\mathcal{A}(\br)$ and ${\mathcal{H}}(\br)$ are positive definite, and that $\mathcal{G}(\bc,\br)$ is positive definite if and only if $c_i\not=0$ for all $i\in \{1,\ldots,n\}$. Similarly to the GN matrix structure discussed in \Cref{sub:sggn}, we can also effectively handle singularity in the discrete setting. With these discrete definitions of its core components, the SgGN method (\Cref{alg:iterations}) can be readily applied to solve discrete least-squares optimization problems, inheriting its notable structural advantages and robust performance characteristics from the continuous case.

\section{Numerical Experiments}\label{sec: numerics}
In this section, we present a series of numerical experiments to demonstrate the effectiveness and accuracy of the proposed SgGN algorithm. A key advantage of SgGN lies in its ability to naturally exploit the GN matrix structure, producing effective search directions without requiring artificial regularization techniques. 
% like in LM to ensure positive definiteness.
To specifically highlight this benefit, our first set of experiments (\Cref{sub:1dpiece}) includes a direct comparison between SgGN and LM, illustrating the performance advantage of SgGN over shifting-based LM.

Across all test cases, we benchmark the SgGN method against several widely used NN optimization algorithms, including the first-order method Adam \cite{kingma2015}, the quasi-Newton method BFGS \cite{broyden1970convergence,fletcher1970new,goldfarb1970family,shanno1970conditioning}, and the GN-based method KFRA\cite{botev2017practical}, which is considered more applicable than the earlier GN-based method KFAC \cite{martens2015optimizing}. This comprehensive comparison evaluates each algorithm's performance in terms of convergence speed, solution quality, and ability to address challenging function approximation tasks.

Our test problems span a range of function types, including step functions in one and two dimensions, a delta-like peak function in 1D, and a continuous piecewise linear function in 2D. These functions are well-suited for accurate approximations using shallow ReLU NNs. However, they pose significant challenges for optimization algorithms due to the presence of discontinuities or sharp transitions. 
As noted in \cite{Cai2021linear, LiuCai1}, the nonlinear parameters $\br$ correspond to the breaking points/lines of the neurons, which in turn form a physical partition of the domain. Therefore, an optimization algorithm’s effectiveness can also be assessed by its ability to reposition these breaking points/lines from their initial uniform distribution to the optimal configuration aligned with the features of the target function.

For consistency, we use BFGS as a baseline. For each test, BFGS is first repeated 30 times; and we report the median loss and the corresponding approximation result. The other methods (KFRA and SgGN) are run for the same number of iterations. For Adam, due to its slower convergence, we allow a significantly larger number of iterations, continuing until the loss function plateaus.

It is worth pointing out that different methods entail different per-iteration computational complexities. For example, the per-iteration cost of BFGS \cite{Wright} and KFRA \cite{botev2017practical} are both $O(n^2)$, while Adam has a lower cost of $O(n)$. The SgGN involves solving two dense linear systems, with coefficient matrices $\mathcal{A}(\br^{(k)})$ and $\mathcal{H}(\br^{(k)})$, respectively (see \cref{alg:iterations}).
These matrices are typically very ill-conditioned for the test problems considered here. In our current implementation, truncated SVDs are used for the solution, and its cost is $O(rn^2)$ with $r$ depending on the desired accuracy. Although we use the number of iterations as one of the reference metrics for assessing the efficiency, our primary focus in these experiments is on the quality of the solution. 
As demonstrated in the test cases, SgGN consistently converges to more accurate approximations than the other methods, even for problems involving sharp transitions and discontinuities.

The detailed parameter settings for each optimization method are provided in \cref{tab:para}. All methods are initialized with the same starting configuration, as described in the Initialization section in \cref{tab:para}. The integration of the loss function ${\cal J}_\mu \left(u_n\left(\cdot;\hat{\bc},\br\right)\right)$ is computed using the composite mid-point rule over a uniform partition with mesh size $h = 0.01$.

\begin{table}[!ht]
    \centering
    \caption{List of parameters used in the methods, where the parameters for BFGS and Adam follow deep learning toolbox \cite{MatlabDL}.}
    \label{tab:para}
    \begin{tabular}{|cl|}
        \hline

        \multicolumn{2}{|c|}{\textbf{BFGS}}\\
        \hline
       \mcode{net.trainParam.min_grad} & minimum performance gradient with value $0$\\
       \mcode{net.trainParam.max_fail} & maximum validation failures with value $10^4$\\
       \mcode{net.trainParam.epochs} & maximum number of epochs to train with value $10^4$\\
        \hline
        \multicolumn{2}{|c|}{\textbf{KFRA}}\\
        \hline
        \multirow{2}{*}{$\gamma$} &A damping parameter for the approximated Gauss-Newton\\
            & matrix induced by the full Gauss-Newton\\
        \hline
        \multicolumn{2}{|c|}{\textbf{Adam}}\\
        \hline
        \mcode{InitialLearnRate} & initial learning rate $\alpha_0$\\
        \mcode{DropRateFactor} & multiplicative factor $\alpha_f$  by which the learning rate drops\\
        \multirow{2}{*}{\mcode{DropPeriod}} & number of epochs that passes between adjustments to\\
            & the learning rate, denoted by $T$\\
        \hline\hline
        \multicolumn{2}{|c|}{\textbf{Initialization}}\\
        \hline
        Linear coefficient $\hat{\bc}$ & initialized by solving \cref{initial-c} for SgGN and by a narrow \\ & normal distribution $\mathcal{N}(0,0.01)$ for the other methods\\
        Nonlinear parameter $\br_i$ & the corresponding breaking hyperplanes uniformly partition\\ & the domain \\
        \hline
    \end{tabular}
\end{table}

\subsection{One-dimensional piecewise constant function\label{sub:1dpiece}}
The first test problem is a one-dimensional piece-wise constant function defined in the interval $[0,10]$, consisting of ten segments with values drawn from a skewed distribution(see \cref{fig1:step1D:problem}). This setup is designed to evaluate whether an optimizer can effectively relocate uniformly initialized breaking points to align with the discontinuities in the target function. Theoretically, a shallow ReLU NN with 20 neurons suffices to approximate a ten-piece step function to any prescribed accuracy $\epsilon >0$ \cite{Cai2021linear, Cai2023linear}. However, due to the non-convex nature of the optimization problem, we employ 30 neurons in our experiments to ensure sufficient model capacity.

To illustrate the structural advantages of the SgGN method discussed in \Cref{sub:sggn}, we begin with a direct comparison against the Levenberg–Marquardt (LM) algorithm, as implemented in MATLAB’s built-in routine \texttt{lsqnonlin}. We consider two LM configurations: (i) applying LM solely to the nonlinear parameters, and (ii) applying LM to all parameters simultaneously. In both cases, the LM-related settings are carefully tuned for optimal performance. The resulting approximations are shown in \Cref{subfig:LM-nl} and \Cref{subfig:LM-full}, respectively, and a comparative plot of the training loss versus that of SgGN is presented in \Cref{subfig:LM-loss}.

\begin{figure}[ptbh]
    \centering
    \subfigure[LM (nonlinear) $u_n$\label{subfig:LM-nl}]{
        \includegraphics[width=1.8in]{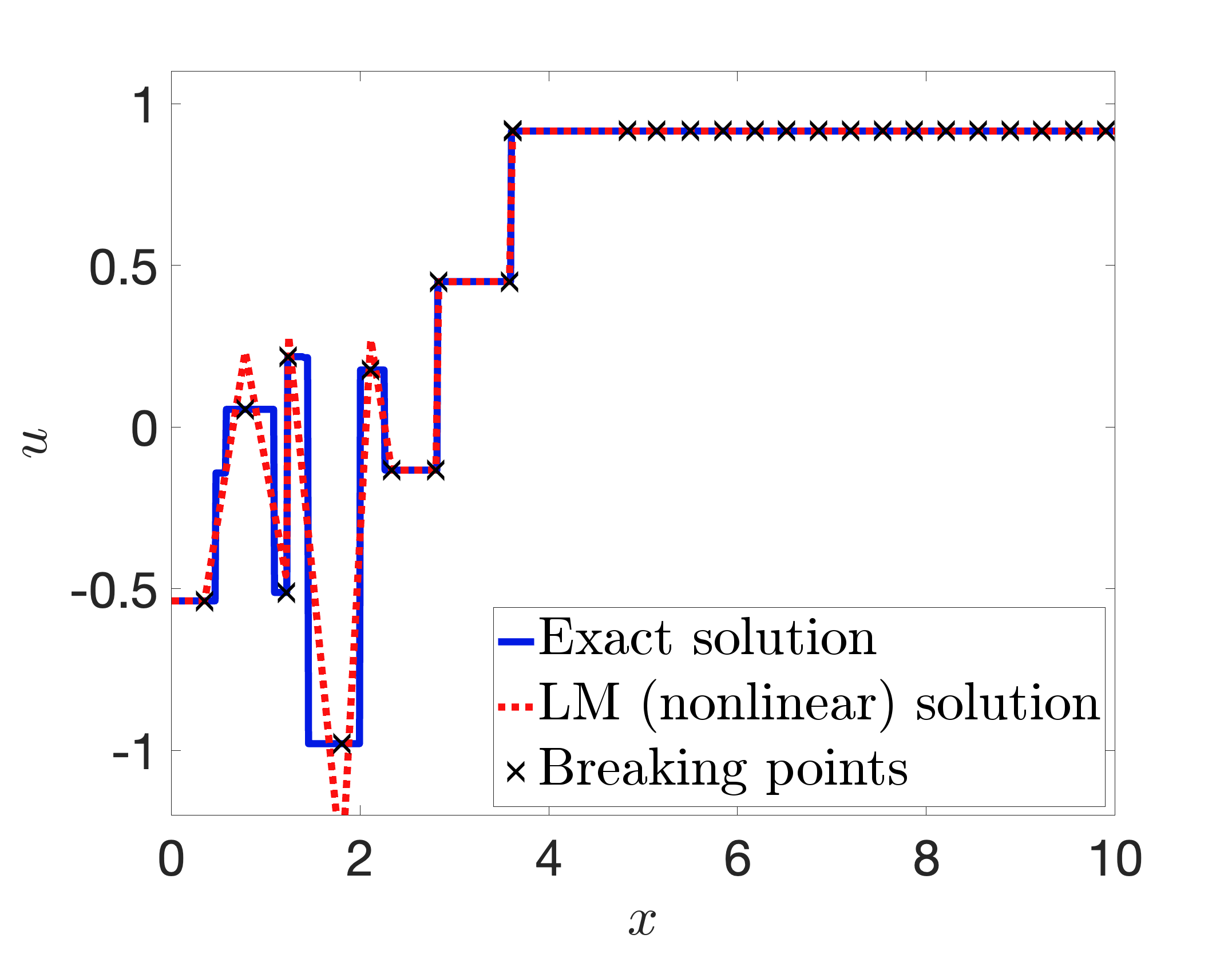}
    }
    \hfill
     \subfigure[LM (full) $u_n$\label{subfig:LM-full}]{
        \includegraphics[width=1.8in]{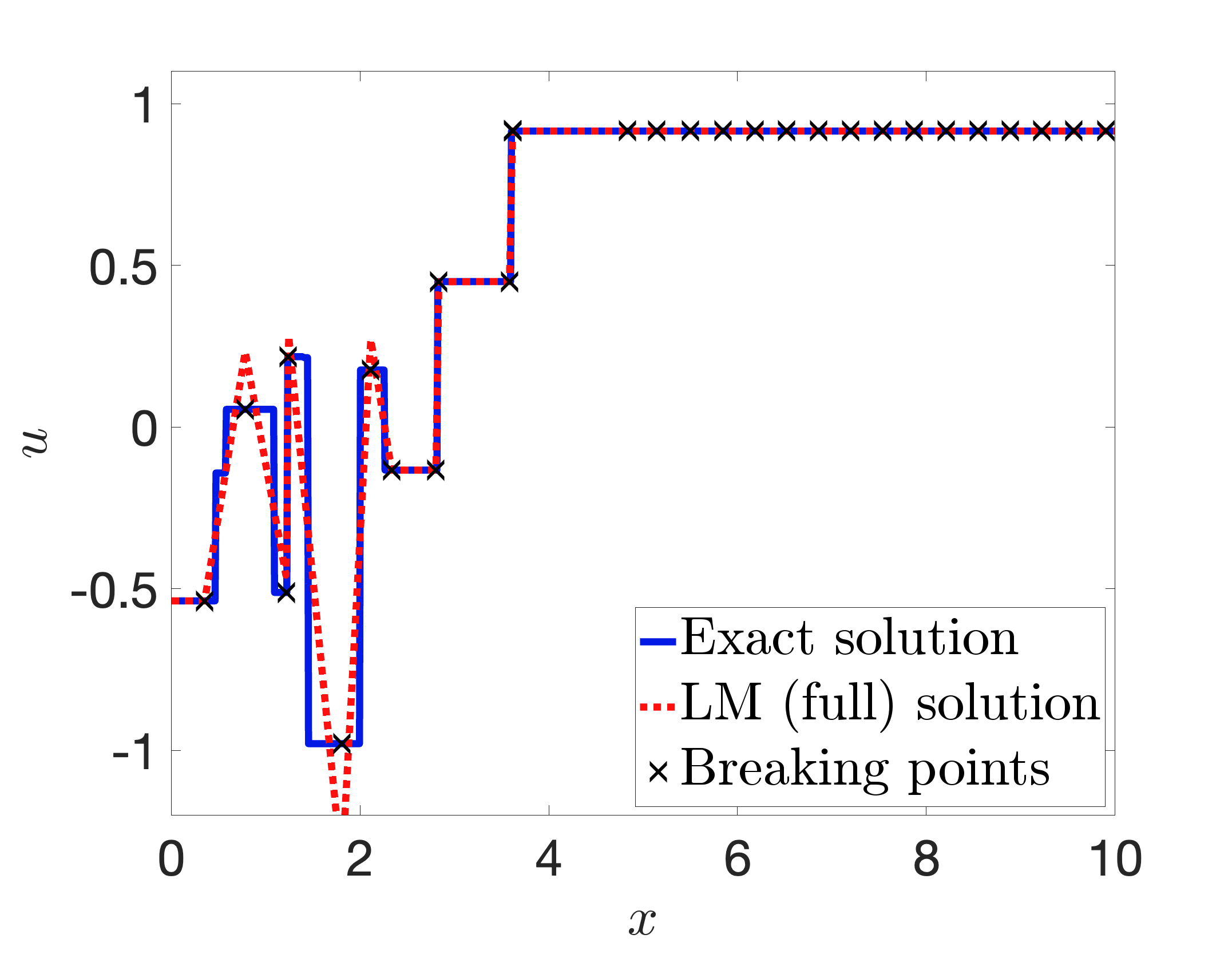}
    }
    \hfill
    \subfigure[Loss curves \label{subfig:LM-loss}]{
        \includegraphics[width=1.8in]{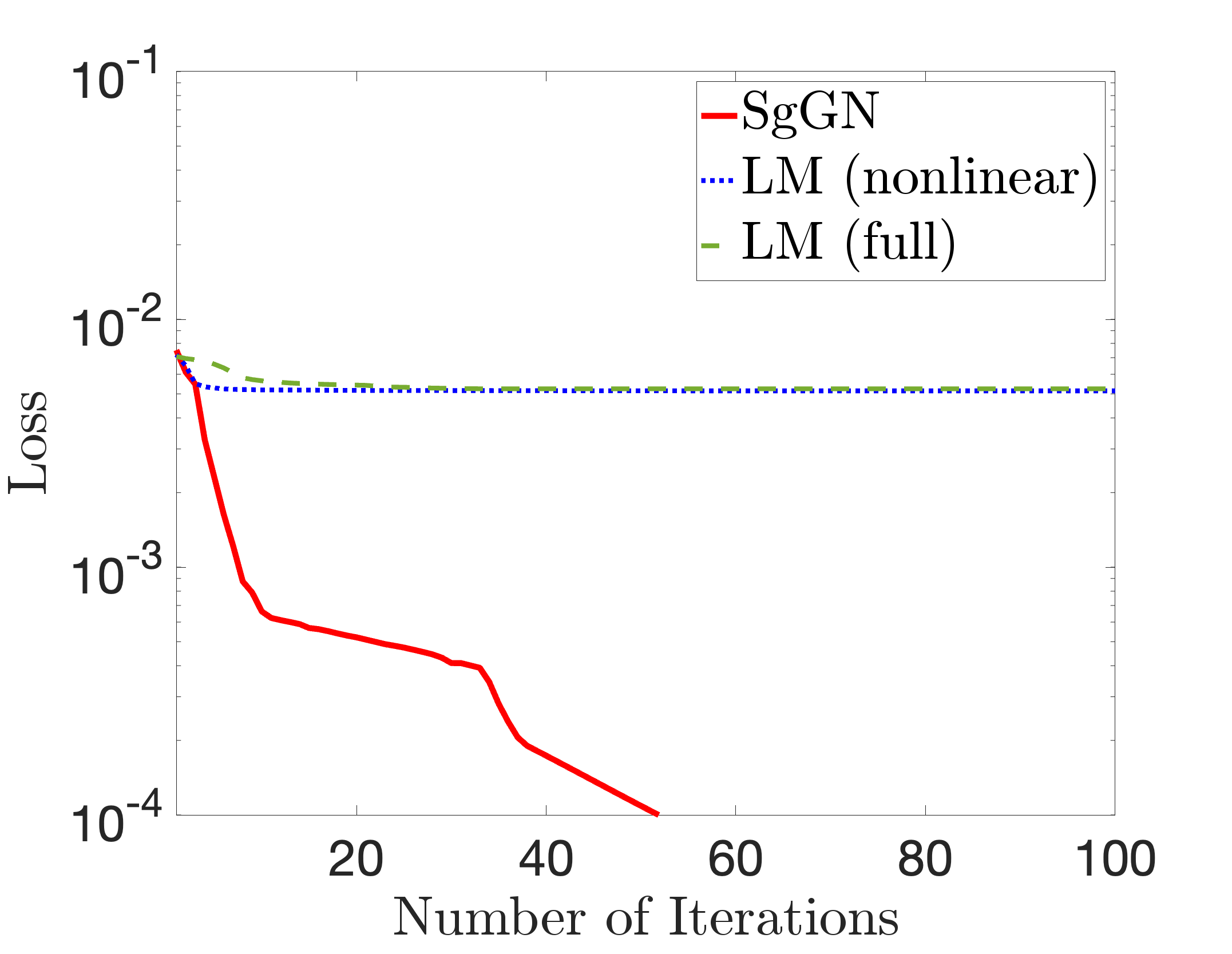}
    }
    \caption{Performance of the Levenberg-Marquardt (LM) method as compared with SgGN.}
    \label{fig:LM-performance}
\end{figure}

As shown in \Cref{fig:LM-performance}, the SgGN method achieves significantly lower training loss, highlighting the effectiveness of its structure-aware formulation.  In contrast, the LM algorithm relies on shift-based regularization to ensure matrix invertibility, which may hinder its ability to converge to high-accuracy solutions, especially in problems characterized by irregular discontinuities. 

Having established the structural advantages of SgGN over the LM method, we now turn our attention to a broader evaluation of SgGN against other widely adopted optimization techniques, including BFGS, Adam, and KFRA. To ensure a fair comparison, we carefully tune the learning rates and key hyper-parameters of each method to achieve their best possible performance on the test problems. Specifically, in this test problem, for the Adam optimizer, we use an initial learning rate of $\alpha_1= 0.1$, a drop factor $\alpha_f =0.5 $, and a drop period T=1000. For the KFRA method, we set the damping parameter to $\gamma=0.01$.

\begin{figure}[ptbh]
  \centering 
  \subfigure[Target function $u$ and initial breaking points as reflected by $\times$'s]{ 
    \label{fig1:step1D:problem} 
    \includegraphics[width=2.2in]{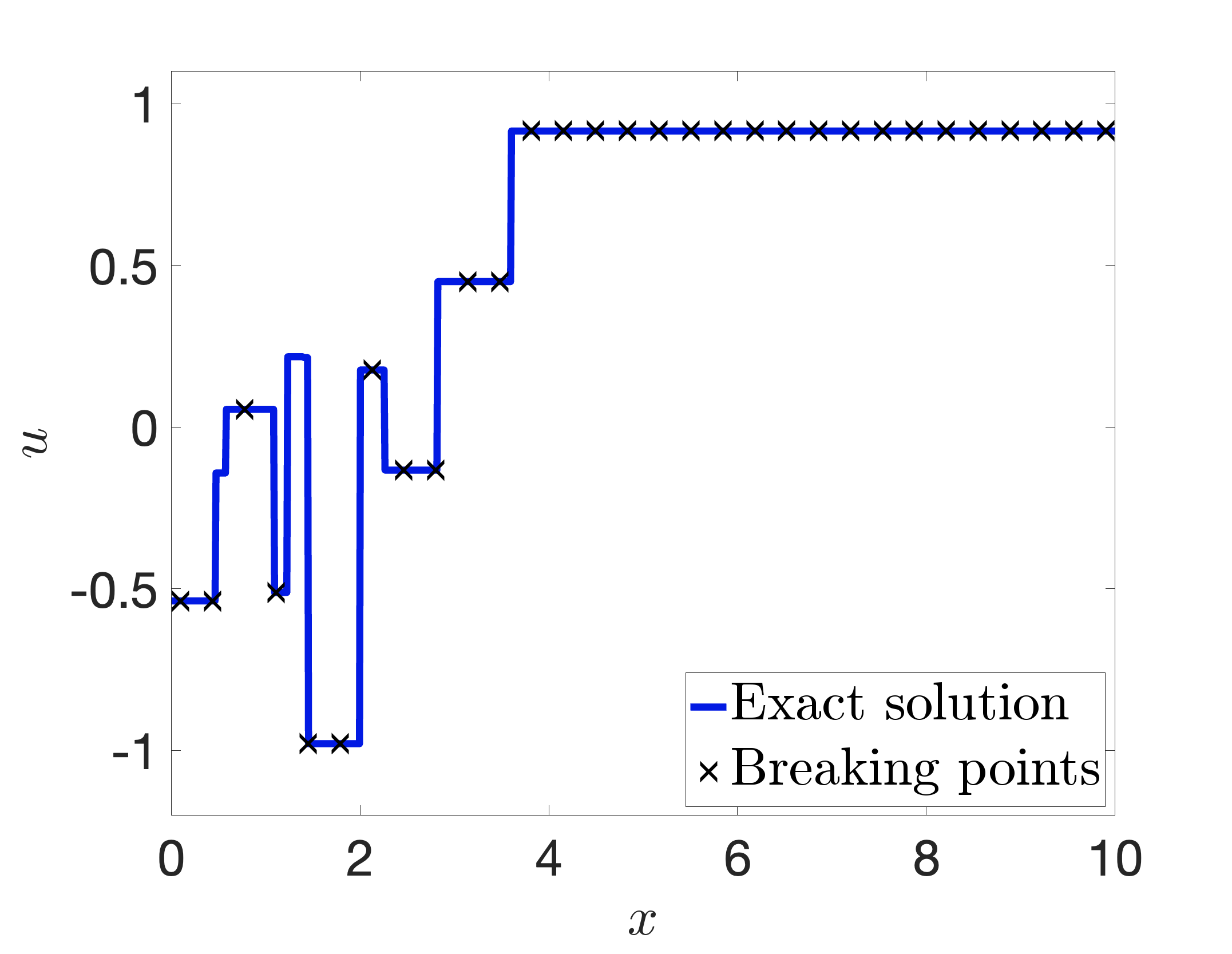}}
    \hspace{0.05in}
    \subfigure[Loss curves]{ 
    \label{fig1:step1D:loss} 
    \includegraphics[width=2.2in]{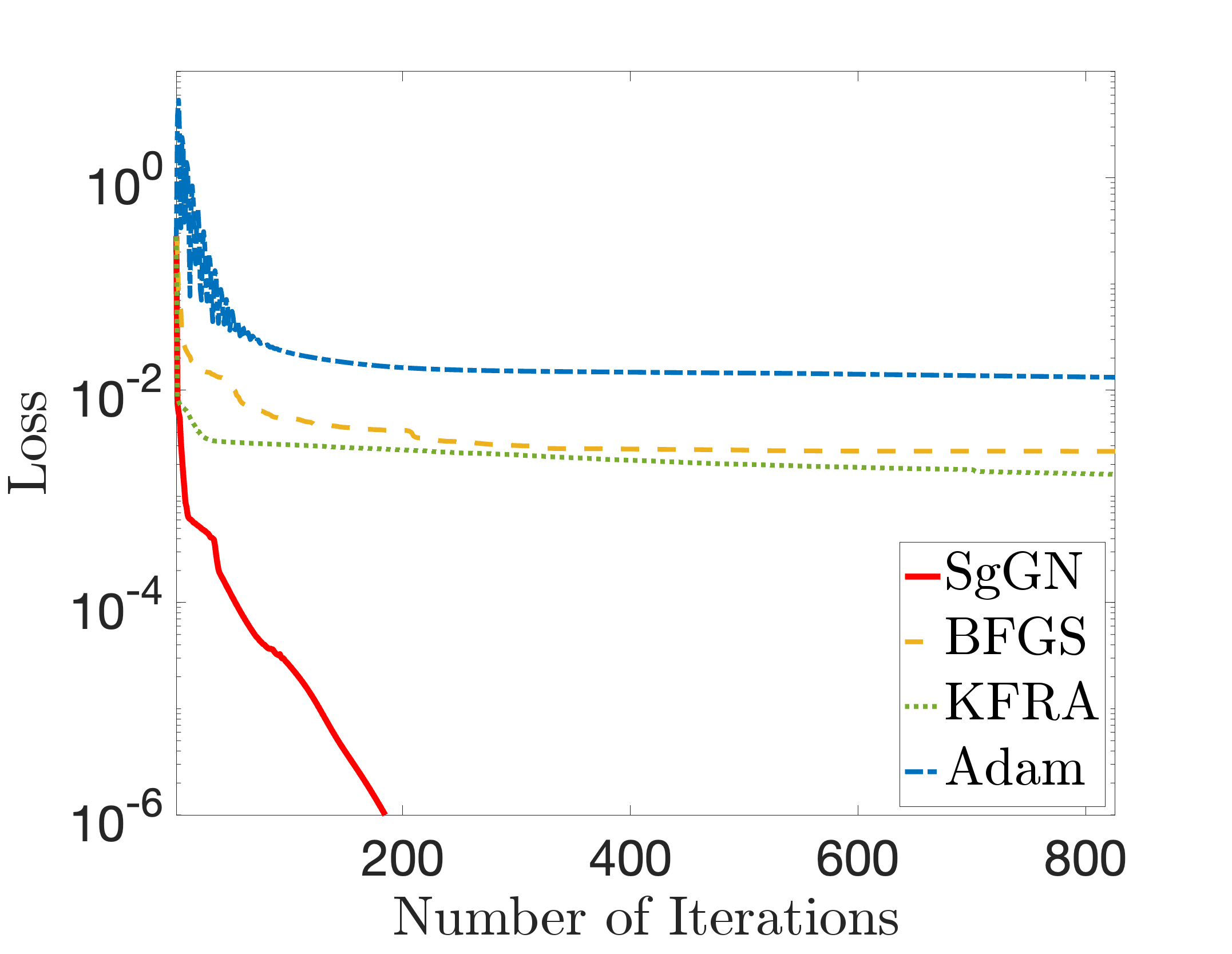}}
    \caption{Approximation of a piecewise constant function.}
  \label{fig1:step1D} 
\end{figure}

The loss decay curves for all methods are shown in \cref{fig1:step1D:loss}. While the loss for SgGN continues to decrease steadily beyond 200 iterations, the other three methods exhibit slower convergence, with their curves plateauing near their final loss values.
A quantitative comparison of the least-squares loss is provided in Table~\ref{tab1:step1D}, where SgGN achieves a remarkably low final loss on the order of $10^{-9}$, substantially outperforming the other methods, which reach loss values around $10^{-3}$.

Further insight is provided by the approximation results in Figure~\ref{fig2:step1D}. Only the SgGN method accurately captures all step discontinuities (Figure~\ref{fig2:step1D:SgGN}) by precisely aligning the neuron breakpoints with the function jumps. In contrast, the other methods (Figures~\ref{fig2:step1D:BFGS}, \ref{fig2:step1D:KFRA}, and \ref{fig2:step1D:Adam}) either miss several discontinuities or exhibit overshooting near the steps, reflecting their limited ability to adapt to the function's discontinuities.

\begin{table}[ptbh]
        \caption{Accuracy comparison for the approximation of the one-dimensional piecewise constant function.}
        \label{tab1:step1D}
    \begin{center}
        \begin{tabular}{ |c|c|c|c|c|c|c| }
             \hline
             Method & \multicolumn{2}{c|}{SgGN} &\multicolumn{2}{c|}{BFGS} & KFRA  & Adam \\
             \hline\hline
            Iteration& $\bm{9}$ &$825$& $207$& $ 825$ & $825$ & $10,000$ \\
             \hline
            $ \mathcal{J}_{m,\mu}$& $8.76\text{E-}4$ &$\bm{6.56\textbf{E-}9}$& $4.03\text{E-}3 $& $2.65\text{E-}3 $& $ 1.61\text{E-}3$  & $8.14\text{E-}3$ \\
             \hline
        \end{tabular}
    \end{center}
\end{table}

\begin{figure}[ptbh]
  \centering 
    \subfigure[SgGN $u_n$]{ 
    \label{fig2:step1D:SgGN} 
    \includegraphics[width=2.2in]{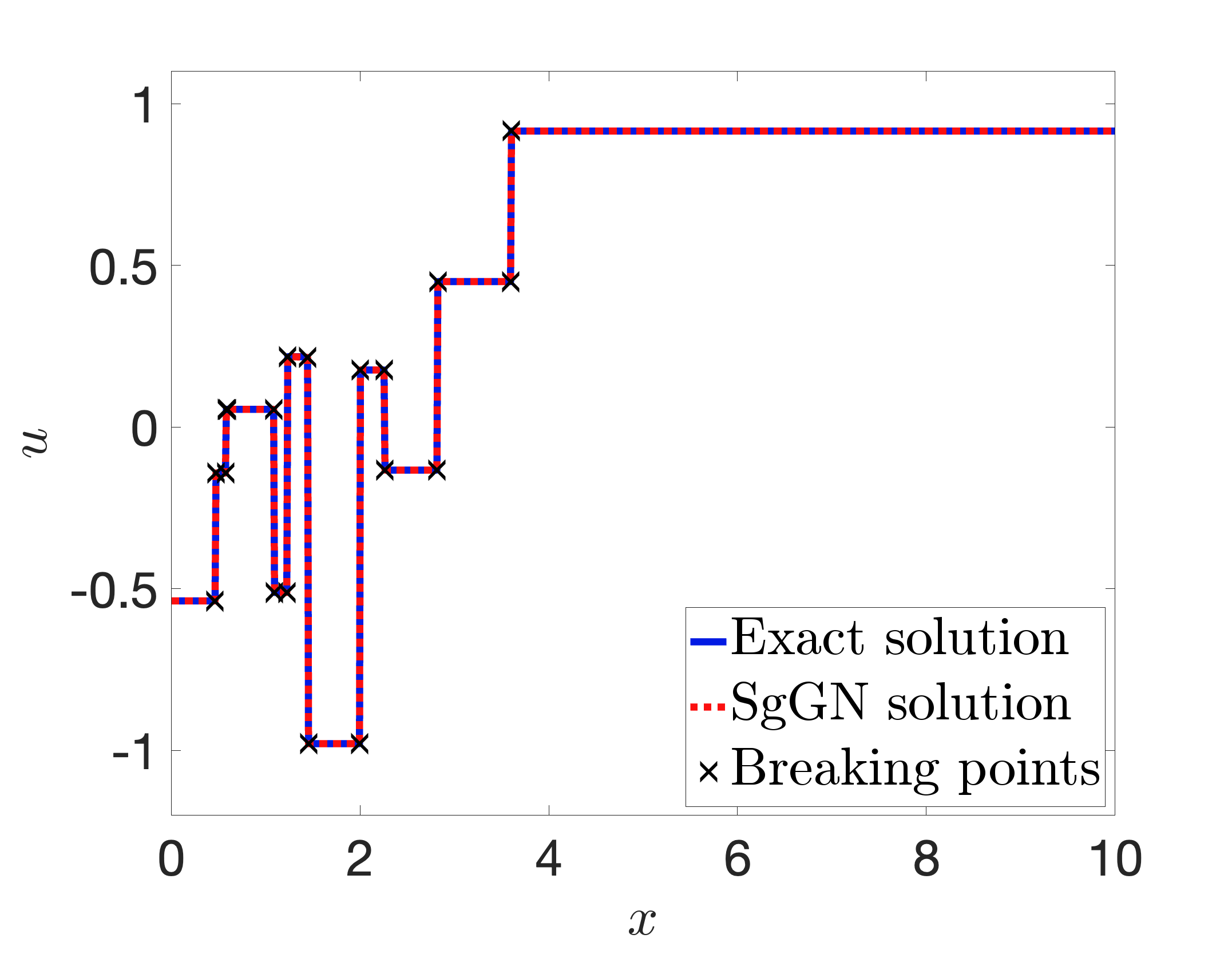}}
    \hspace{0.05in}
  \subfigure[BFGS $u_n$]{ 
    \label{fig2:step1D:BFGS} 
    \includegraphics[width=2.2in]{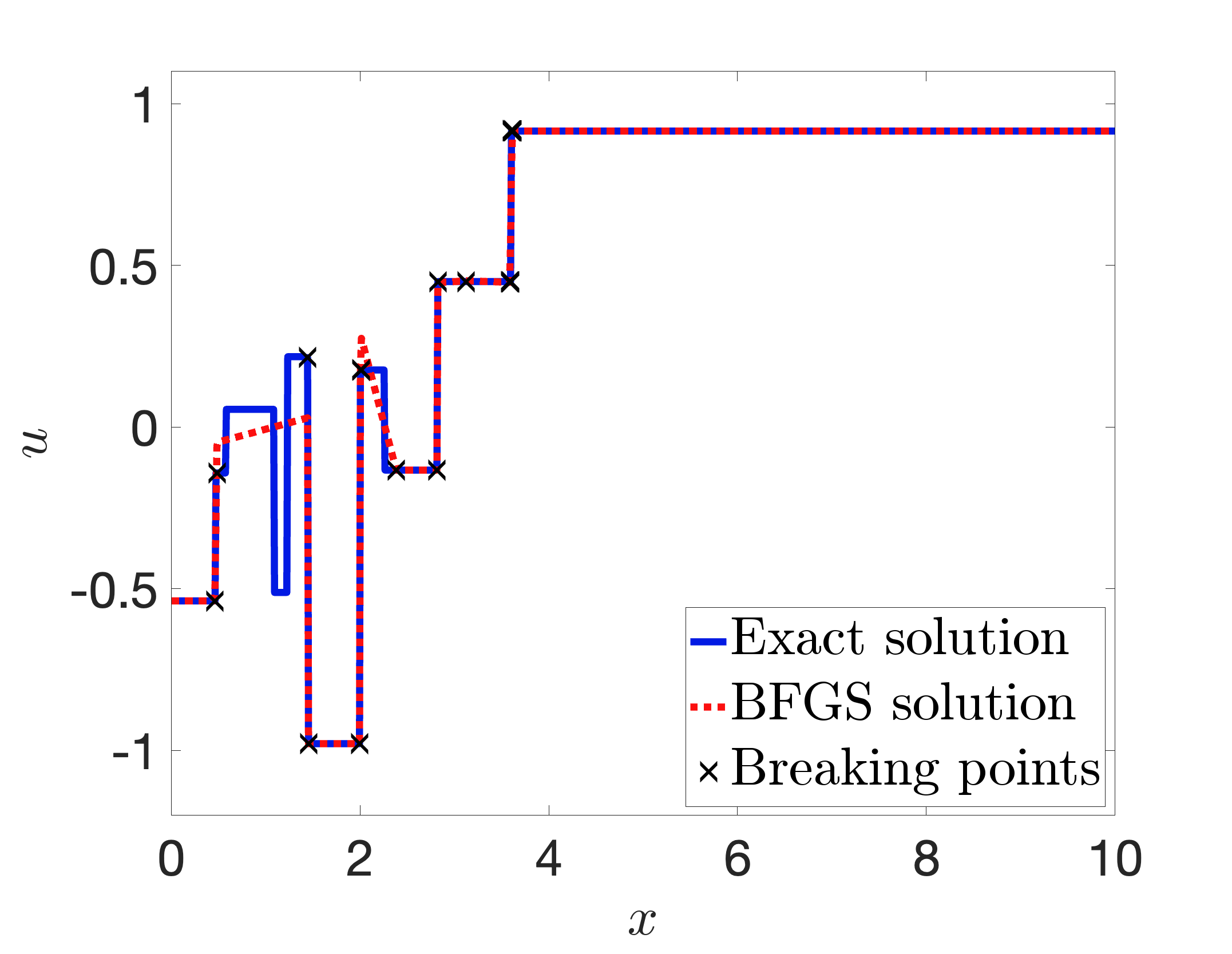}}
  \subfigure[KFRA $u_n$]{ 
    \label{fig2:step1D:KFRA}   \includegraphics[width=2.2in]{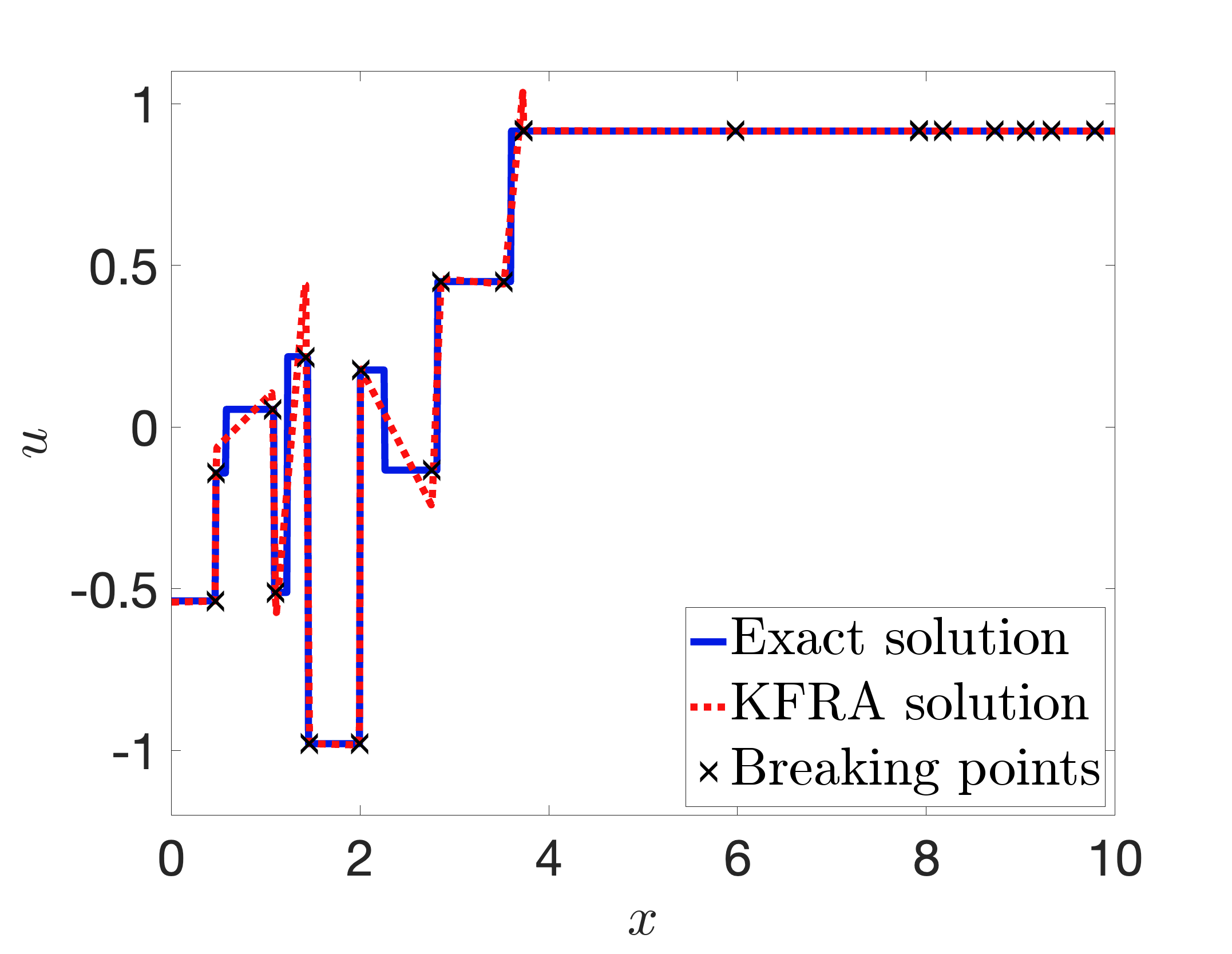}}
    \hspace{0.05in}
  \subfigure[Adam $u_n$]{ 
    \label{fig2:step1D:Adam} 
    \includegraphics[width=2.2in]{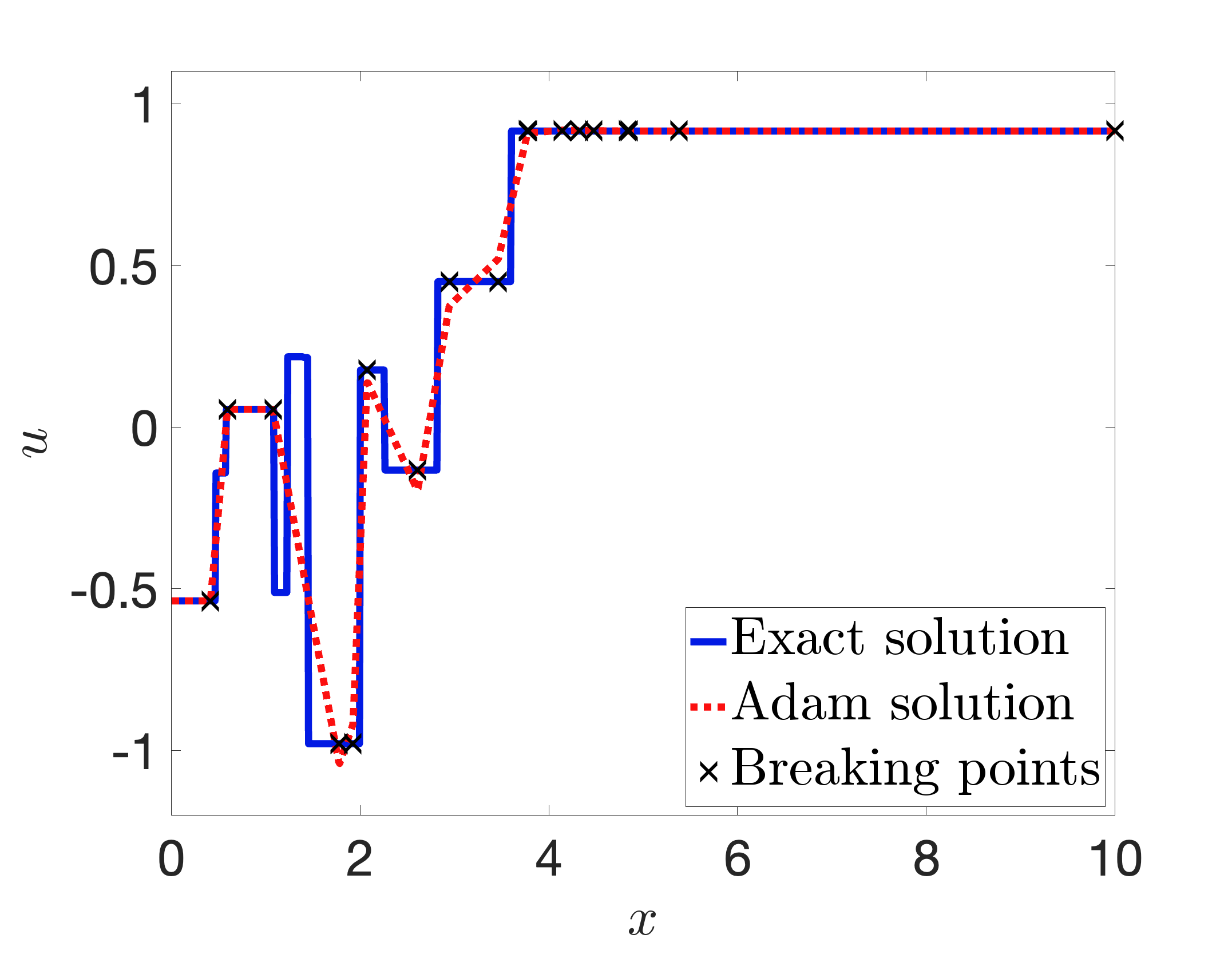}}  
  \caption{Approximation of the one-dimensional piecewise constant function.}
  \label{fig2:step1D} %% label for entire figure 
\end{figure}

\subsection{One-dimensional delta-like function}
In the second experiment, we examine a smooth but sharply peaked delta-like function introduced and defined in \cite{XiaShenWang16} 
\begin{equation*}
    u(x)=\sum_{i=1}^{k}\frac{1}{d_i(x-x_i)^2+1}, \quad x \in [-1.5,1.5], 
\end{equation*}
where $k$ denotes the number of centers, $x_i$ is the center position, and $d_i$ are sharpness parameters controlling the width of each peak \textemdash larger $d_i$ values correspond to narrower peaks.

In the test, we set $k=3$ with centers $\left\{x_1,x_2,x_3\right\} =\left\{ -\frac{\pi^2}{10}, -(\pi-\frac{5}{2}), \frac{\sqrt{85}}{10}\right\}$, and corresponding width parameters $\left\{d_1,d_2,d_3\right\} =\left\{ 10^4, 10^3 ,5\times10^3\right\}$.  A shallow ReLU NN with 15 neurons in the first hidden layer was used for all tests. The network was initialized with uniformly distributed breaking points,  as shown in \cref{fig3:delta:problem}. This experiment aims to evaluate two aspects of each optimization method: (1) its ability to relocate the initially rationally placed breaking points to accurately capture all three {\it irrationally} located peak centers, and (2) its capacity to adaptively allocate the 15 breaking points to reflect the varying sharpness of the peaks. For Adam, we used $\alpha_1=0.02 $, $\alpha_f =0.6 $ and $T = 2000$; for KFRA, we set $\gamma=0.0001$.

As shown in \cref{fig3:delta:loss}, SgGN achieves a significantly faster loss decay and converges to a superior solution within approximately 50 iterations. It reaches a final loss of magnitude $10^{-4}$, considerably lower than the $10^{-3}\sim10^{-2}$ range observed for the other methods (see \cref{tab2:delta}). The quality of the solution is further illustrated in Figure~\ref{fig4:delta}, where SgGN successfully aligns the breakpoints with all three peak centers and adaptively redistributes the remaining neurons to capture the peaks' different widths with high fidelity. In contrast, the other methods (Figures~\ref{fig4:delta:BFGS}, \ref{fig4:delta:KFRA}, and \ref{fig4:delta:Adam}) either fail to resolve all the peaks or inadequately distribute the breaking points, resulting in poor approximations of the narrow peak features.

\begin{figure}[!ht]
    \begin{center}
    \subfigure[Target function $u$ and initial breaking points as reflected by $\times$'s]{ 
        \label{fig3:delta:problem} 
     \includegraphics[width=2.2in]{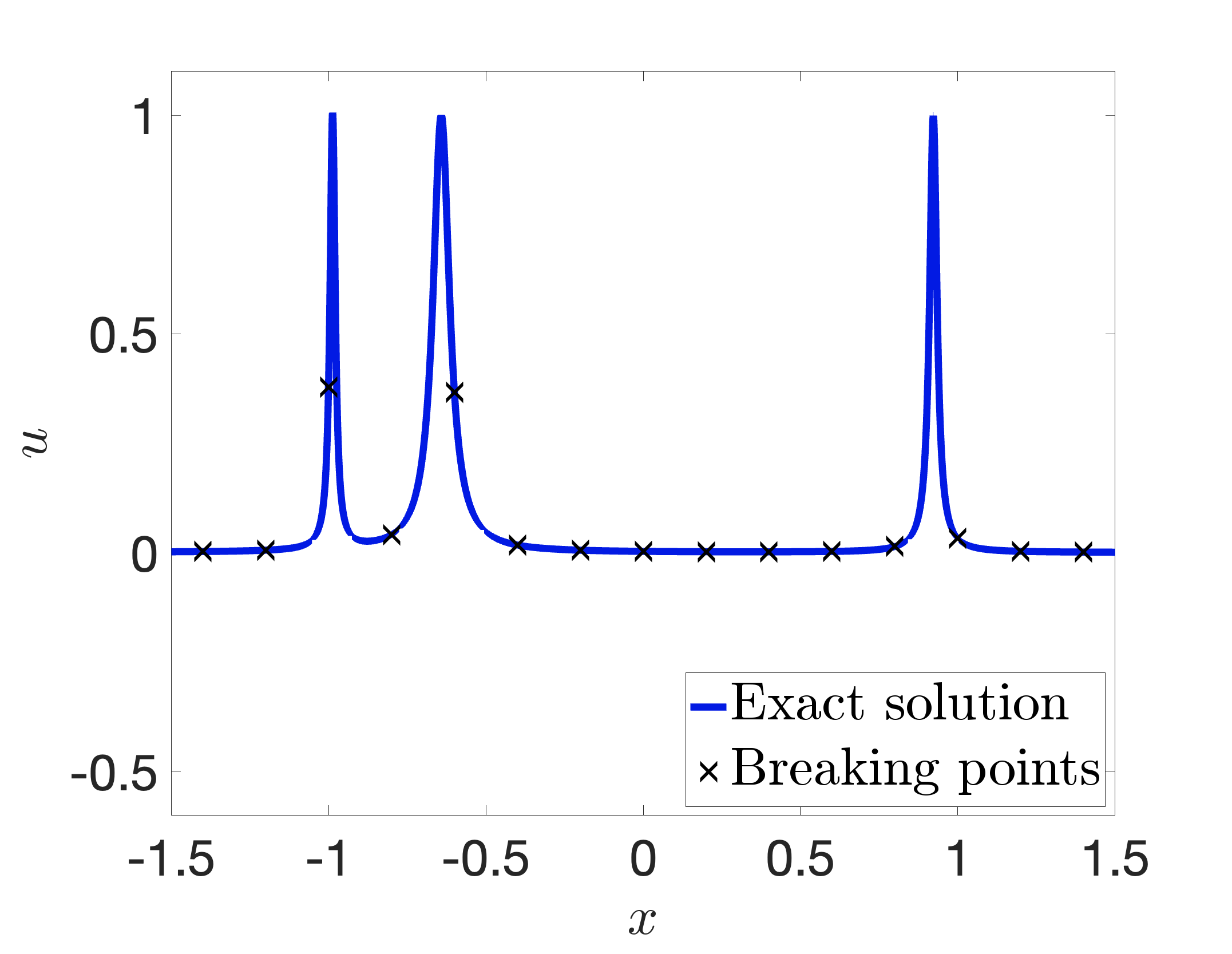}}   
     \subfigure[Loss curves]{ 
        \label{fig3:delta:loss} 
     \includegraphics[width=2.2in]{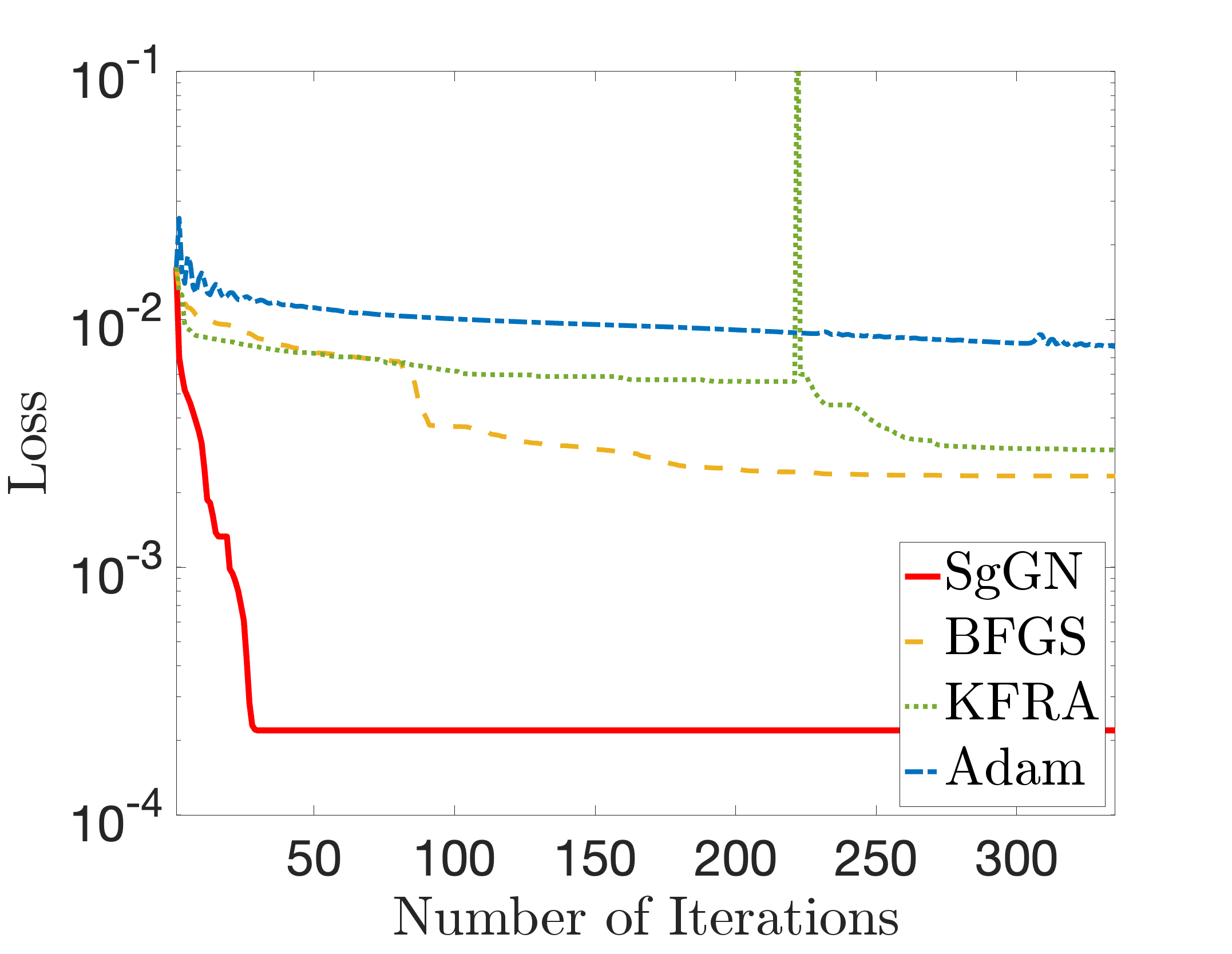}}
      \caption{Approximation of the one-dimensional delta-like function: target function, initial breaking points, and loss curves.}
    \label{fig3:delta}
    \end{center}
\end{figure}
 \begin{figure}[!ht]
    \begin{center}   
    \subfigure[SgGN $u_n$]{ 
    \label{fig4:delta:SgGN} 
    \includegraphics[width=2.2in]{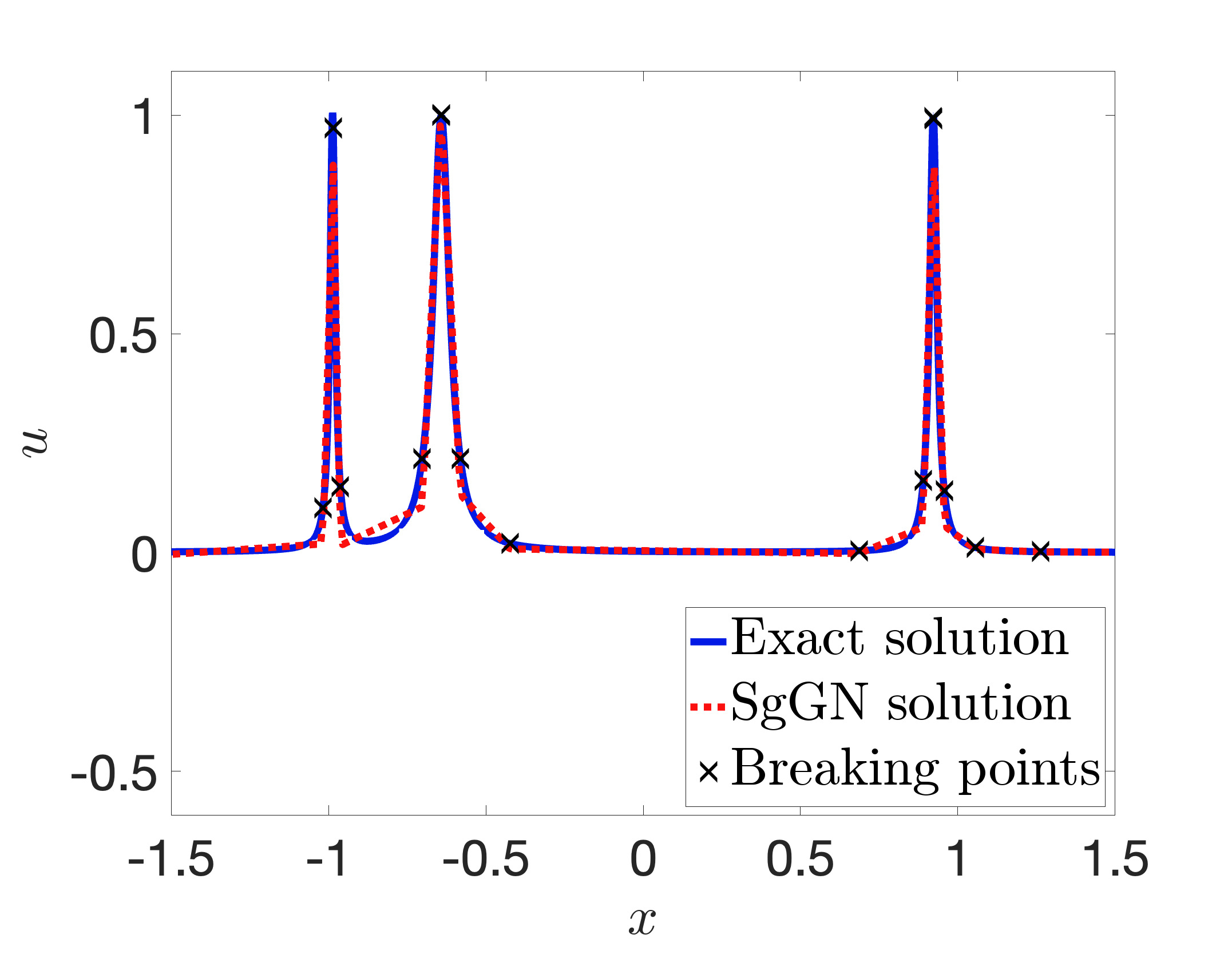}}
    \hspace{0.05in}
  \subfigure[BFGS $u_n$]{ 
    \label{fig4:delta:BFGS} 
    \includegraphics[width=2.2in]{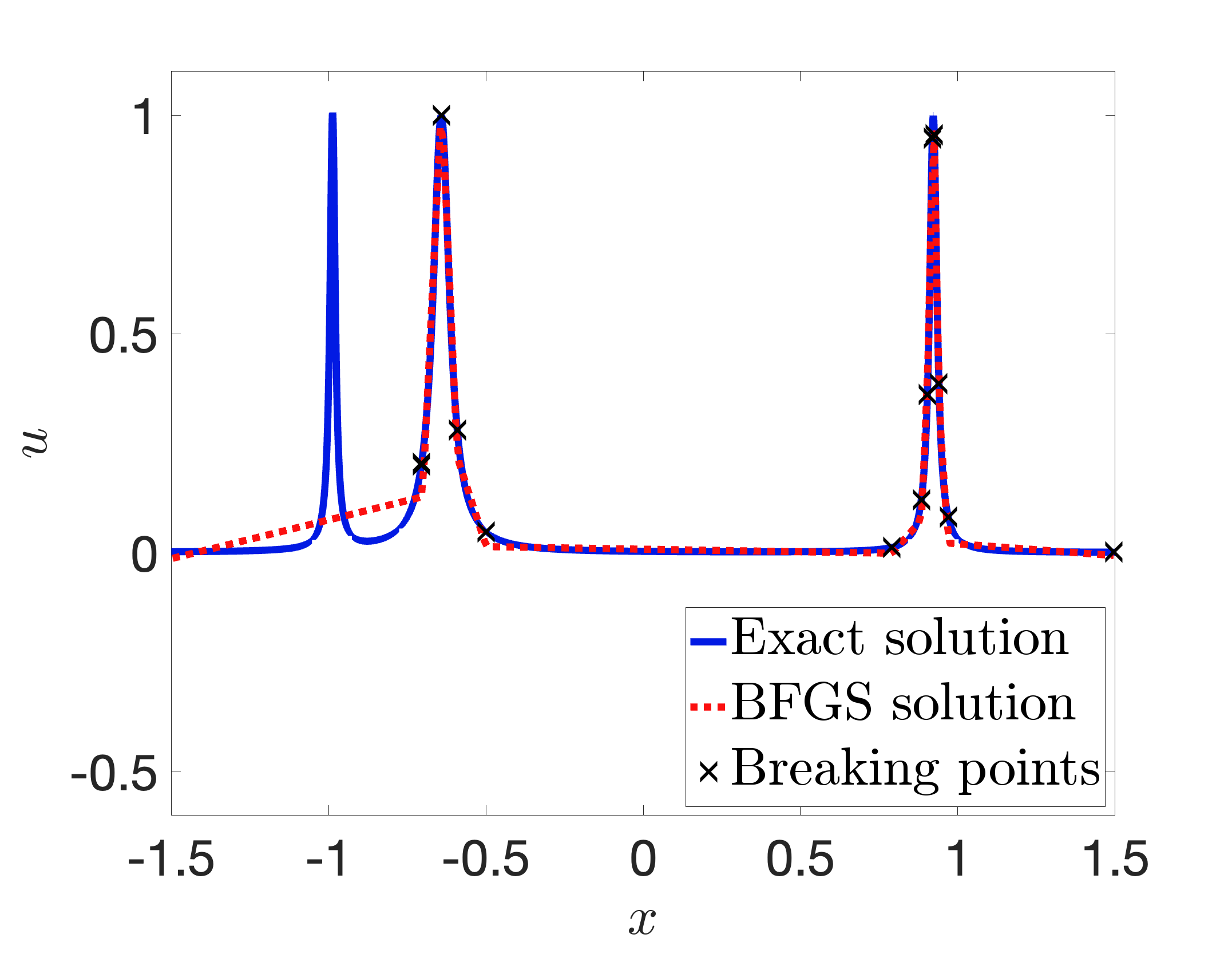}}
  \subfigure[KFRA $u_n$]{ 
    \label{fig4:delta:KFRA} 
    \includegraphics[width=2.2in]{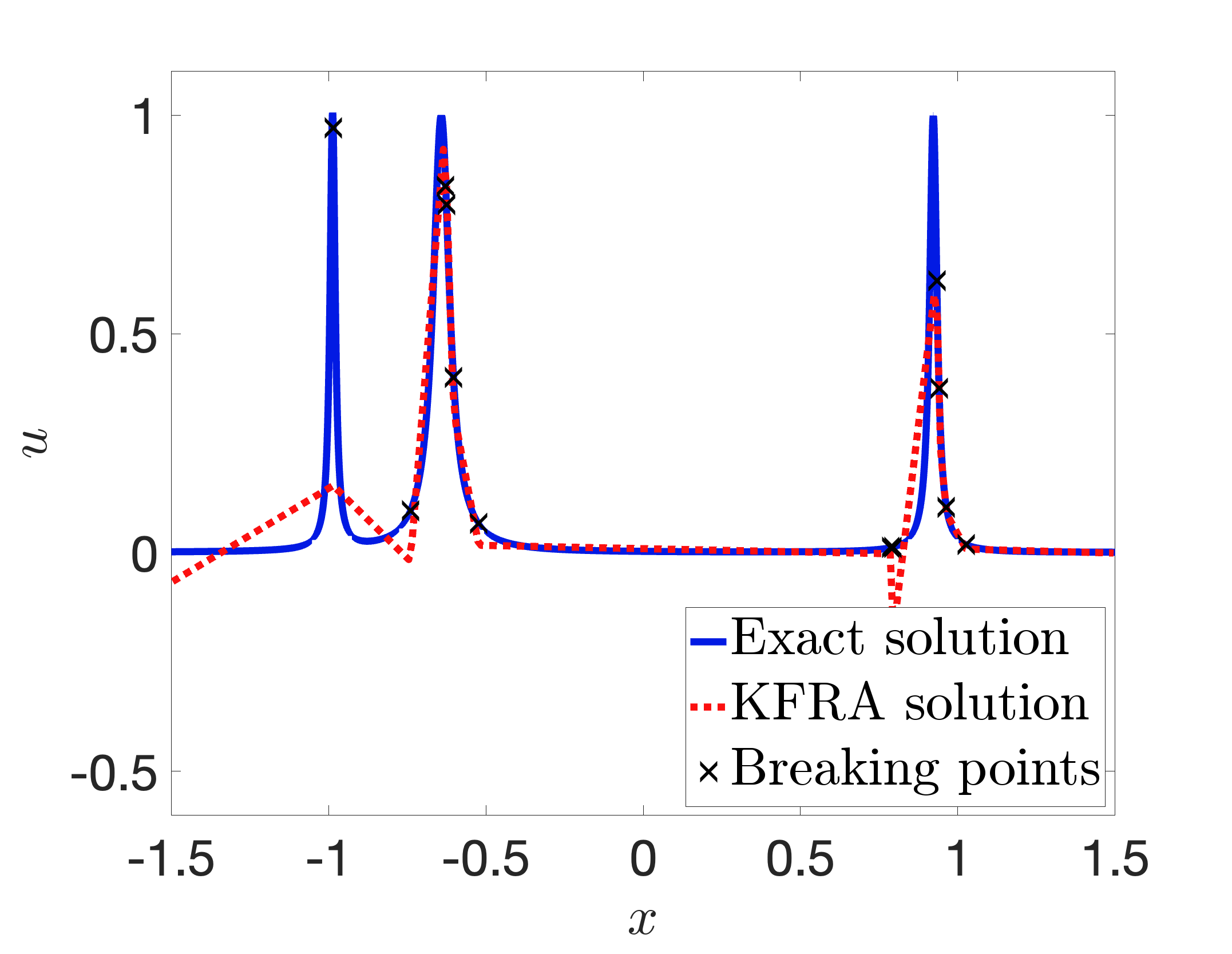}}
    \hspace{0.05in}
  \subfigure[Adam $u_n$]{ 
    \label{fig4:delta:Adam} 
    \includegraphics[width=2.2in]{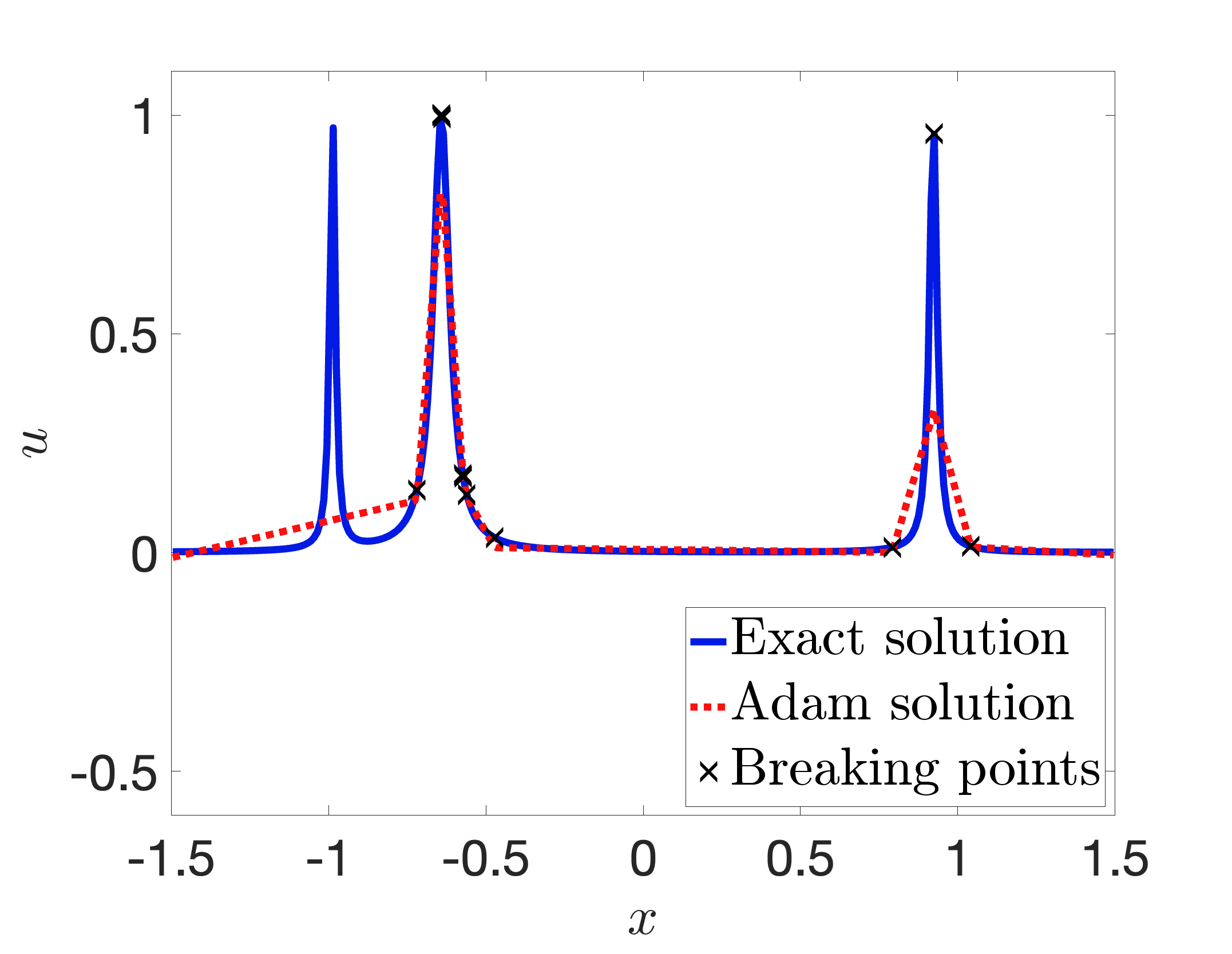}} 
     \caption{Approximation of the one-dimensional delta-like function: loss curves and approximation results using four optimization methods.}
    \label{fig4:delta}
    \end{center}
\end{figure}

\begin{table}[!ht]
    \caption{Accuracy comparison for the approximation of the one-dimensional delta-like function.}
    \label{tab2:delta}
    \begin{center}
        \begin{tabular}{ |c|c|c|c|c|c|c| }
             \hline
             Method & \multicolumn{2}{c|}{SgGN} &\multicolumn{2}{c|}{BFGS} & KFRA & Adam \\
             \hline\hline
            Iteration& $\bm{12}$ & $334$ &$91$&$334 $ &$ 334$ &$10,000$ \\
             \hline
             $\mathcal{J}_{m,\mu}$ & $1.87\text{E-}3$ & $\bm{2.19\textbf{E-}4}$ & $3.74\text{E-}3$&$2.33\text{E-}3$& $2.97\text{E-}3$&$3.94\text{E-}3$ \\
             \hline
        \end{tabular}
    \end{center}
\end{table}

\subsection{Two-dimensional piecewise constant function}
Next, we consider a 2D piecewise constant function defined in the domain $[-1,1]^2$:
\begin{equation*}
    u(\bx) = \begin{cases}
        1, & -0.5\leq x_1+x_2\leq 0.5,\\
        -1, &\text{otherwise.}
    \end{cases}
\end{equation*}

In contrast to the previous two examples \textemdash where more neurons than theoretically necessary were used to account for non-convex optimization uncertainties \textemdash this test deliberately employs the minimum number of neurons needed to assess each method's effectiveness under constrained model capacity. As illustrated in Figure~\ref{fig5:step2D:3D_exact}, each discontinuity in the target function can theoretically be approximated using just two neurons. Thus, a high-quality approximation should place a pair of neurons on either side of the discontinuity, with the closeness of their corresponding breaking lines indicating the accuracy of alignment. Based on this principle, we use only four neurons in this experiment. For the Adam optimizer, we set $\alpha_1=0.01$, $\alpha_f =0.8 $ and $T = 2000$; and for the KFRA method, $\gamma$ is set as 0.005.

The loss decay behavior is shown in Figure~\ref{fig5:step2D:loss}. The SgGN method not only exhibits faster convergence but also reaches its final training loss within approximately 20 iterations \textemdash substantially earlier than the other methods. Table~\ref{tab3:step2D} compares the least-squares losses after 142 iterations for the second-order methods and 10,000 iterations for Adam. Given the integration mesh size $h=0.01$,  there exists a theoretical lower bound on the achievable proximity of the breaking lines; which constrains the minimal attainable loss. Even under this limitation, SgGN achieves a final loss on the order of $10^{-3}$, outperforming the other methods, whose losses remain in the $10^{-2}$ range.
Figure~\ref{fig5:step2D} further demonstrates the spatial alignment of the breaking lines. As shown in Figure~\ref{fig5:step2D:bl_SgGN}, SgGN successfully positions all four breaking lines to accurately capture the discontinuities in the target function. In contrast, the other methods (Figures~\ref{fig5:step2D:bl_BFGS} and \ref{fig5:step2D:bl_Adam}) manage to align with only one side of the discontinuity, indicating a significantly less effective approximation.

\begin{figure}[!ht]
    \begin{center}
        \subfigure[Exact solution]{ 
        \label{fig5:step2D:3D_exact} 
        \includegraphics[width=1.5in]{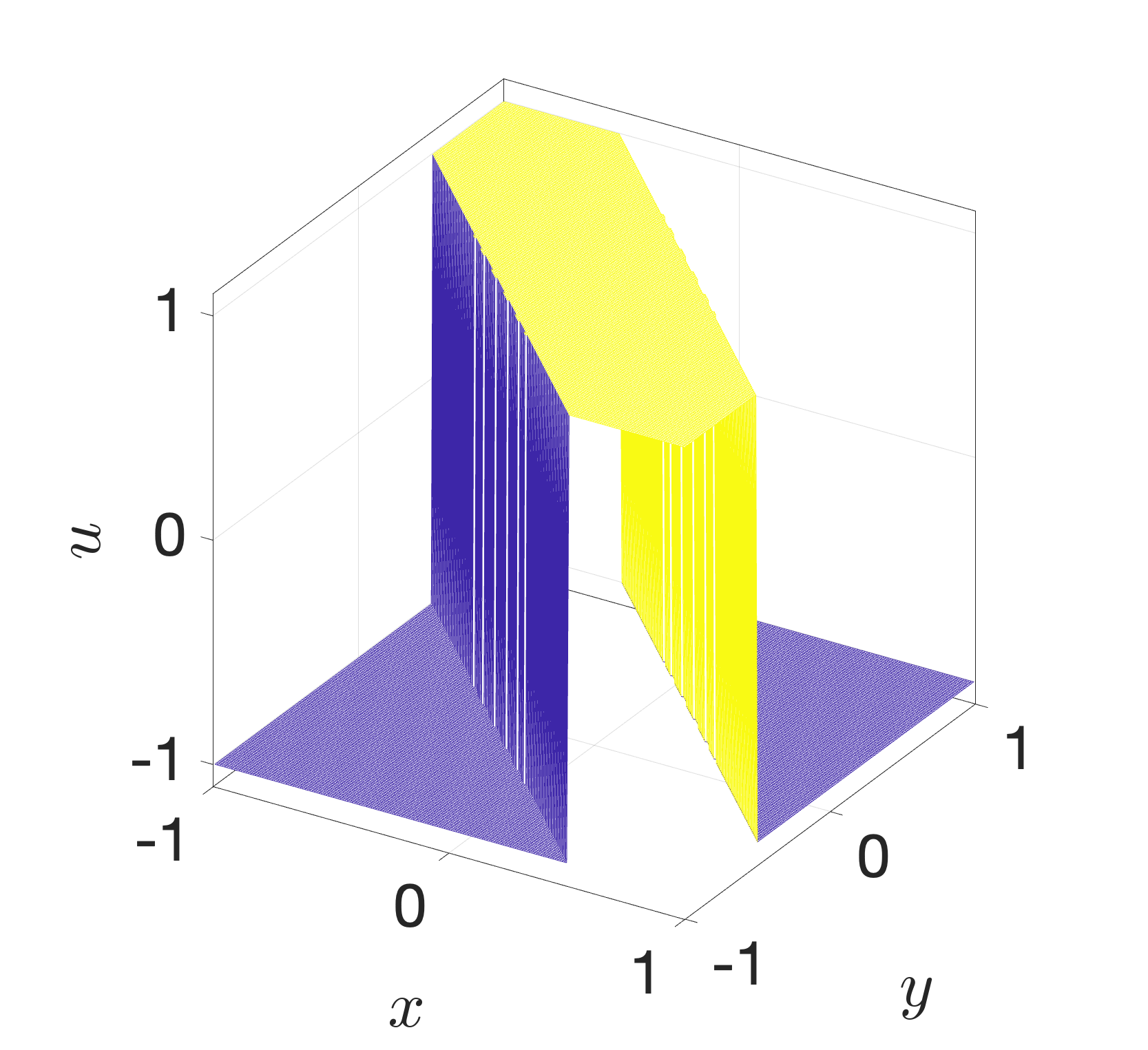}}
        \hspace{0.05in}
        \subfigure[Initial breaking lines]{ 
        \label{fig5:step2D:bl_init} 
        \includegraphics[width=1.5in]{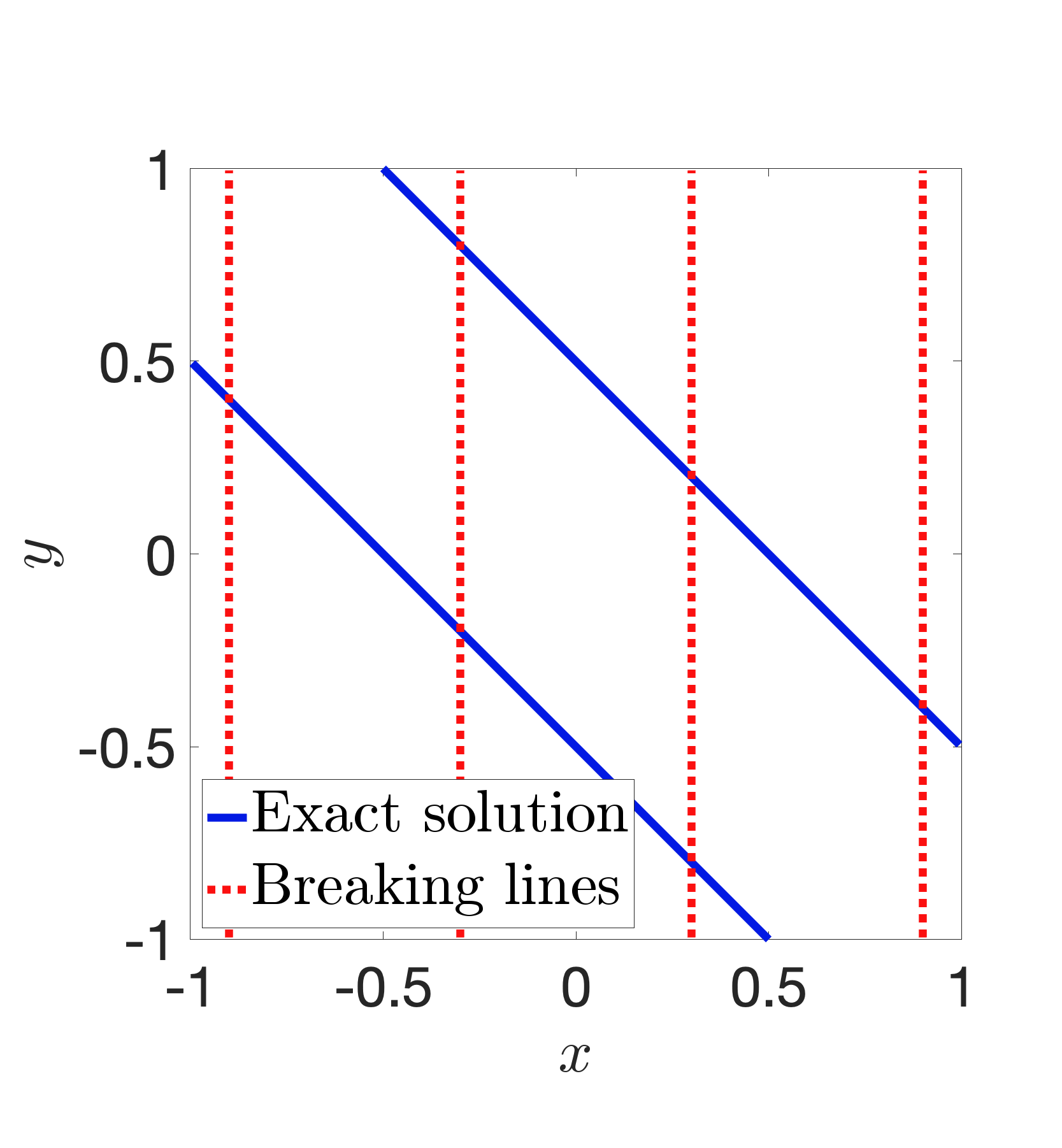}}
        \hspace{0.05in}
        \subfigure[Loss curves]{ 
        \label{fig5:step2D:loss} 
        \includegraphics[width=1.5in]{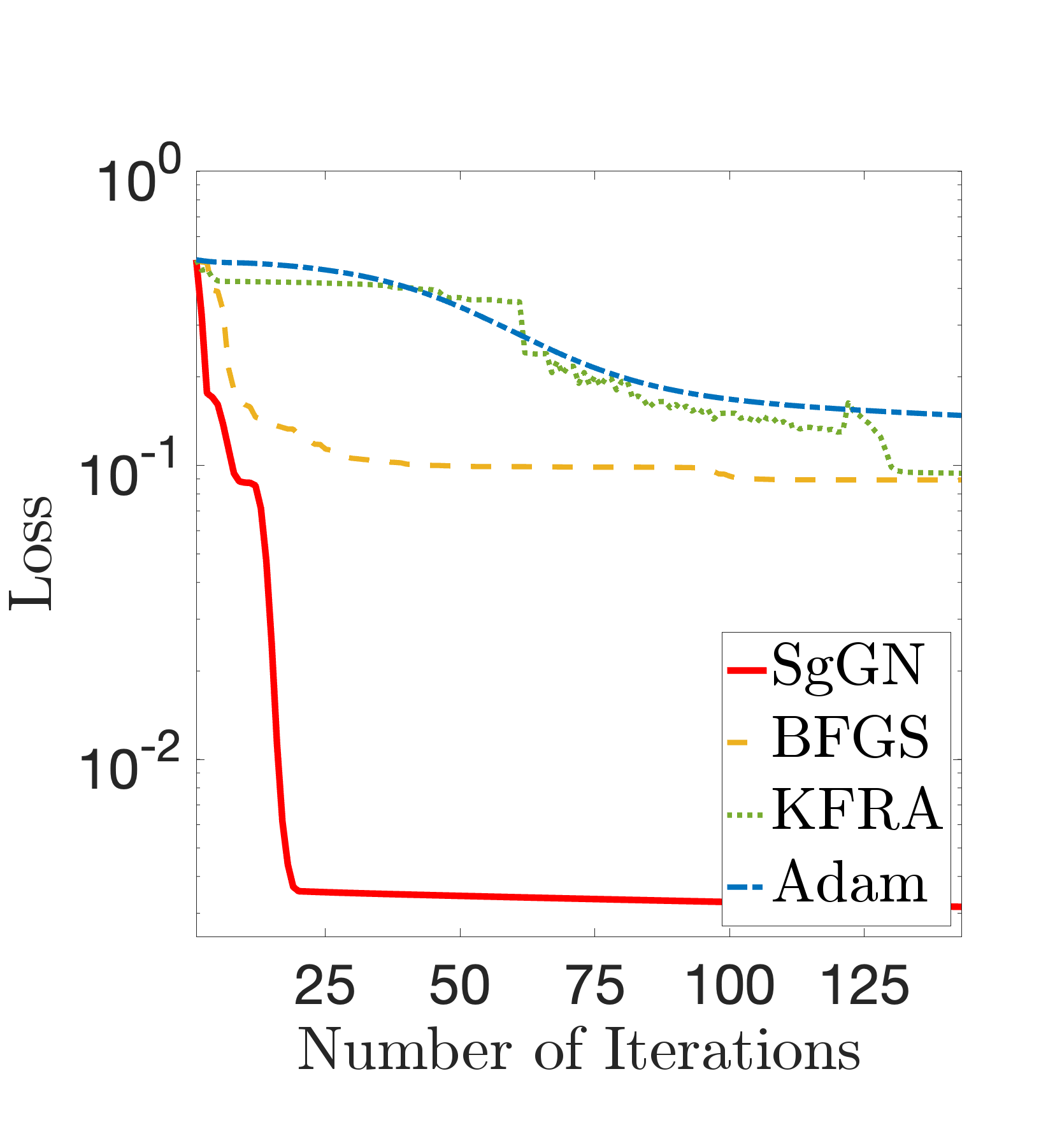}}
        \hspace{0.05in}\\
        \subfigure[SgGN $u_n$]{ 
        \label{fig5:step2D:3D_SgGN} 
        \includegraphics[width=1.35in]{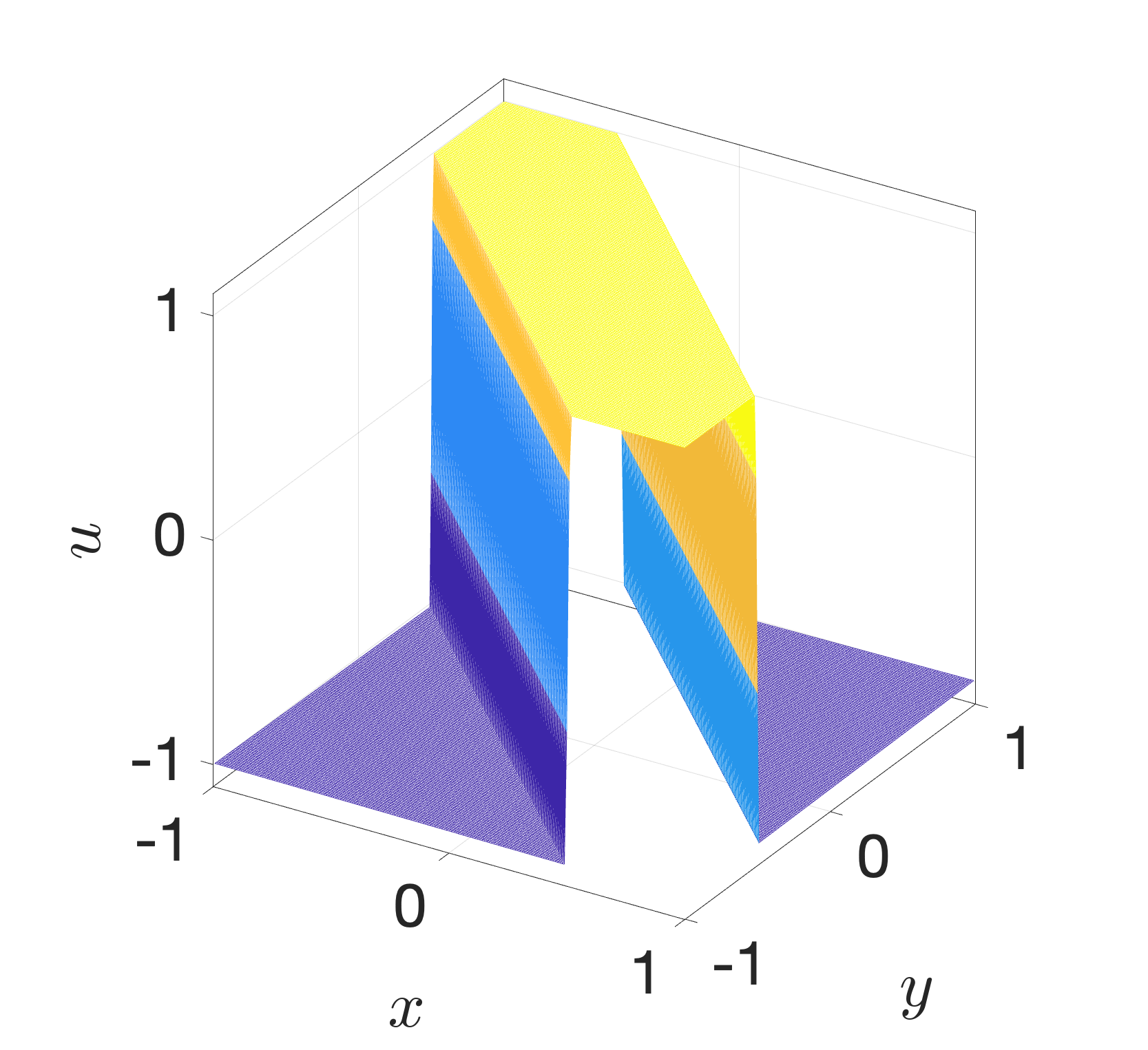}}
        \subfigure[BFGS $u_n$]{ 
        \label{fig5:step2D:3D_BFGS} 
        \includegraphics[width=1.35in]{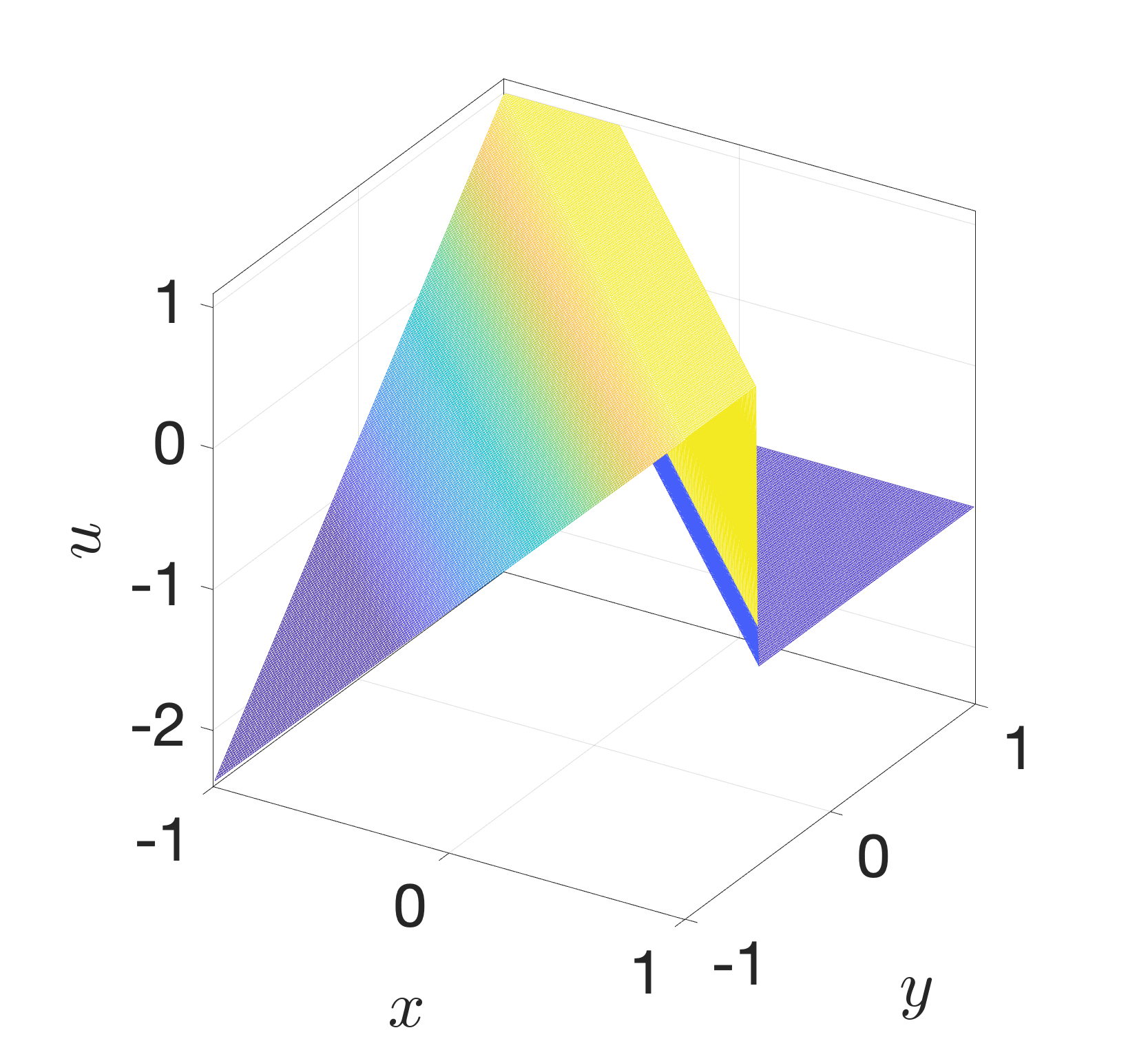}}
        \subfigure[KFRA $u_n$]{ 
        \label{fig5:step2D:3D_KFRA} 
        \includegraphics[width=1.35in]{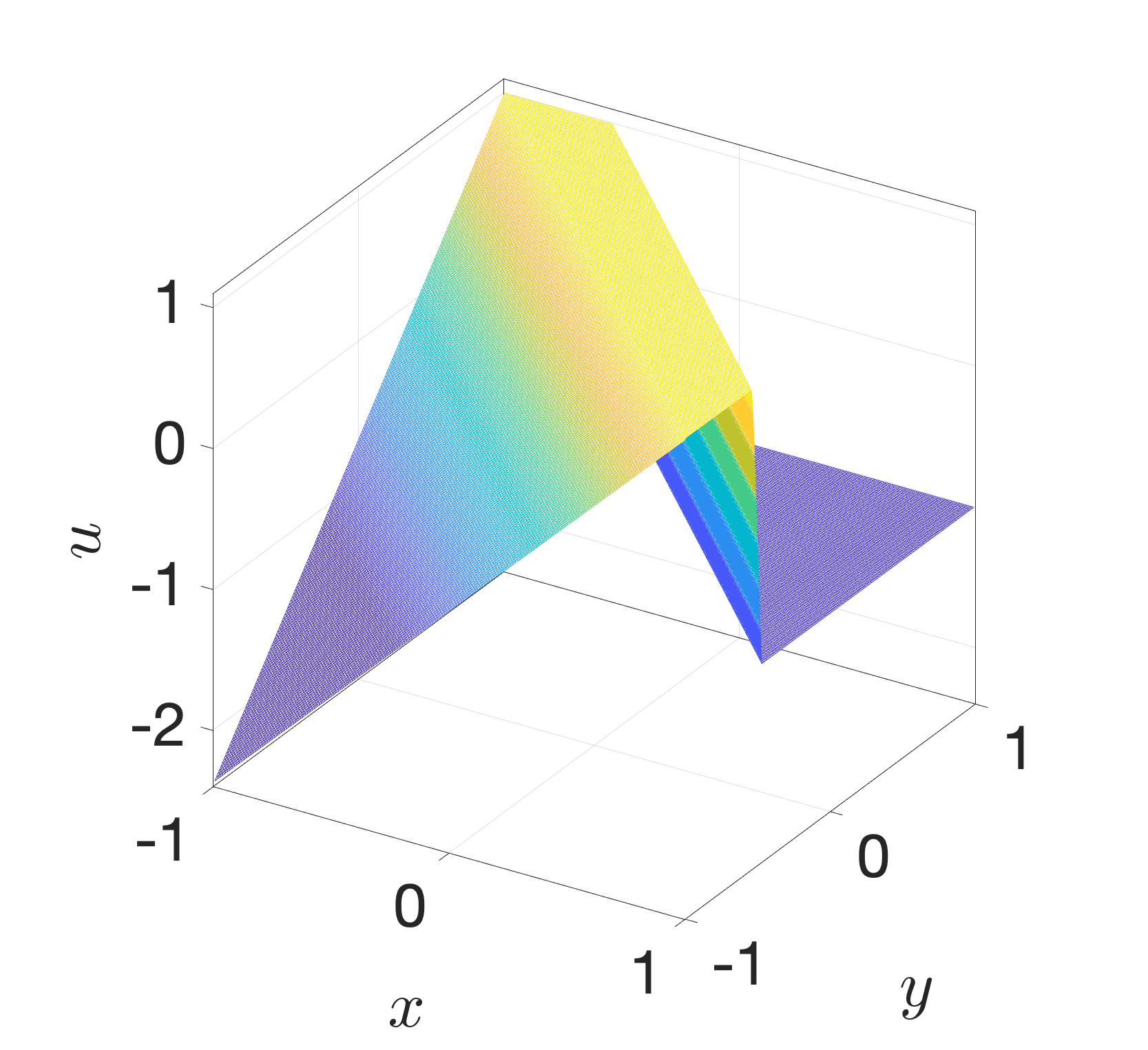}}
        \subfigure[Adam $u_n$]{ 
        \label{fig5:step2D:3D_Adam} 
        \includegraphics[width=1.35in]{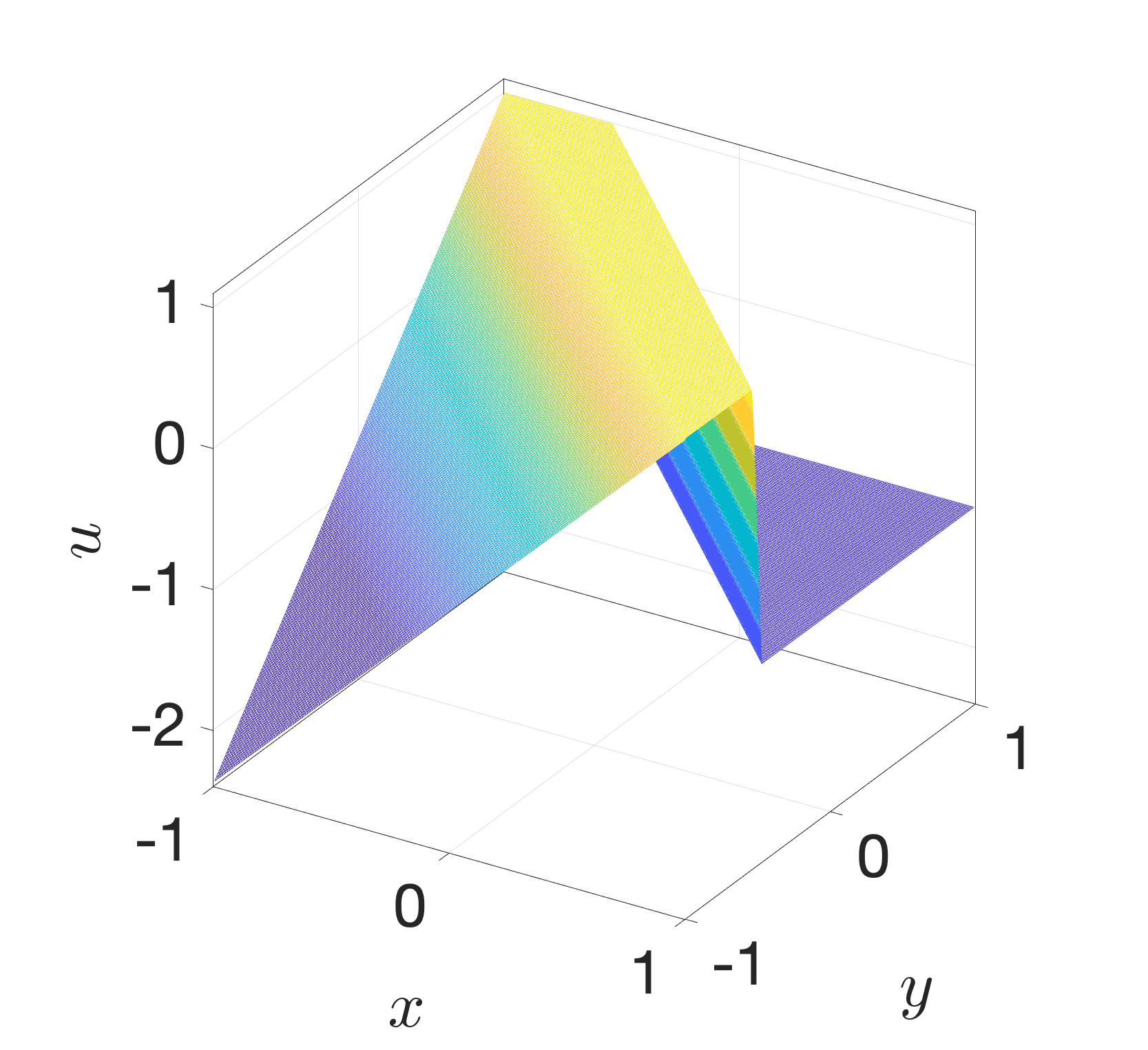}}
         \subfigure[SgGN trained breaking lines]{ 
        \label{fig5:step2D:bl_SgGN} 
        \includegraphics[width=1.35in]{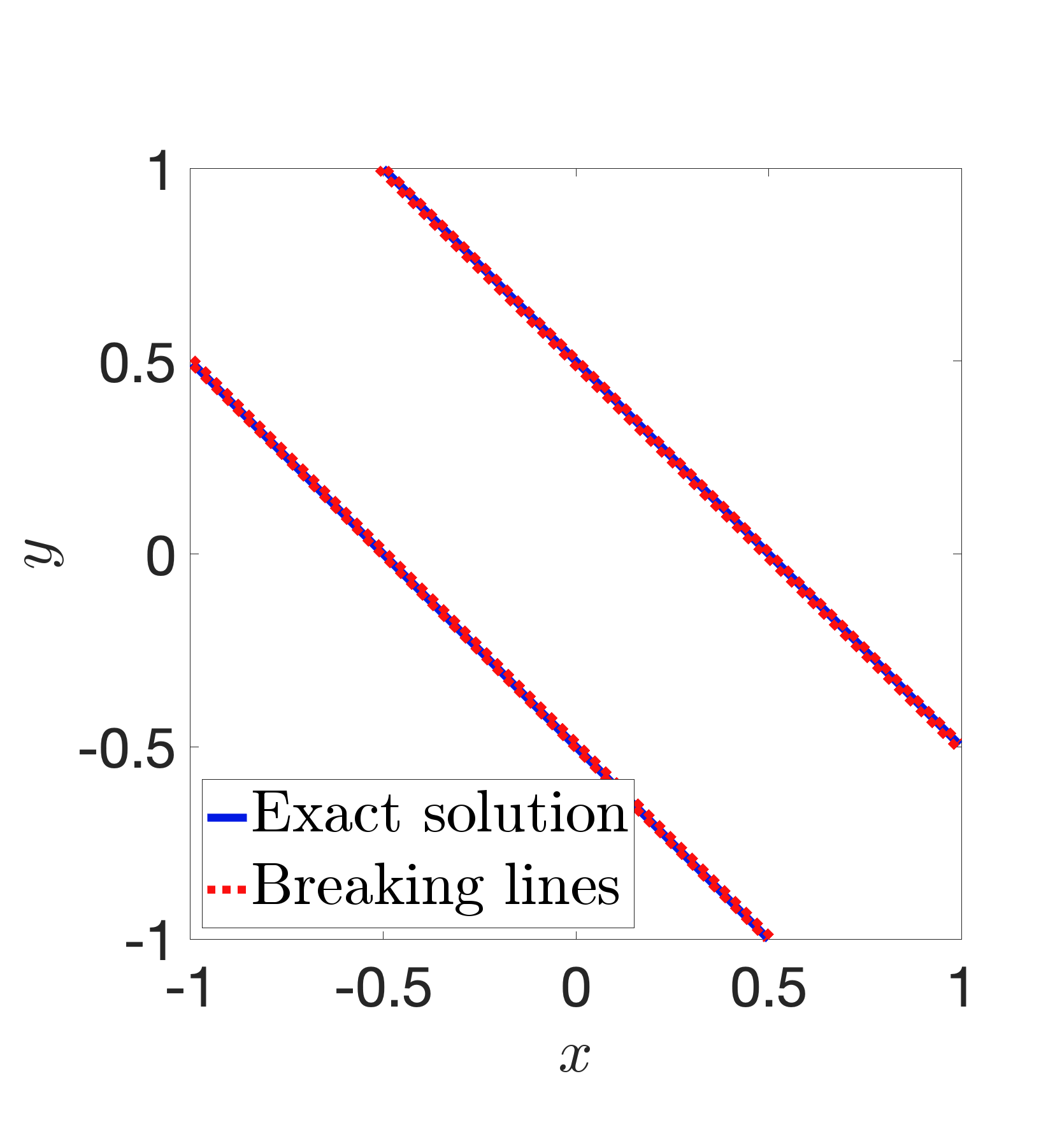}}
        \subfigure[BFGS trained breaking lines]{ 
        \label{fig5:step2D:bl_BFGS} 
        \includegraphics[width=1.35in]{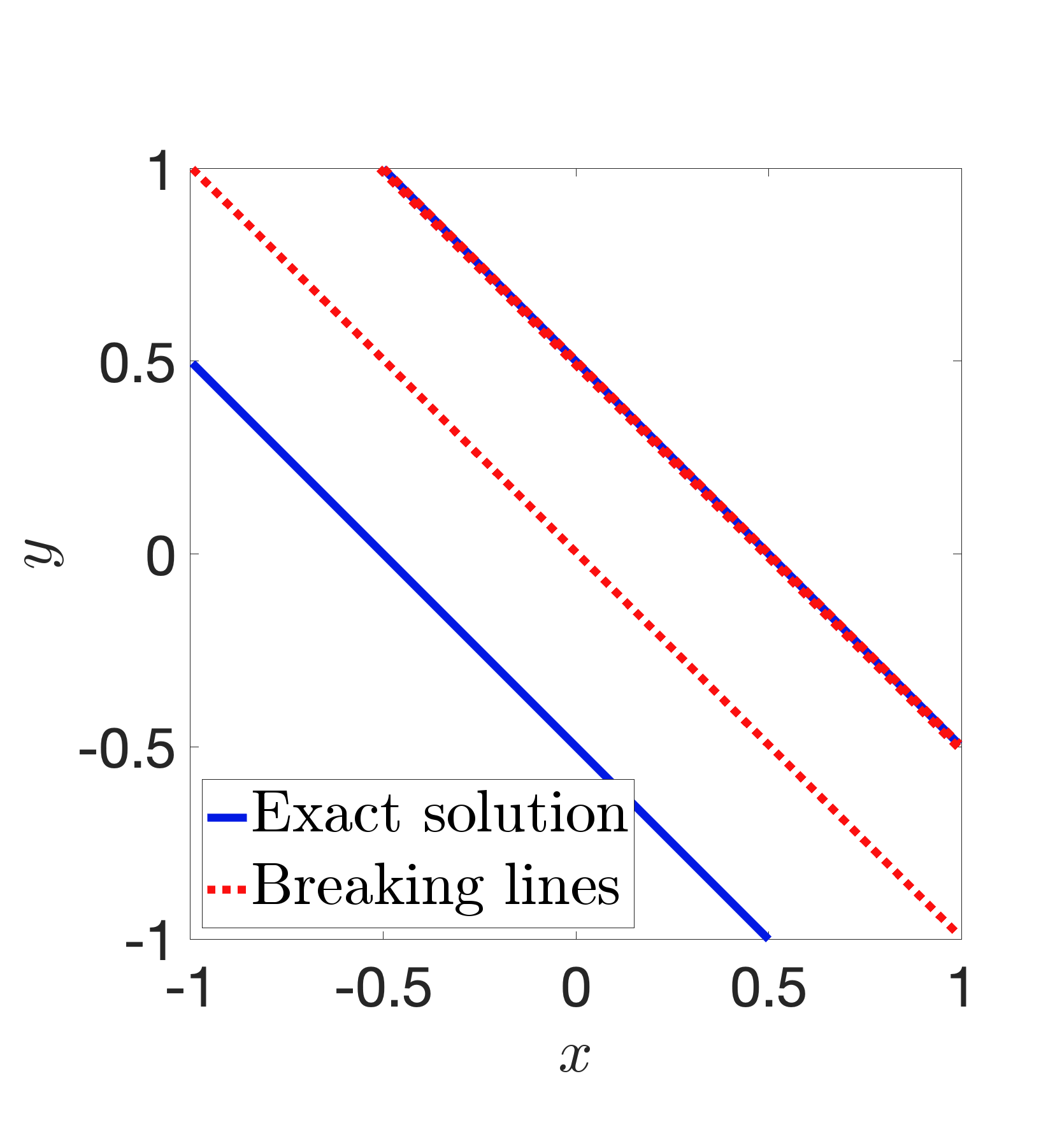}}
        \subfigure[KFRA trained breaking lines]{ 
        \label{fig5:step2D:bl_KFRA} 
        \includegraphics[width=1.35in]{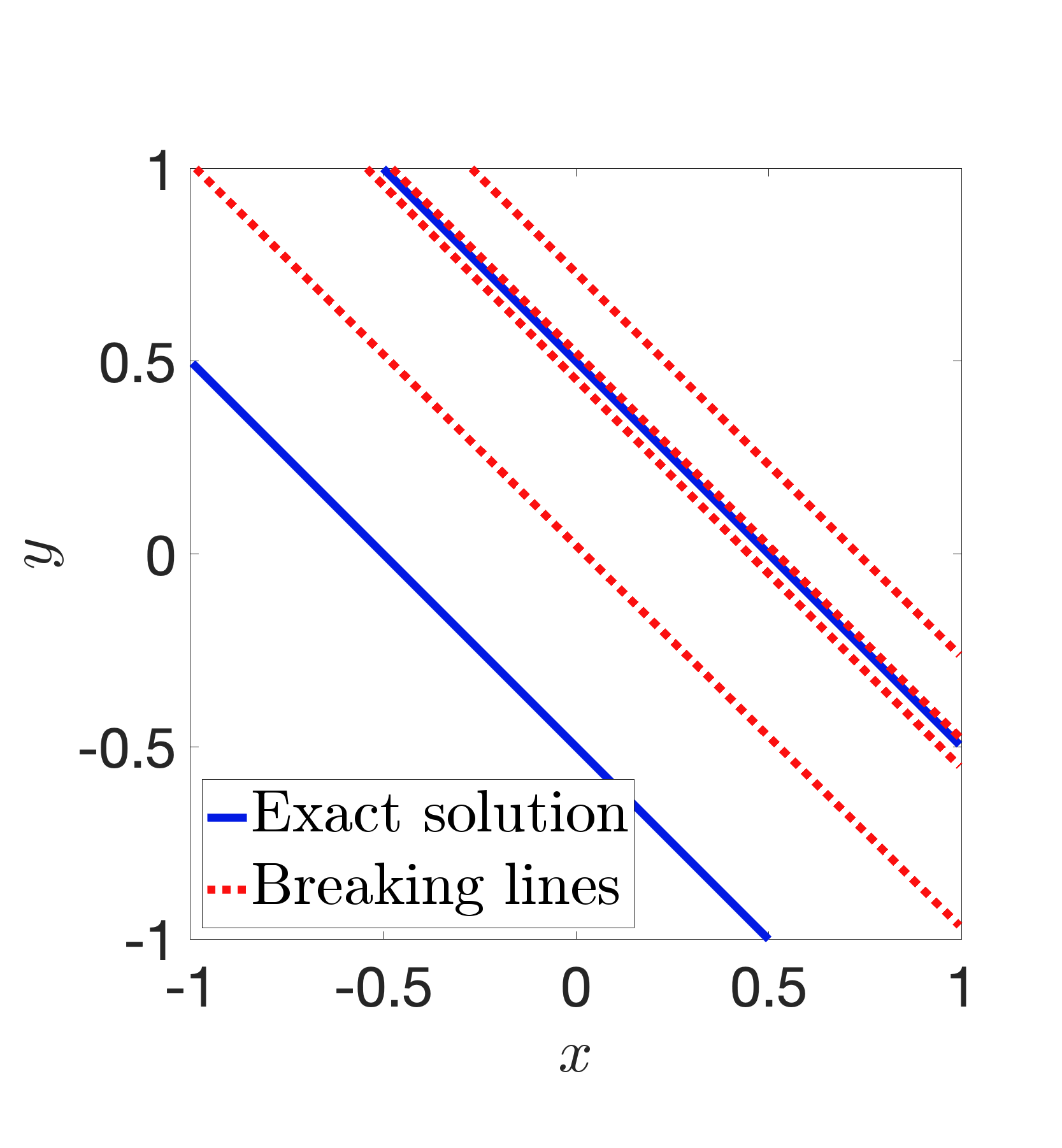}}
        \subfigure[Adam trained breaking lines]{ 
        \label{fig5:step2D:bl_Adam} 
        \includegraphics[width=1.35in]{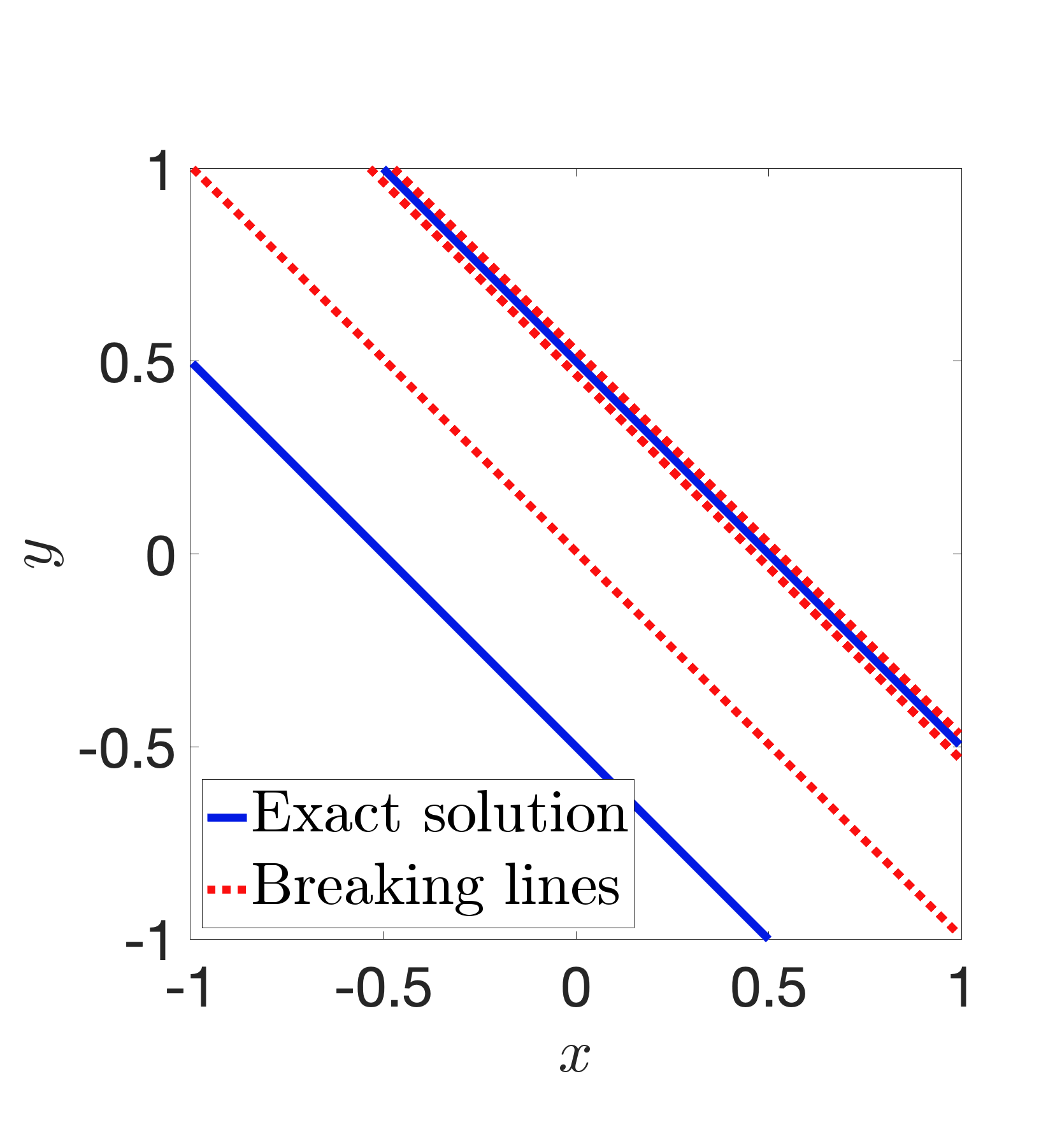}}
     \end{center}
     \caption{Approximation of the two-dimensional piecewise constant function: target function, initial breaking lines, optimization loss curves and approximation results using four optimization methods.}
     \label{fig5:step2D}
\end{figure} 

\begin{table}[!h]
        \caption{Accuracy comparison for the two-dimensional piecewise constant function.}
        \label{tab3:step2D}
    \begin{center}
        \begin{tabular}{ |c|c|c|c|c|c|c|}
             \hline
             Method &  \multicolumn{2}{c|}{SgGN} & \multicolumn{2}{c|}{BFGS}&KFRA & Adam \\
             \hline\hline
            Iteration& $\bm{9}$ &$142$ &$100$& $142$ & 142&$10,000$\\
             \hline
            $\mathcal{J}_{m,\mu} $& $8.82\text{E-}2$ &$\bm{3.16\textbf{E-}3}$ &$9.20\text{E-}2$ &$8.92\text{E-}2$& $9.40 \text{E-}2$& $9.23\text{E-}2$\\
             \hline
        \end{tabular}
    \end{center}
\end{table}
 
\subsection{Two-dimensional function in $\mathcal{M}_n(\Omega)$}
For the previous three test problems, the target functions do not reside within the defined network function set $\mathcal{M}_n(\Omega)$, resulting in an inherent approximation challenge that prevents the loss from converging to machine precision. 
 
To better assess the optimization performance, we introduce a synthetic test case in which the target function is explicitly constructed from $\mathcal{M}_n(\Omega)$, using randomly selected optimal parameters $\hat{\bc}^{*}$ and $\br^{*}$:
\begin{equation}\label{art-eq}
    u(\bx) = \sum_{i=1}^{N}c_i^{*}\phi_{i}(\bx;\br_i^*) + c_{0}^{*},
\end{equation}
where $N$ is the number of neurons. This setup allows for direct evaluation of each optimization method's performance by tracking the movement of the breaking lines toward the known optimal configuration. To avoid trivial convergence due to over-parameterization or initial proximity to the optimum, we restrict the network to just $N=5$ neurons and constrain the initialization of the five breaking lines to lie only along horizontal or vertical directions. For the Adam optimizer, the parameters were set as follows: $\alpha_1=0.1 $, $\alpha_f =0.5 $, $T = 2000$ for horizontal initialization, and $\alpha_1= 0.1$, $\alpha_f =0.8 $, $T = 3000$ for vertical initialization. For KFRA, the damping parameter was chosen as $\gamma = 0.005$.

\begin{figure}[!ht]
    \begin{center}
    \subfigure[Horizontal initial (HI)]{ 
    \label{fig6:init_h} 
    \includegraphics[width=1.35in]{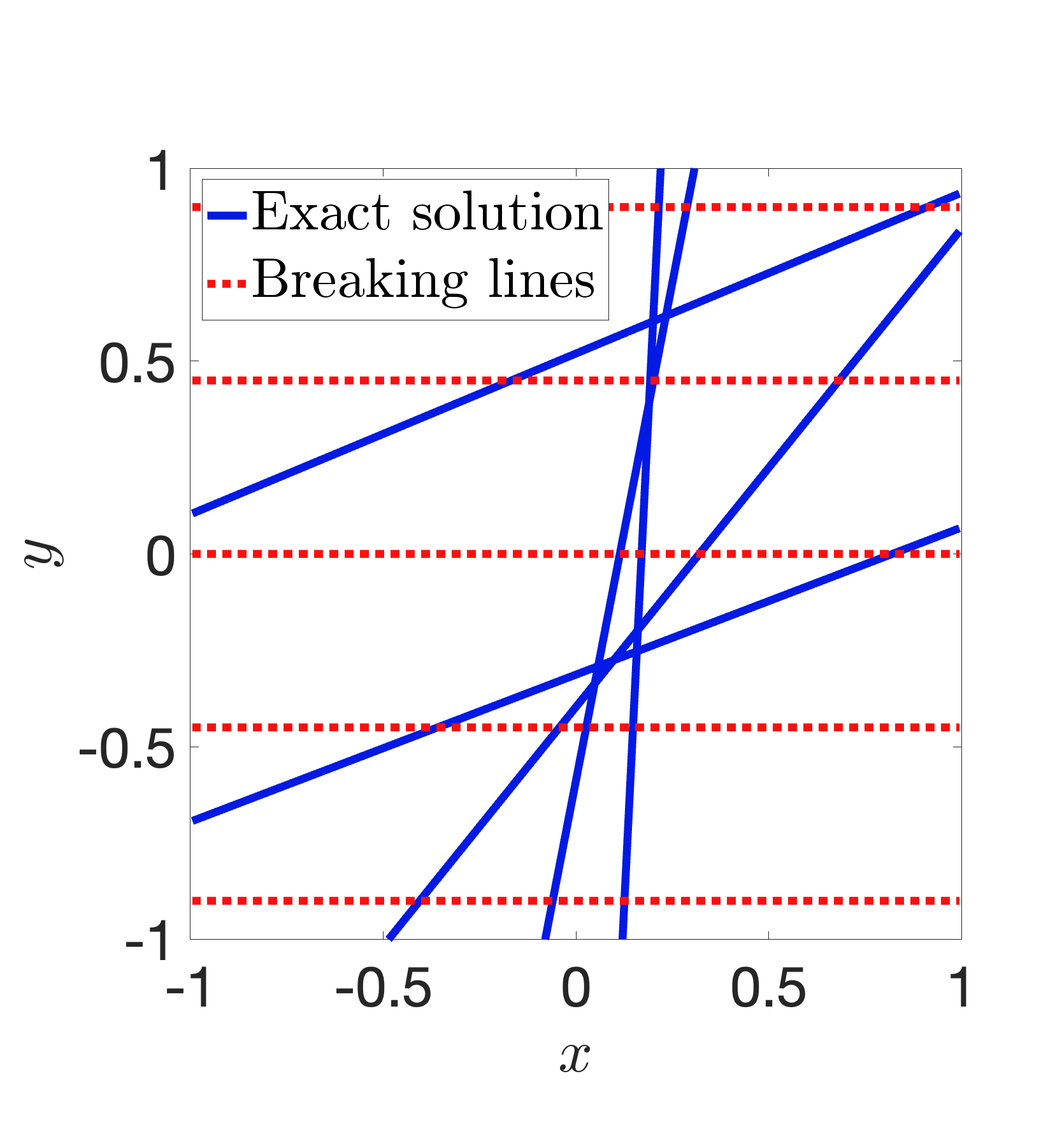}}
     \subfigure[Loss curves (horizontal)]{ 
    \label{fig6:loss_h} 
     \includegraphics[width=1.35in]
     {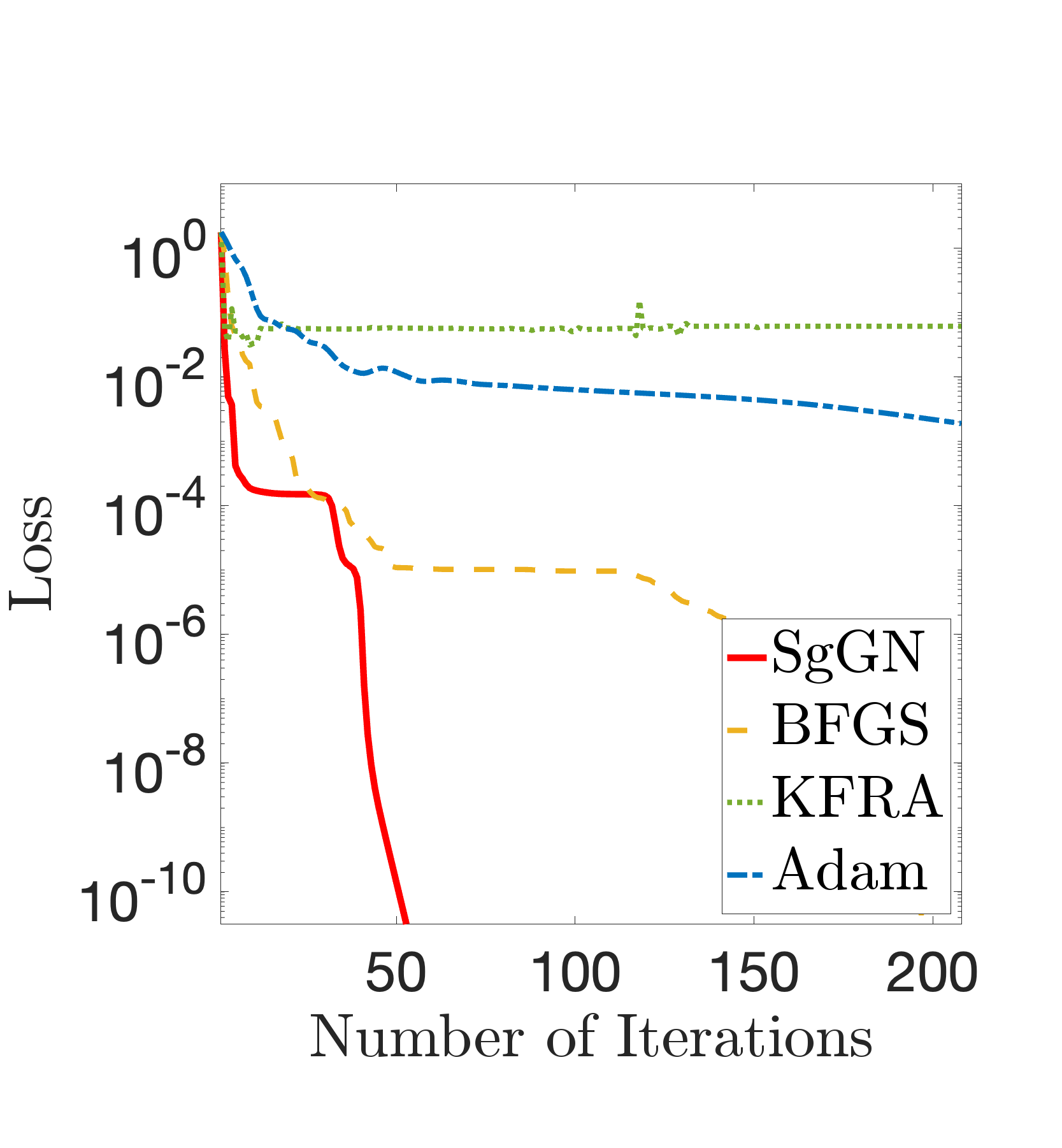}}
    \subfigure[Vertical initial (VI)]{ 
        \label{fig6:init_v} 
        \includegraphics[width=1.35in]{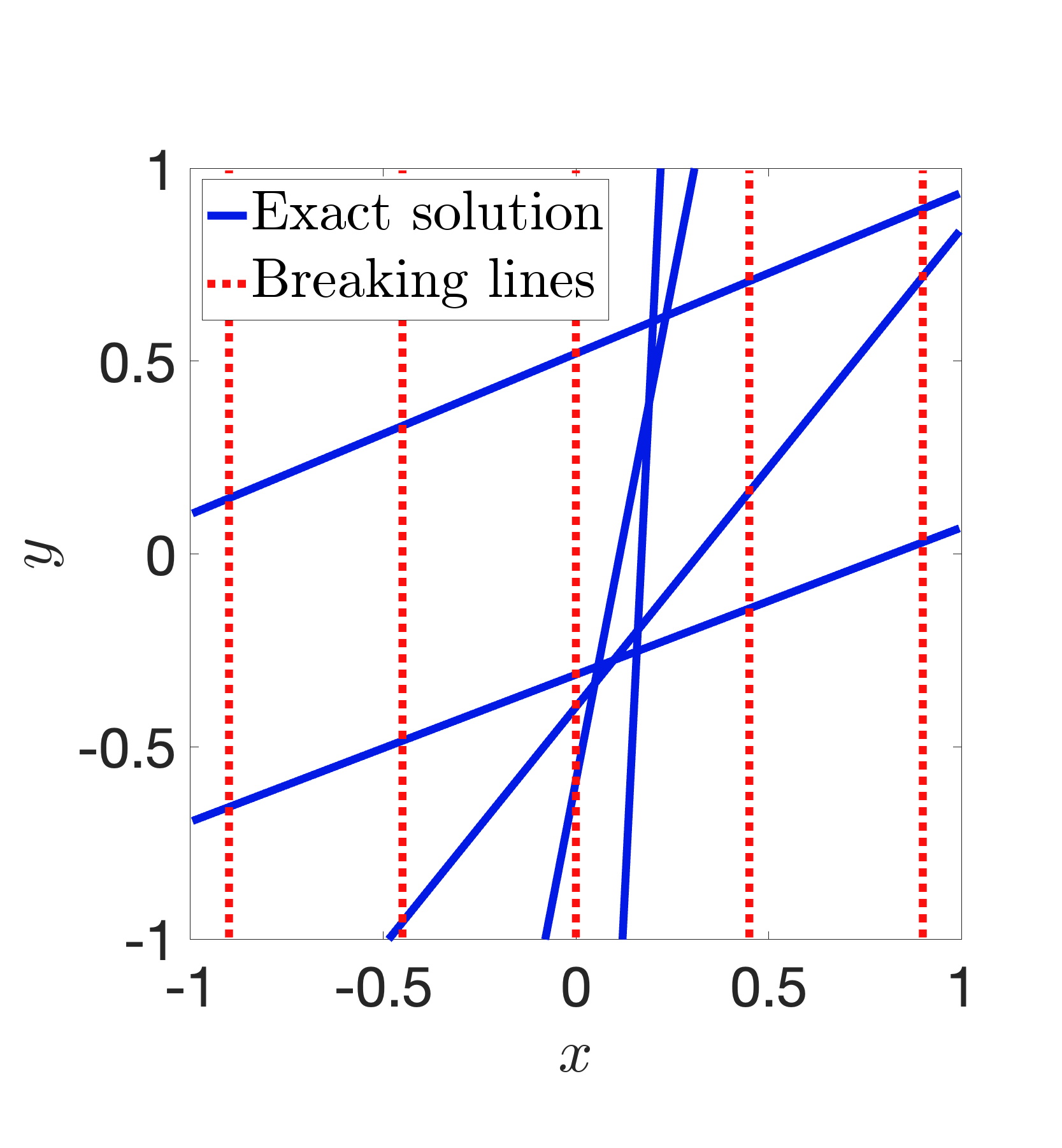}}
     \subfigure[Loss curves (vertical)]{ 
        \label{fig6:loss_v} 
     \includegraphics[width=1.35in]{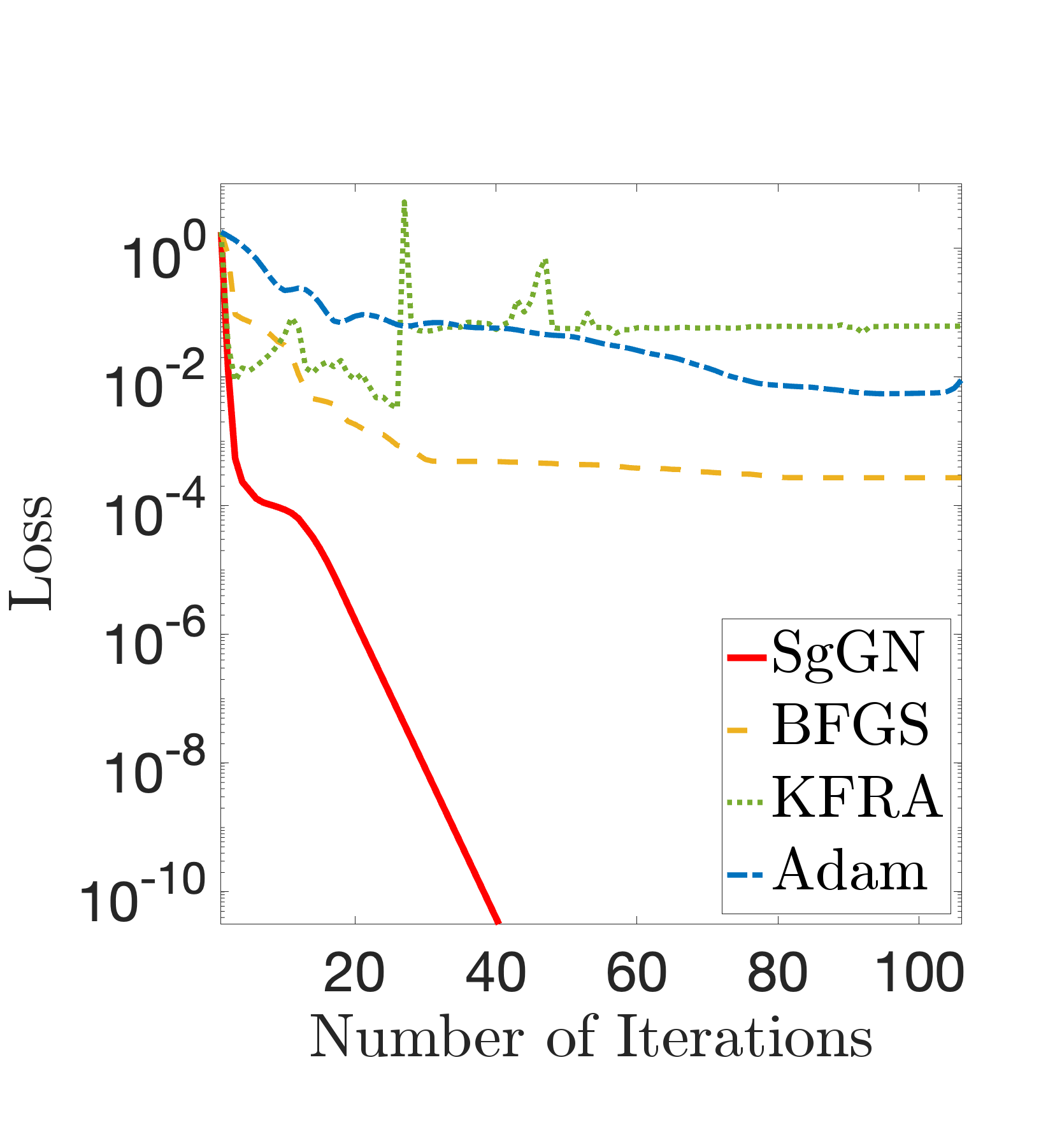}} 
        \subfigure[SgGN $u_n$ (HI)]{ 
        \label{fig6:SgGN_h} 
        \includegraphics[width=1.35in]{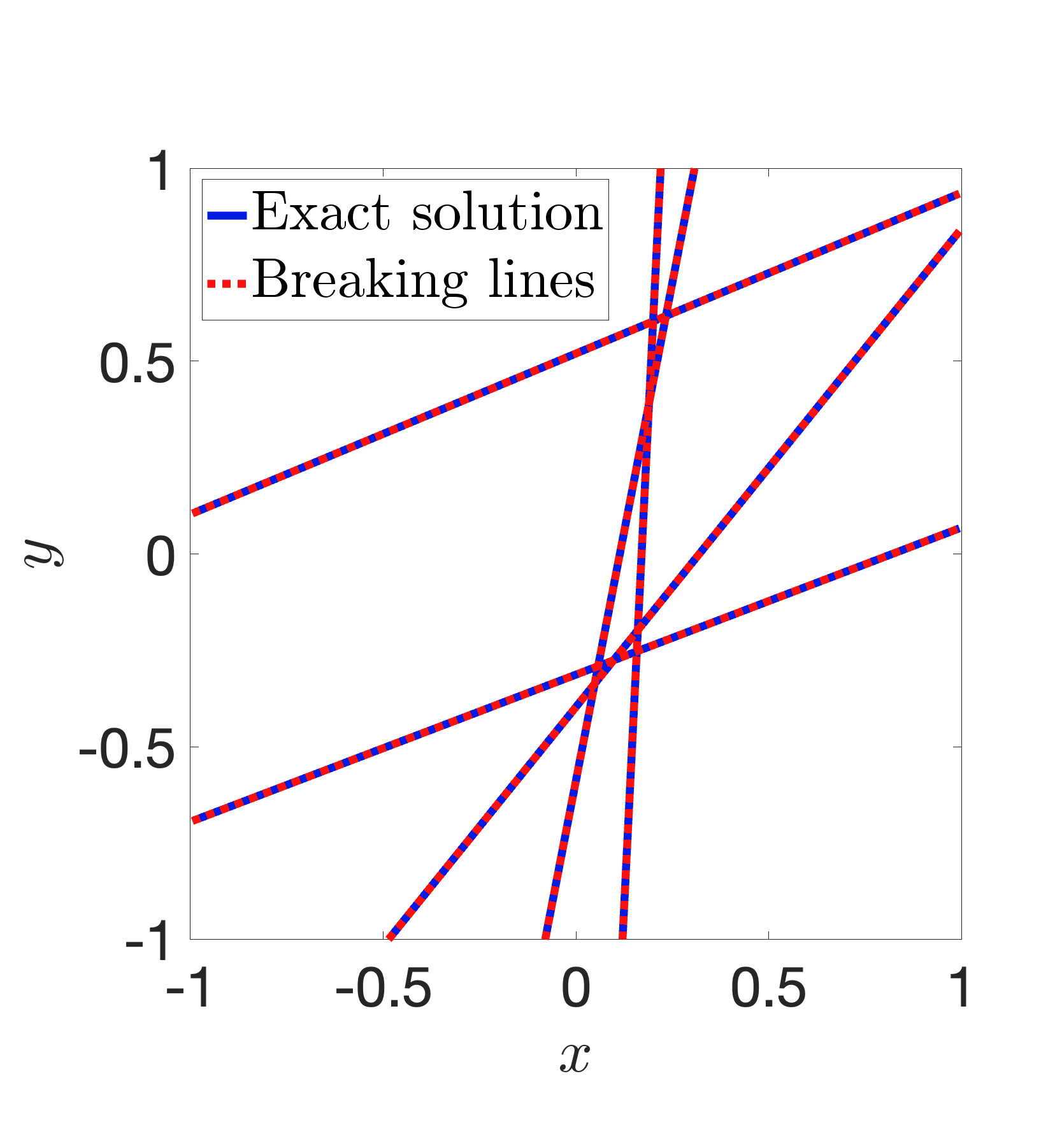}}
        \subfigure[BFGS $u_n$ (HI)]{ 
        \label{fig6:BFGS_h} 
        \includegraphics[width=1.35in]{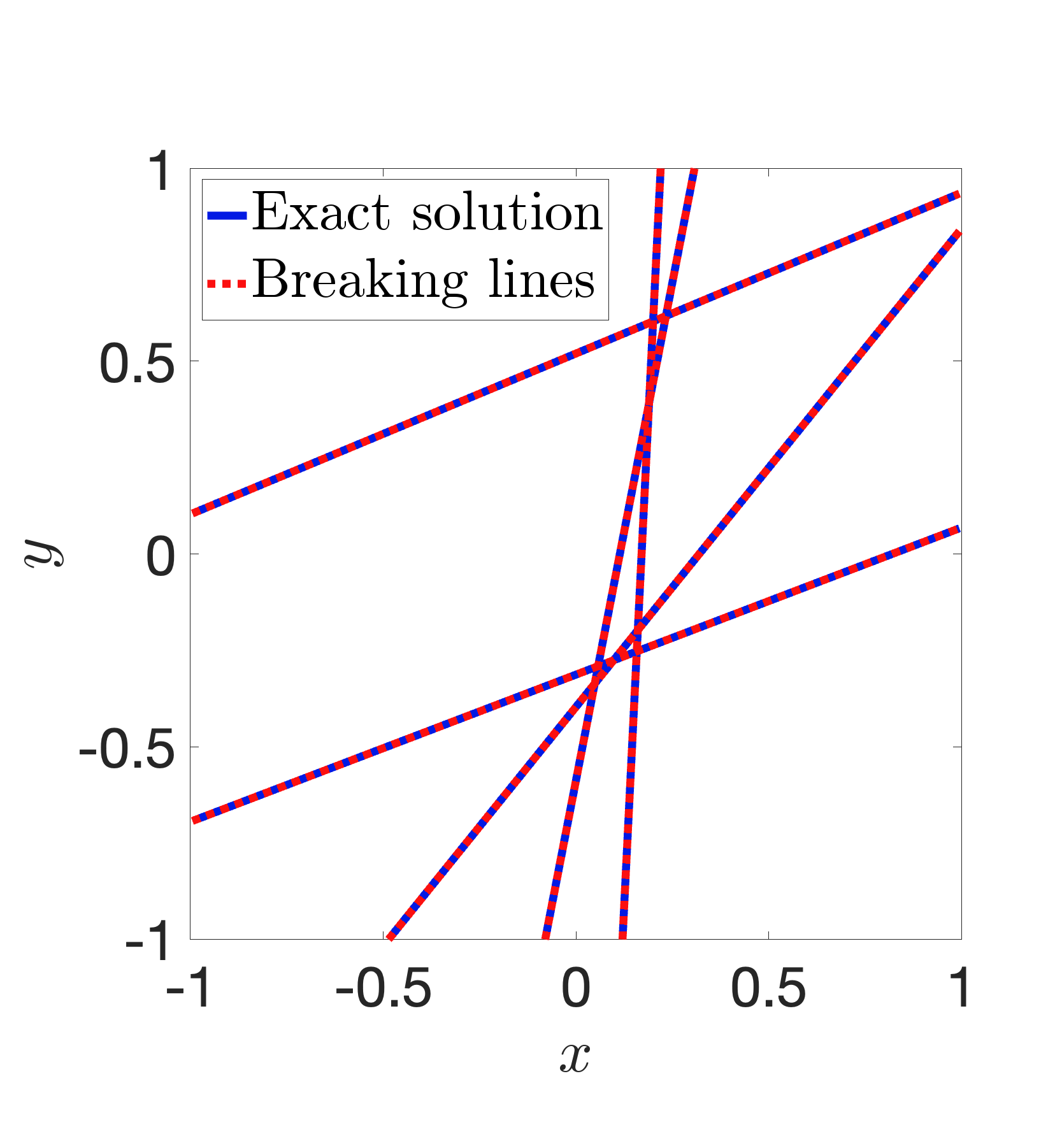}}
        \subfigure[KFRA $u_n$ (HI)]{ 
        \label{fig6:KFRA_h} 
        \includegraphics[width=1.35in]{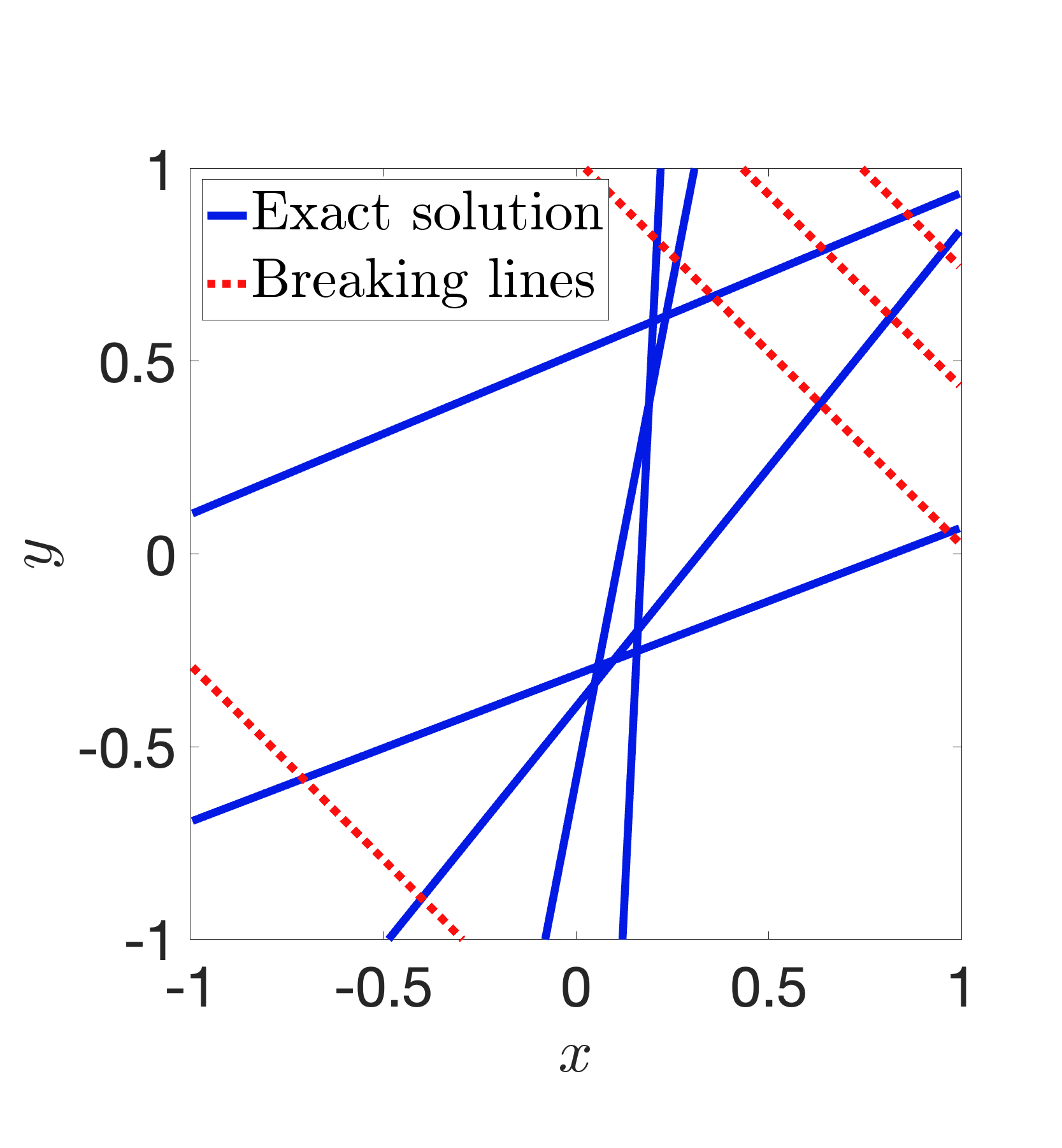}}
        \subfigure[Adam $u_n$ (HI)]{ 
        \label{fig6:Adam_h} 
        \includegraphics[width=1.35in]{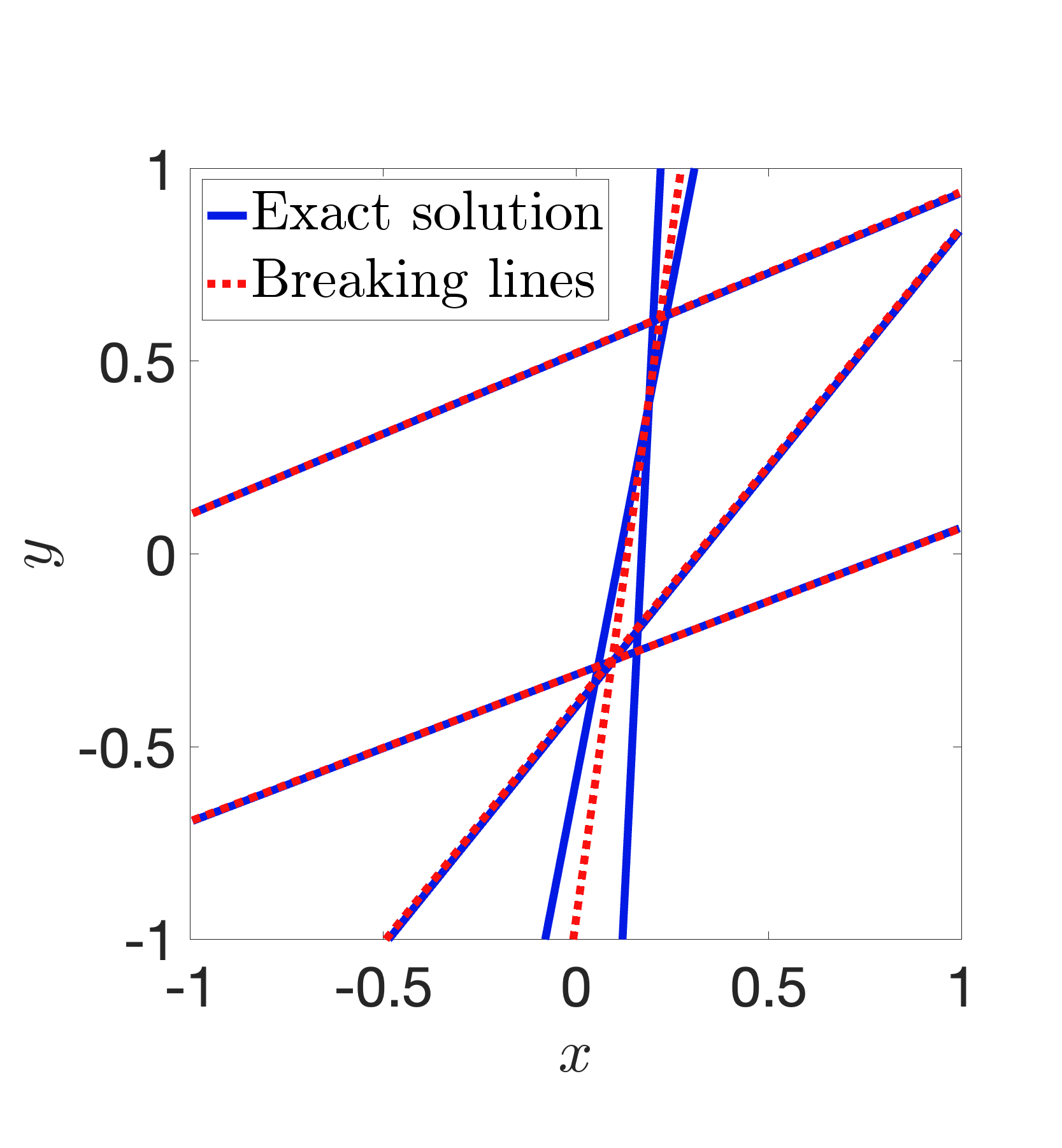}}
        \subfigure[SgGN $u_n$ (VI)]{ 
        \label{fig6:SgGN_v} 
        \includegraphics[width=1.35in]{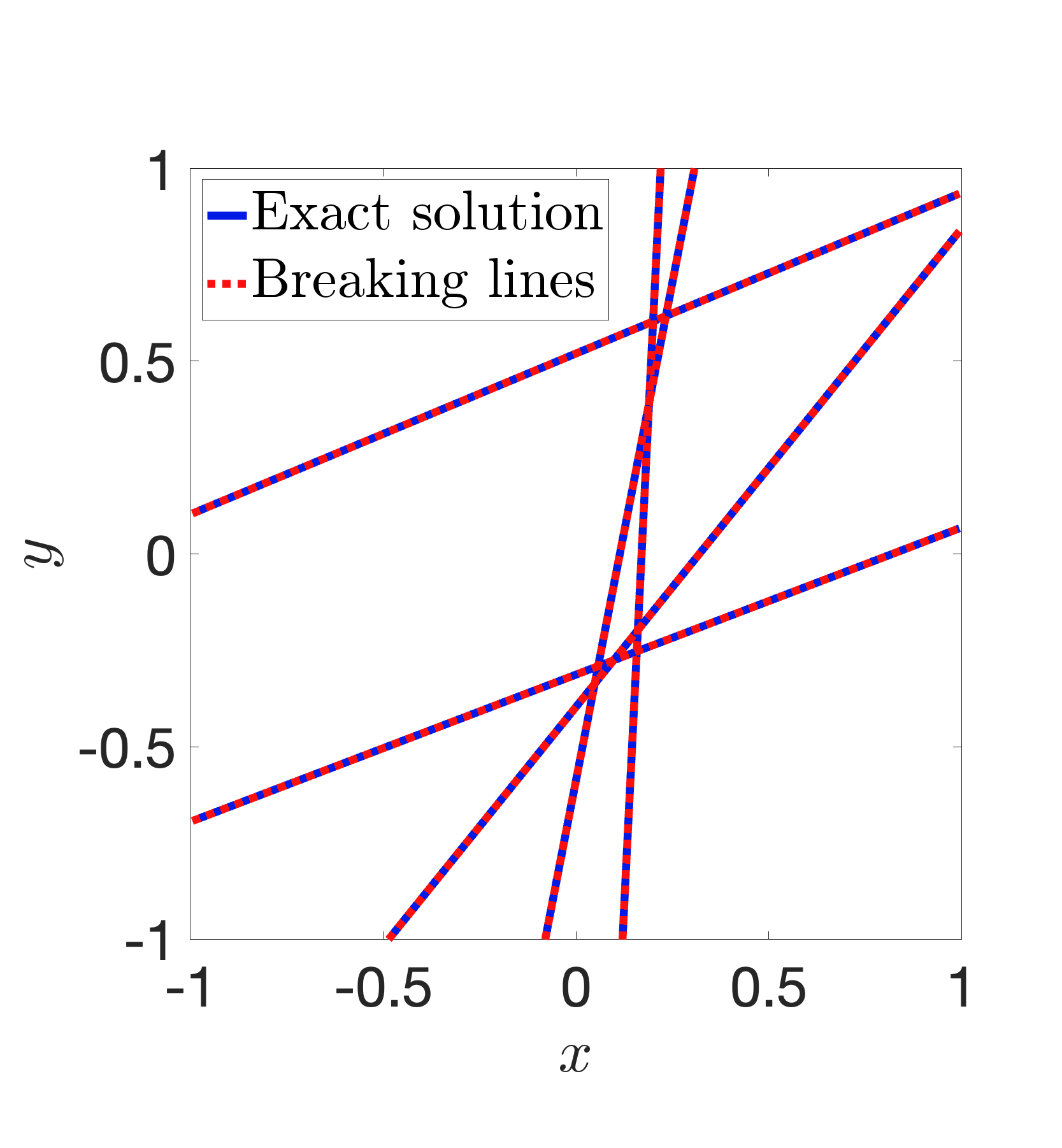}}
        \subfigure[BFGS $u_n$ (VI)]{ 
        \label{fig6:BFGS_v} 
        \includegraphics[width=1.35in]{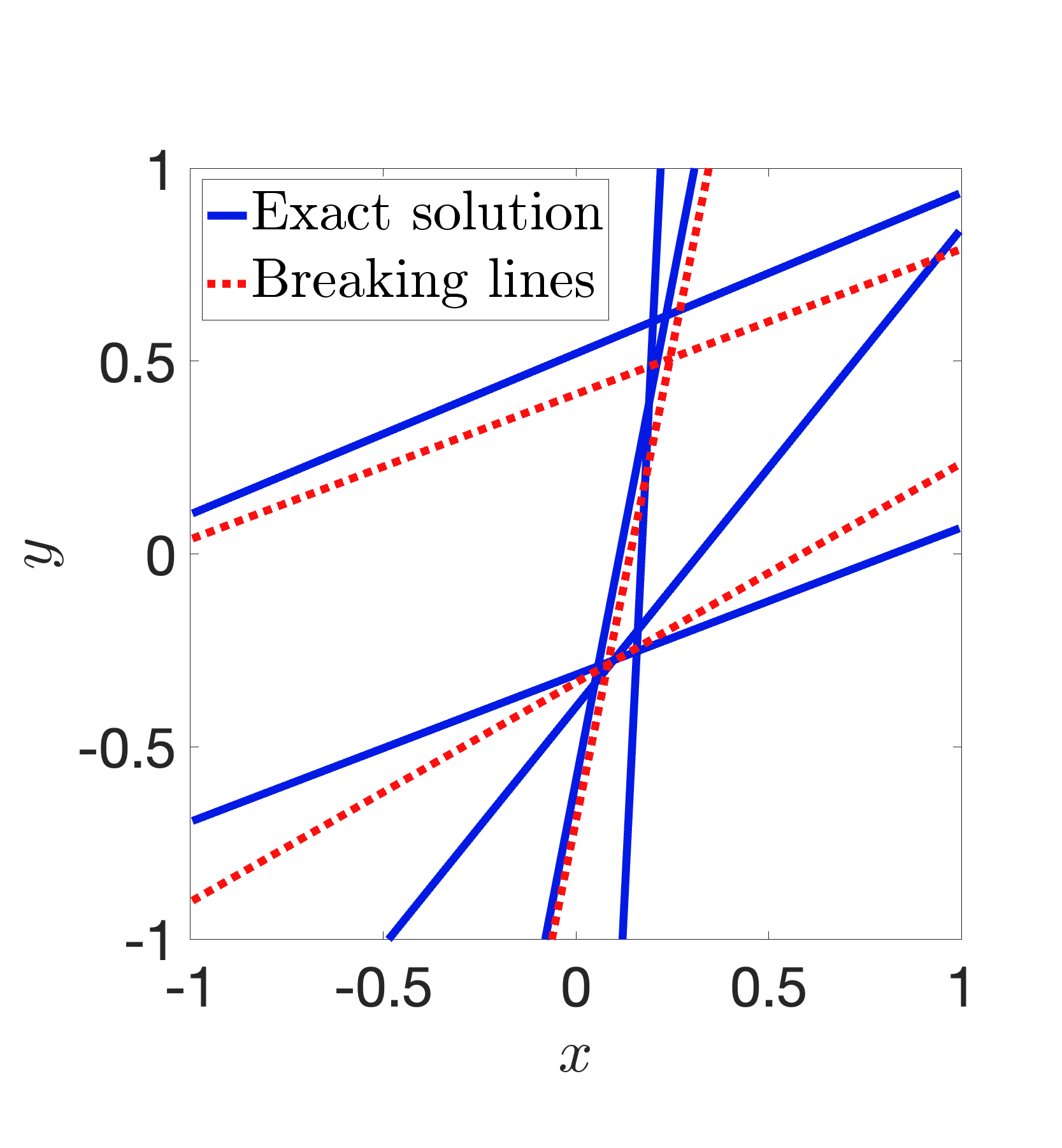}}
        \subfigure[KFRA $u_n$ (VI)]{ 
        \label{fig6:KFRA_v} 
        \includegraphics[width=1.35in]{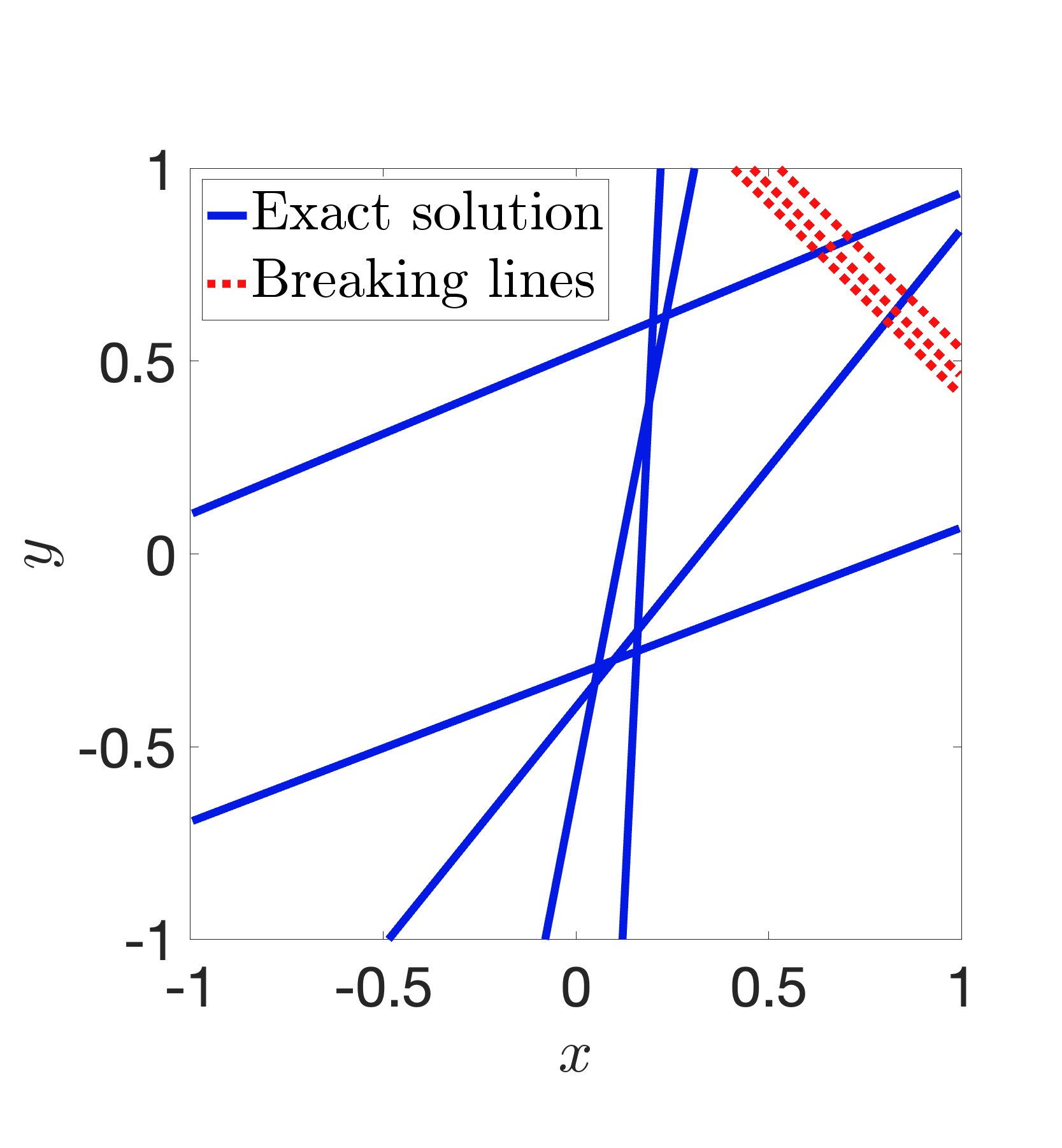}}
        \subfigure[Adam $u_n$ (VI)]{ 
        \label{fig6:Adam_v} 
        \includegraphics[width=1.35in]{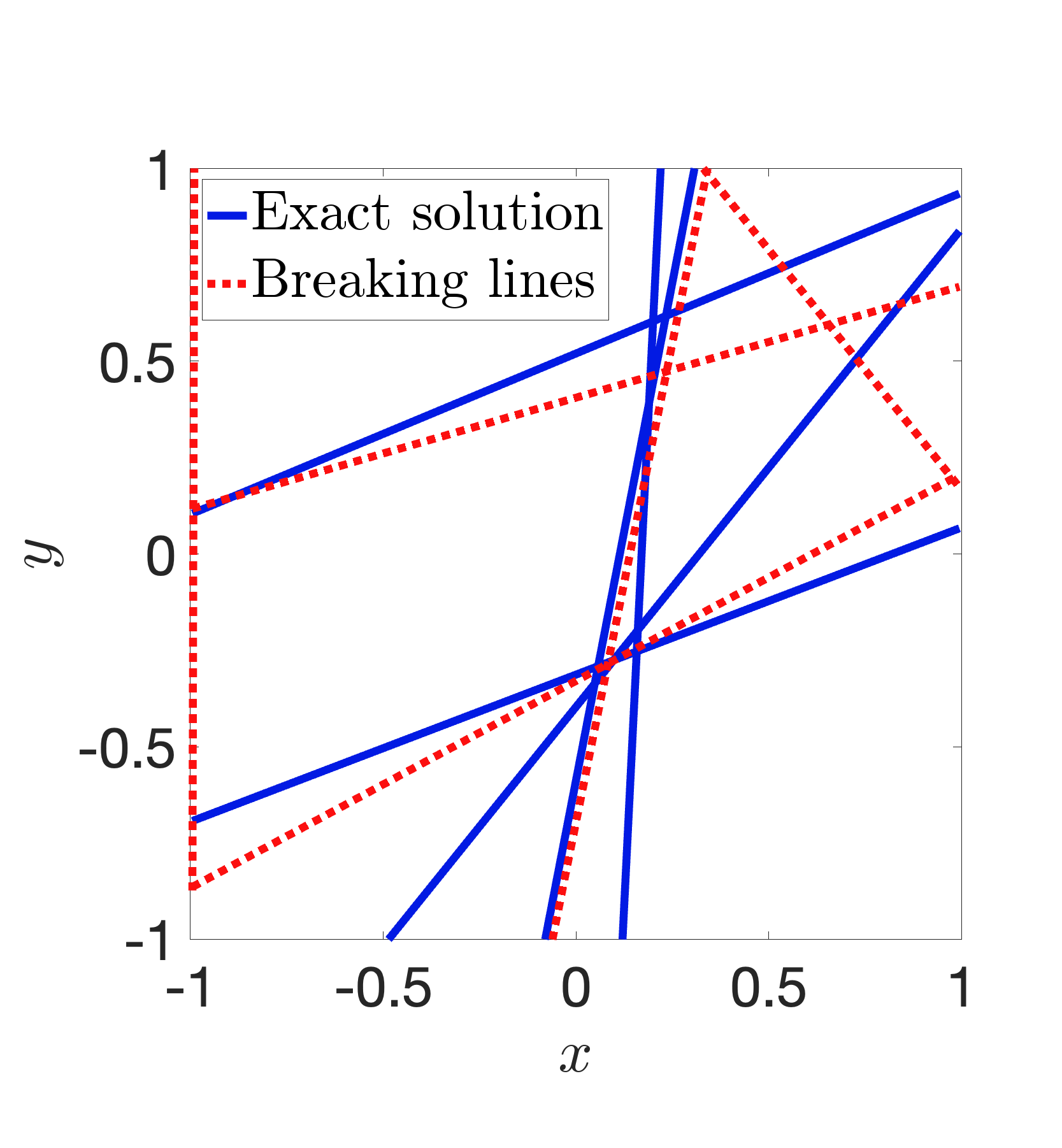}}
     \caption{Approximation of the two-dimensional piecewise linear function: target function, initial breaking lines, optimization loss curves and approximation results using the optimization methods with horizontal initial (HI) and vertical initial (VI) breaking lines.}
    \label{fig6:linear}
    \end{center}
\end{figure}

\begin{table}[!h]
        \caption{Accuracy comparison for the approximation of the two-dimensional piecewise linear function with horizontal initial breaking lines (HI) and vertical initial breaking lines (VI).}
        \label{tab4:linear}
    \begin{center}
        \begin{tabular}{ |c|c|c|c|c|c|c|}
             \hline
             Method & \multicolumn{2}{c|}{SgGN} & \multicolumn{2}{c|}{BFGS} &KFRA& Adam \\
             \hline\hline
             Iteration& $\bm{99}$ &$207$ &$204$ & $207$ &$207$&$10,000$\\
             \hline
             $\mathcal{J}_{m,\mu}$ (HI) &  $6.28\text{E-}22$ &$\bm{6.68\textbf{E-}27}$ & $7.50\text{E-}22$& $7.50\text{E-}22$& $6.12\text{E-}2$&$1.17\text{E-}5$\\
             \hline
            Iteration& $\bm{4}$ &$105$ &$30$ & $105$ &105 &$10,000$\\
             \hline
            $\mathcal{J}_{m,\mu}$ (VI) & $2.35\text{E-}4$ &$\bm{4.34\textbf{E-}{26}}$& $5.21\text{E-}4$ & $2.71\text{E-}4$& $5.56\text{E-}2$&$2.15\text{E-}4$\\
            \hline
        \end{tabular}
    \end{center}
\end{table}

The loss decay results are shown in Figures~\ref{fig6:loss_h} and \ref{fig6:loss_v}. SgGN rapidly converges to a loss on the order of $10^{-10}$ within just 50 iterations. In contrast, the other methods show significantly slower convergence and require many more iterations to approach their final training losses. Table~\ref{tab4:linear} presents a comparison of the final loss values under both horizontal and vertical initializations. Notably, SgGN consistently achieves near-zero loss in both settings, significantly outperforming the other methods.
Further insight is provided in Figure~\ref{fig6:linear}, which shows the final configuration of the breaking lines. SgGN accurately relocates all five breaking lines to their respective optimal positions under both horizontal and vertical initializations (Figures~\ref{fig6:SgGN_h} and \ref{fig6:SgGN_v}). In contrast, the other methods (Figures~\ref{fig6:BFGS_h}, \ref{fig6:BFGS_v}, \ref{fig6:Adam_h}, and \ref{fig6:Adam_v}) either perform well only under horizontal initialization or fail to correctly relocate all breaking lines in both scenarios.

\section{Conclusions and Discussions}\label{sec:conl}
For nonlinear LS problems, GN methods offer attractive features by exploiting the quadratic form of the objective function. However, they often suffer from the singularity of the GN matrix, necessitating additional strategies, such as shifting, to ensure invertibility.
The SgGN method introduced in this paper is an iterative method for solving nonlinear LS problems using shallow ReLU NNs as the model. In addition to leveraging the LS structure, the method also makes effective use of the structure of the network. Guided by both structures, the SgGN method offers several significant advantages. Foremost among these is its guarantee of positive definiteness of the mass matrix and the layer GN matrix without requiring additional shifting \textemdash a common requirement in standard GN methods. Moreover, SgGN explicitly removes the singularity of the GN matrix following its structured form along the NN optimization process. This work thus gives a practical strategy to take advantage of NN LS structures while uncovering the singularity structure of GN matrices.

Another notable advantage of SgGN is its rapid convergence in practice.
The method has been tested for several one- and two-dimensional LS problems that are particularly challenging for commonly used machine learning optimizers such as BFGS and Adam. The resulting loss curves consistently demonstrate the superior convergence behavior of SgGN, which frequently outperforms these baseline methods by a considerable margin. This performance advantage is further supported by SgGN’s ability to effectively reposition the breaking hyperplanes (breaking points for one dimension and breaking lines for two dimensions) defined by the network's nonlinear parameters. 

Each iteration of the SgGN algorithm requires linear solvers to approximately invert the mass matrix $\mathcal{A}(\br^{(k)})$ and the layer GN matrix $\mathcal{H}(\br^{(k)})$ for updating the linear and nonlinear parameters, respectively. While both matrices are symmetric positive definite, they can be highly ill-conditioned. In the numerical experiments reported in this paper, we used truncated SVD as a robust linear solver, albeit at a significant computational cost. To improve efficiency, future work will focus on developing more scalable and structure-aware linear solvers within the SgGN framework \cite{SgGN2}.

\bibliographystyle{siam}
\bibliography{reference}

\def\cprime{$'$}
\begin{thebibliography}{10}

\bibitem{ain22}
{\sc M.~Ainsworth and Y.~Shin}, {\em Active neuron least squares: a training method for multivariate rectified neural networks}, SIAM J. Sci. Comput., 44 (2022), pp.~A2253--A2275.

\bibitem{AlGe:90}
{\sc E.~Allgower and K.~Georg}, {\em Numerical Continuation Methods}, Springer-Verlag, Berlin and Heidelberg, 1990.

\bibitem{botev2017practical}
{\sc A.~Botev, H.~Ritter, and D.~Barber}, {\em Practical {G}auss--{N}ewton optimisation for deep learning}, in Proc. Int. Conf. Mach. Learning (ICML), PMLR, 2017, pp.~557--565.

\bibitem{SIAM_Review18}
{\sc L.~Bottou, F.~E. Curtis, and J.~Nocedal}, {\em Optimization methods for large-scale machine learning}, SIAM Rev., 60 (2018), p.~223–311.

\bibitem{broyden1970convergence}
{\sc C.~G. Broyden}, {\em The convergence of a class of double-rank minimization algorithms 1. general considerations}, IMA J. Appl. Math., 6 (1970), pp.~76--90.

\bibitem{Cai2021linear}
{\sc Z.~Cai, J.~Chen, and M.~Liu}, {\em Least-squares {ReLU} neural network {(LSNN)} method for linear advection-reaction equation}, J. Comput. Phys., 443 (2021), p.~110514.

\bibitem{Cai2021nonlinear}
\leavevmode\vrule height 2pt depth -1.6pt width 23pt, {\em Least-squares {ReLU} neural network {(LSNN)} method for scalar nonlinear hyperbolic conservation law}, Appl. Numer. Math., 174 (2022), pp.~163--176.

\bibitem{Cai2023linear}
{\sc Z.~Cai, J.~Choi, and M.~Liu}, {\em Least-squares neural network {(LSNN)} method for linear advection-reaction equation: discontinuity interface}, SIAM J. Sci. Comput., 46 (2024), pp.~C448--C478.

\bibitem{SgGN2}
{\sc Z.~Cai, T.~Ding, M.~Liu, X.~Liu, and J.~Xia}, {\em Matrix analysis for shallow relu neural network least-squares approximations}, preprint.

\bibitem{Cai2023AI}
{\sc Z.~Cai and M.~Liu}, {\em Self-adaptive {ReLU} neural network method in least-squares data fitting}, in Principles and Applications of Adaptive Artificial Intelligence, IGI Global, 2024, pp.~242--262.

\bibitem{varproj2}
{\sc E.~C. Cyr, M.~A. Gulian, R.~G. Patel, M.~Perego, and N.~A. Trask}, {\em Robust training and initialization of deep neural networks: An adaptive basis viewpoint}, Proc. Mach. Learn. Res., 107 (2020), pp.~512--536.

\bibitem{fletcher1970new}
{\sc R.~Fletcher}, {\em A new approach to variable metric algorithms}, Comput. J., 13 (1970), pp.~317--322.

\bibitem{Survey19}
{\sc C.~Gambella, B.~Ghaddar, and J.~Naoum-Sawaya}, {\em Optimization problems for machine learning: a survey}, Eur. J. Oper. Res., 290 (2021), pp.~807--828.

\bibitem{goldfarb1970family}
{\sc D.~Goldfarb}, {\em A family of variable-metric methods derived by variational means}, Math. Comp., 24 (1970), pp.~23--26.

\bibitem{varproj}
{\sc G.~H. Golub and V.~Pereyra}, {\em The differentiation of pseudo-inverses and nonlinear least squares problems whose variables separate}, SIAM J. Numer. Anal., 10 (1973), pp.~413--432.

\bibitem{Hao2024}
{\sc W.~Hao, Q.~Hong, and X.~Jin}, {\em Gauss--{N}ewton method for solving variational problems of {PDE}s with neural network discretizations}, J. Sci. Comput., 100 (2024).

\bibitem{he2020relu}
{\sc J.~He, L.~Li, J.~Xu, and C.~Zheng}, {\em {{ReLU}} deep neural networks and linear finite elements}, J. Comput. Math., 38 (2020), p.~502–527.

\bibitem{Qing22}
{\sc Q.~Hong, J.~W. Siegel, and J.~Xu}, {\em On the activation function dependence of the spectral bias of neural networks}, arXiv:2208.04924,  (2022).

\bibitem{Jnini2024}
{\sc A.~Jnini, F.~Vella, and M.~Zeinhofer}, {\em Gauss--{N}ewton natural gradient descent for physics-informed computational fluid dynamics}, arXiv:2402.10680,  (2024).

\bibitem{dennis1996numerical}
{\sc J.~E.~D. Jr. and R.~B. Schnabel}, {\em Numerical methods for unconstrained optimization and nonlinear equations}, SIAM, 1996.

\bibitem{kingma2015}
{\sc D.~P. Kingma and J.~Ba}, {\em {ADAM}: A method for stochastic optimization}, in Proc. Int. Conf. Learn. Represent. (ICLR), 2015.
\newblock arXiv:1412.6980.

\bibitem{Levenberg}
{\sc K.~Levenberg}, {\em A method for the solution of certain non-linear problems in least squares}, Quart. Appl. Math., 2 (1944), pp.~164--168.

\bibitem{LiuCai2}
{\sc M.~Liu and Z.~Cai}, {\em Adaptive two-layer relu neural network: {II}. {Ritz} approximation to elliptic {PDEs}}, Comput. Math. Appl., 113 (2022), pp.~103--116.

\bibitem{LiuCai1}
{\sc M.~Liu, Z.~Cai, and J.~Chen}, {\em Adaptive two-layer {ReLU} neural network: {I.} best least-squares approximation}, Comput. Math. Appl., 113 (2022), pp.~34--44.

\bibitem{LiCaRa23}
{\sc M.~Liu, Z.~Cai, and K.~Ramani}, {\em Deep ritz method with adaptive quadrature for linear elasticity}, Comput. Methods Appl. Mech. Engrg., 415 (2023), p.~116229.

\bibitem{Marquardt}
{\sc D.~W. Marquardt}, {\em An algorithm for least-squares estimation of nonlinear parameters}, J. SIAM, 11 (1963), pp.~431--441.

\bibitem{martens2015optimizing}
{\sc J.~Martens and R.~Grosse}, {\em Optimizing neural networks with {K}ronecker-factored approximate curvature}, in Proc. Int. Conf. Mach. Learning (ICML), PMLR, 2015, pp.~2408--2417.

\bibitem{MatlabDL}
{\sc \mbox{The MathWorks Inc.}}, {\em Deep {L}earning {T}oolbox}, 2023.

\bibitem{Wright}
{\sc J.~Nocedal and S.~J. Wright}, {\em Numerical Optimization}, Springer, 2006.

\bibitem{ortega2000iterative}
{\sc J.~M. Ortega and W.~C. Rheinboldt}, {\em Iterative {S}olution of {N}onlinear {E}quations in {S}everal {V}ariables}, SIAM, 2000.

\bibitem{osborne2007separable}
{\sc M.~Osborne}, {\em Separable least squares, variable projection, and the gauss-newton algorithm}, Electron. Trans. Numer. Anal., 28 (2007), pp.~1--15.

\bibitem{shanno1970conditioning}
{\sc D.~F. Shanno}, {\em Conditioning of quasi-{N}ewton methods for function minimization}, Math. Comp., 24 (1970), pp.~647--656.

\bibitem{XiaShenWang16}
{\sc J.~Shen, Y.~Wang, and J.~Xia}, {\em Fast structured direct spectral methods for differential equations with variable coefficients, i. the one-dimensional case}, SIAM J. Sci. Comput., 38 (2016), pp.~A28--A54.

\bibitem{SCZZ2020_Survey}
{\sc S.~Sun, Z.~Cao, H.~Zhu, and J.~Zhao}, {\em A survey of optimization methods from a machine learning perspective}, IEEE Trans. Cybern., 50 (2020), pp.~3668 -- 3681.

\end{thebibliography}

\end{document}